\documentclass{article}

 \PassOptionsToPackage{numbers, compress}{natbib}

    \usepackage[preprint]{neurips_2023}

\usepackage[utf8]{inputenc} %
\usepackage[T1]{fontenc}    %
\usepackage{hyperref}       %
\usepackage{url}            %
\usepackage{booktabs}       %
\usepackage{amsfonts}       %
\usepackage{nicefrac}       %
\usepackage{microtype}      %
\usepackage{xcolor}         %

\usepackage{amsmath,amsfonts,bm,mathtools}

\def\eqref#1{equation~\ref{#1}}

\def\1{\bm{1}}

\def\vtheta{{\bm{\theta}}}

\def\ve{{\bm{e}}}

\def\vp{{\bm{p}}}
\def\vq{{\bm{q}}}

\def\vu{{\bm{u}}}
\def\vv{{\bm{v}}}

\def\vx{{\bm{x}}}
\def\vy{{\bm{y}}}

\def\vU{{\bm{U}}}
\def\vV{{\bm{V}}}

\def\vtheta{{\bm{\theta}}}
\def\vphi{{\bm{\phi}}}

\DeclareMathAlphabet{\mathsfit}{\encodingdefault}{\sfdefault}{m}{sl}
\SetMathAlphabet{\mathsfit}{bold}{\encodingdefault}{\sfdefault}{bx}{n}

\def\gA{{\mathcal{A}}}
\def\gB{{\mathcal{B}}}

\def\gG{{\mathcal{G}}}

\def\gI{{\mathcal{I}}}

\def\gL{{\mathcal{L}}}

\def\gS{{\mathcal{S}}}
\def\gT{{\mathcal{T}}}

\def\sR{{\mathbb{R}}}
\def\sS{{\mathbb{S}}}

\newcommand{\R}{\mathbb{R}}

\DeclarePairedDelimiter{\dotp}{\langle}{\rangle}
\DeclarePairedDelimiter{\abs}{|}{|}
\DeclarePairedDelimiter{\norm}{\lVert}{\rVert}

\newcommand{\expect}{{\mathbb{E}}}
\newcommand{\prob}{{\mathbb{P}}}

\usepackage{subfigure}
\usepackage{rotating}
\usepackage{tikz}

\usepackage{amsmath}
\usepackage{amssymb}
\usepackage{mathtools}
\usepackage{amsthm}
\usepackage{thm-restate}
\usepackage{bbm} %
\usepackage{wrapfig}
\usepackage{amssymb}
\usepackage{amsfonts}
\usepackage{bbm}
\usepackage{enumitem}

\usepackage{cuted}
\usepackage{flushend}

\usepackage{soul}
\usepackage[ruled,vlined,linesnumbered]{algorithm2e}
\usepackage{algorithmicx}
\usepackage{algpseudocode}

\newcommand{\eg}{{\it e.g.}, }
\newcommand{\ie}{{\it i.e.}, }
\newcommand\op[1]{\operatorname{#1}}

\newcommand{\lclip}{\gL^{\op{con}}}
\newcommand{\lcliptilde}{\widetilde{\gL}^{\op{con}}}
\newcommand{\lclipstar}{\widetilde{\gL}^{\op{con}\star}}

\usepackage[capitalize,noabbrev]{cleveref}

\theoremstyle{plain}

\newtheorem{theorem}{Theorem}

\newtheorem{lemma}{Lemma}
\newtheorem{corollary}{Corollary}
\theoremstyle{definition}
\newtheorem{definition}{Definition}

\newtheorem{remark}{Remark}

\title{Mini-Batch Optimization of Contrastive Loss}

\author{%
  Jaewoong Cho\thanks{Equal Contributions. Emails: <jwcho@krafton.com, ksreenivasan@cs.wisc.edu>. } \\
  KRAFTON\\
  \And
  Kartik Sreenivasan\footnotemark[1] \\
   University of Wisconsin-Madison\\
  \And
  Keon Lee\\
  KRAFTON \\
  \AND
  Kyunghoo Mun \\
  KRAFTON \\
  \And
  Soheun Yi \\
  Seoul National University \\
  KRAFTON \\
  \And
  Jeong-Gwan Lee\\
  KRAFTON \\
  \And
  Anna Lee\\
  KRAFTON \\
  \And
  Jy-yong Sohn\\
  Yonsei University\\
  \And
  Dimitris Papailiopoulos\\
  University of Wisconsin-Madison \\
  KRAFTON \\
  \And
  Kangwook Lee\thanks{Correspondence to: Kangwook Lee <kangwook.lee@wisc.edu>.}\\
  University of Wisconsin-Madison \\
  KRAFTON \\
}

\begin{document}

\maketitle

\begin{abstract}
Contrastive learning has gained significant attention as a method for self-supervised learning. The contrastive loss function ensures that embeddings of positive sample pairs (e.g., different samples from the same class or different views of the same object) are similar, while embeddings of negative pairs are dissimilar. Practical constraints such as large memory requirements make it challenging to consider all possible positive and negative pairs, leading to the use of \emph{mini-batch} optimization.
In this paper, we investigate the theoretical aspects of mini-batch optimization in contrastive learning. We show that mini-batch optimization is equivalent to full-batch optimization if and only if all $\binom{N}{B}$ mini-batches are selected, while sub-optimality may arise when examining only a subset. We then demonstrate that utilizing high-loss mini-batches can speed up SGD convergence and propose a spectral clustering-based approach for identifying these high-loss mini-batches.
Our experimental results validate our theoretical findings and demonstrate that our proposed algorithm outperforms vanilla SGD in practically relevant settings, providing a better understanding of mini-batch optimization in contrastive learning.
\end{abstract}

\section{Introduction}
Contrastive learning has been widely employed in various domains as a prominent method for self-supervised learning~\citep{jaiswal2020survey}.
The contrastive loss function is designed to ensure that the embeddings of two samples are similar if they are considered a ``positive'' pair, in cases such as coming from the same class~\citep{khosla2020supervised}, being an augmented version of one another~\cite{chen2020simple}, or being two different modalities of the same data~\citep{radford2021learning}. Conversely, if two samples do not form a positive pair, they are considered a ``negative'' pair, and the contrastive loss encourages their embeddings to be dissimilar.

In practice, it is not feasible to consider all possible positive and negative pairs when implementing a contrastive learning algorithm due to the quadratic memory requirement $\mathcal{O}(N^2)$ when working with $N$ samples. 
To mitigate this issue of \emph{full-batch training}, practitioners typically choose a set of $N/B$ \emph{mini-batches}, each of size $B = \mathcal{O}(1)$, and consider the loss computed for positive and negative pairs within each of the $N/B$ batches~\citep{chen2022why, chen2020simple, hu2021adco, zeng2021positional, chen2021ice, zolfaghari2021crossclr, gadre2023datacomp}. For instance, \citet{gadre2023datacomp} train a model on a dataset where $N = 1.28 \times 10^7$ and $B = 4096$.
This approach results in a memory requirement of $\mathcal{O}(B^2) = \mathcal{O}(1)$ for each mini-batch, and a total computational complexity linear in the number of chosen mini-batches. Despite the widespread practical use of mini-batch optimization in contrastive learning, there remains a lack of theoretical understanding as to whether this approach is truly reflective of the original goal of minimizing \emph{full-batch} contrastive loss. This paper examines the theoretical aspects of optimizing mini-batches loaded for the contrastive learning.

\paragraph{Main Contributions.} The primary contributions of this paper are twofold. First, we show that under certain parameter settings, mini-batch optimization is equivalent to full-batch optimization if and only if all $\binom{N}{B}$ mini-batches are selected. 
These results are based on an interesting connection between contrastive learning and the neural collapse phenomenon~\cite{lu2020neural}.
From a computational complexity perspective, the identified equivalence condition may be seen as somewhat prohibitive, as it implies that all $\binom{N}{B} = \mathcal{O}(N^B)$ mini-batches must be considered. 

Our second contribution is to show that \emph{Ordered SGD (OSGD)}~\citep{kawaguchi2020ordered} can be effective in finding mini-batches that contain the most informative pairs and thereby speeding up convergence. OSGD, proposed in a work by~\citet{kawaguchi2020ordered}, is a variant of SGD that modifies the model parameter updates. Instead of using the gradient of the average loss of all samples in a mini-batch, it uses the gradient of the average loss over the top-$q$ samples in terms of individual loss values. We show that the convergence result from~\citet{kawaguchi2020ordered} can be applied directly to contrastive learning. We also show that OSGD can improve the convergence rate of SGD by a constant factor in certain scenarios. 
Furthermore, in a novel approach to address the challenge of applying OSGD to the ${\binom{N}{B}}$ mini-batch optimization (which involves examining $\mathcal{O}(N^B)$ batches to select high-loss ones), we reinterpret the batch selection as a min-cut problem in graph theory~\citep{cormen2022introduction}. This novel interpretation allows us to select high-loss batches efficiently via a spectral clustering algorithm~\citep{ng2001spectral}. The following informal theorems summarize our main findings. 

\begin{theorem}[informal]
    Under certain parameter settings, the mini-batch optimization of contrastive loss is equivalent to full-batch optimization of contrastive loss if and only if all $\binom{N}{B}$ mini-batches are selected. Although $\binom{N}{B}$ mini-batch contrastive loss and full-batch loss are neither identical nor differ by a constant factor, the optimal solutions for both mini-batch and full-batch are identical (see Sec.~\ref{sec:relationship_full_mini}). 
\end{theorem}

\begin{theorem}[informal] In a demonstrative toy example, 
    OSGD operating on the principle of selecting high-loss batches, can potentially converge to the optimal solution of mini-batch contrastive loss optimization faster by a constant factor compared to SGD (see Sec.~\ref{subsec:toy_example}). 
\end{theorem}

We validate our theoretical findings and the efficacy of the proposed spectral clustering-based batch selection method by conducting experiments on both synthetic and real data. On synthetic data, we show that our proposed batch-selection algorithms do indeed converge to the optimal solution of \emph{full-batch optimization} significantly faster than the baselines. We also apply our proposed method to ResNet pre-training with CIFAR-100~\cite{krizhevsky2009learning} and Tiny ImageNet~\cite{le2015tiny}. We evaluate the performance on downstream retrieval tasks, demonstrating that our batch selection method outperforms vanilla SGD in practically relevant settings.

\section{Related Work}\label{sec:related-work}
\paragraph{Contrastive losses.} 
Contrastive learning has been used for several decades to learn a similarity metric to be used later for applications such as object detection and recognition~\citep{misra2020self, aberdam2021sequence}. \citet{chopra2005learning} proposed one of the early versions of contrastive loss which has been updated and improved over the years~\citep{sohn2016improved, song2019understanding, schroff2015facenet, khosla2020supervised, oord2018representation}. More recently, contrastive learning has been shown to rival and even surpass traditional supervised learning methods, particularly on image classification tasks~\citep{chen2020improved, bachman2019learning}. Further, its multi-modal adaptation leverages vast unstructured data, extending its effectiveness beyond image and text modalities~\citep{radford2021learning, jia2021scaling, pham2021combined, ma2020active, sachidananda2022calm, elizalde2022clap, goel2022cyclip, lee2022uniclip, ramesh2021zero, ramesh2022hierarchical}. Unfortunately, these methods require extremely large batch sizes in order to perform effectively. Follow-up works showed that using momentum or carefully modifying the augmentation schemes can alleviate this issue to some extent~\citep{he2020momentum, chen2020improved, grill2020bootstrap, wang2022contrastive}. 

\paragraph{Effect of batch size.}
While most successful applications of contrastive learning use large batch sizes (\eg 32,768 for CLIP and 8,192 for SimCLR), recent efforts have focused on reducing batch sizes and improving convergence rates~\citep{yeh2022decoupled, chen2022why}. \citet{yuan2022provable} carefully study the effect of the requirements on the convergence rate when a model is trained for minimizing SimCLR loss, and prove that the gradient of the solution is bounded by $\mathcal{O}(\frac{1}{\sqrt{B}})$. They also propose SogCLR, an algorithm with a modified gradient update where the correction term allows for an improved convergence rate with better dependence on $B$. 
It is shown that the performance for small batch size can be improved with the technique called hard negative mining~\citep{robinson2020contrastive, kalantidis2020hard, zhang2021understanding}.

\paragraph{Neural collapse.}
Neural collapse is a phenomenon observed in \citep{papyan2020prevalence} where the final classification layer of deep neural nets collapses to the simplex Equiangular Tight Frame (ETF) when trained well past the point of zero training error~\citep{ji2021unconstrained, zhou2022optimization}. \citet{lu2020neural} prove that this occurs when minimizing cross-entropy loss over the unit ball. We extend their proof techniques and show that the optimal solution for minimizing contrastive loss under certain conditions is also the simplex ETF.

\paragraph{Optimal permutations for SGD.}
The performance of SGD without replacement under different permutations of samples has been well studied in the literature~\citep{bottou2009curiously, recht2012beneath, recht2013parallel, nagaraj2019sgd, bicheng2020variance, ahn2020sgd, rajput2020closing, mishchenko2020random, safran2021random, safran2021good, GurbuzbalabanOzdaglarParrilo2021_why, nguyen2021unified, lu2021general, rajput2021permutation, tran21b, lu2022grab, cha2023tighter, cho2023sgda}. One can view batch selection in contrastive learning as a method to choose a specific permutation among the possible $\binom{N}{B}$ mini-batches of size $B$. 
However, it is important to note that these bounds do not indicate an improved convergence rate for general non-convex functions and thus would not apply to the contrastive loss, particularly in the setting where the embeddings come from a shared embedding network. We show that in the case of OSGD \citep{kawaguchi2020ordered}, we can indeed prove that contrastive loss satisfies the necessary conditions in order to guarantee convergence.
\vspace{-2mm}

\section{Problem Setting}\label{sec:prelim}

Suppose we are given a dataset $\{(\vx_i, \vy_i)\}_{i=1}^N$ of $N$ positive pairs (data sample pairs that are conceptually similar or related), where $\vx_i$ and $\vy_i$ are two different \emph{views} of the same object. Note that this setup includes both the multi-modal setting (\eg CLIP~\citep{radford2021learning}) and the uni-modal setting (\eg SimCLR~\citep{chen2020simple}) as follows.
For the multi-modal case, one can view $(\vx_i, \vy_i)$ as two different modalities of the same data, e.g., $\vx_i$ is the image of a scene while $\vy_i$ is the text description of the scene.
For the uni-modal case, one can consider $\vx_i$ and $\vy_i$ as different augmented images from the same image. 

We consider the contrastive learning problem where the goal is to find embedding vectors for $\{\vx_i\}_{i=1}^N$ and  $\{\vy_i\}_{i=1}^N$, such that the embedding vectors of positive pairs $(\vx_i, \vy_i)$ are similar, while ensuring that the embedding vectors of other (negative) pairs are well separated.  
Let $\vu_i \in \sR^d$ be the embedding vector of $\vx_i$, and $\vv_i \in \sR^d$ be
the embedding vector of 
$\vy_i$. In practical settings, one typically considers parameterized encoders so that $\vu_i = f_{\vtheta}(\vx_i)$ and $\vv_i=g_{\vphi}(\vy_i)$. We define embedding matrices $\vU := [\vu_1, \vu_2, \ldots \vu_N]$ and $\vV := [\vv_1, \vv_2, \ldots, \vv_N]$ which are the collections of embedding vectors. Now, we focus on the simpler setting of directly optimizing the embedding vectors instead of model parameters $\vtheta$ and $\vphi$ in order to gain theoretical insights into the learning embeddings. This approach enables us to develop a deeper understanding of the underlying principles and mechanisms. Consider the problem of directly optimizing the embedding vectors for $N$ pairs which is given by 
\begin{equation}\label{prob:standard-full-batch}
    \min_{\vU, \vV}\ \lclip (\vU, \vV) \quad
\text{s.t.} \quad \lVert \vu_i \rVert = 1, \lVert \vv_i \rVert = 1\; \quad \forall i \in [N],
\end{equation}
where $\lVert\cdot \rVert$ denotes the $\ell_2$ norm, the set $[N]$ denotes the set of integers from $1$ to $N$, and the contrastive loss (the standard InfoNCE loss~\cite{oord2018representation}) is defined as
\begin{align}
\label{eq:full_CLIP_loss}
    \gL^{\op{con}}%
    (\vU, \vV)
    := -\frac{1}{N} \sum_{i=1}^N \log \left(\frac{e^{\vu_i^{\intercal} \vv_i}}{\sum_{j=1}^N e^{{\vu}_i^{\intercal} {\vv}_j}}\right)
    -\frac{1}{N} \sum_{i=1}^N \log \left(\frac{e^{\vv_i^{\intercal} \vu_i}}{\sum_{j=1}^N e^{{\vv}_i^{\intercal} {\vu}_j}}\right).
\end{align}
\vspace{-2mm}

Note that $\gL^{\op{con}}(\vU, \vV)$ is the full-batch version of the loss which contrasts all embeddings with each other. However, due to the large computational complexity and memory requirements during optimization, practitioners often consider the following mini-batch version instead. Note that there exist $\binom{N}{B}$ different mini-batches, each of which having $B$ samples. For $k \in \left[\binom{N}{B}\right]$, let $\gB_k$ be the $k$-th mini-batch satisfying $\gB_k \subset [N]$ and $|\gB_k| = B$. Let $\vU_{\gB_k}:=\{\vu_i\}_{i\in\gB_k}$ and $\vV_{\gB_k}:=\{\vv_i\}_{i\in\gB_k}$. 
Then, the contrastive loss for the $k$-th mini-batch is $\gL^{\op{con}}(\vU_{\gB_k}, \vV_{\gB_k})$.
\section{Relationship Between the Optimization for Full-Batch and Mini-Batch}
\label{sec:relationship_full_mini}
Recall that we focus on finding the optimal embedding matrices ($\vU$, $\vV$) that minimize the contrastive loss. In this section, we investigate the relationship between the problem of optimizing the full-batch loss $\gL^{\op{con}}(\vU, \vV)$ and the problem of optimizing the mini-batch loss $\gL^{\op{con}}(\vU_{\gB_k}, \vV_{\gB_k})
   $. 
Towards this goal, we prove three main results, the proof of which are in Appendix~\ref{sec4: theory}. %
\begin{itemize}[leftmargin=0.3cm] 
    \item We derive the optimal solution that minimizes the full-batch loss (Lem.~\ref{thm:standard-full-batch-case1}, Thm.~\ref{thm:standard-full-batch-case2}).
    \item We show that the solution that minimizes the average of $\binom{N}{B}$ mini-batch losses is identical to the one that minimizes the full-batch loss (Prop.~\ref{prop:prop1}, Thm.~\ref{thm:NcB-reaches-etf}).
    \item  We show that minimizing the mini-batch loss summed over only a strict subset of $\binom{N}{B}$ mini-batches can lead to a sub-optimal solution that does not minimize the full-batch loss (Thm.~\ref{thm:sub-batch}). 
\end{itemize}

\subsection{Full-batch Contrastive Loss Optimzation}
In this section, we characterize the optimal solution for the full-batch loss minimization in Eq.~(\ref{prob:standard-full-batch}). We start by providing the definition of the simplex equiangular tight frame (ETF) which turns out to be the optimal solution in certain cases.
The original definition of ETF~\cite{sustik2007existence} is for $N$ vectors in a $d$-dimensional space where $N \geq d+1$~\footnote{See Def.~\ref{def:original-etf} in Appendix~\ref{sec: add def} for the full definition}.
\citet{papyan2020prevalence} defines the ETF for the case where $N \leq d+1$ to characterize the phenomenon of neural collapse. 
In our work, we use the latter definition of simplex ETFs which is stated below.
\begin{definition}[Simplex ETF]
\label{def:nc-simplex-etf}
    We call a set of $N$ vectors $\{\vu_i\}_{i=1}^N$ form a simplex Equiangular Tight Frame (ETF) if
    $\lVert \vu_i \rVert = 1, \forall i \in [N]$ and $\vu_i^{\intercal} \vu_j = -1/(N-1), \forall i\neq j$.
\end{definition}
In the following Lemma, we first prove that the optimal solution of full-batch contrastive learning is the simplex ETF for $N \leq d+1$ which follows almost directly from \citet{lu2020neural}. 
\begin{restatable}[Optimal solution when $N\leq d+1$]{lemma}{thmOne}
\label{thm:standard-full-batch-case1}
Suppose $N\leq d+1$. Then, the optimal solution $(\vU^{\star}, \vV^{\star})$ of the full-batch contrastive learning problem in Eq.~(\ref{prob:standard-full-batch}) satisfies two properties: (i) $\vU^{\star}=\vV^{\star}$, and (ii) the columns of $\vU^{\star}$ form a simplex ETF.
\end{restatable}
Actually, many practical scenarios satisfy $N > d+1$. %
However, the approach used in~\citet{lu2020neural} cannot be directly applied for $N> d+1$, leaving it as an open problem. While solving the open problem for the general case seems difficult, we characterize the optimal solution for the specific case of $N=2d$, subject to the conditions stated below.

\begin{definition}[Symmetric and Antipodal]\label{def:sym_and_anti} Embedding matrices 
$\vU$ and $\vV$ are called \emph{symmetric and antipodal} if $(\vU, \vV)$ satisfies two properties: (i) Symmetric \ie $\vU=\vV$; (ii) Antipodal \ie 
for each $i \in [N]$, there exists $j(i)$ such that $\vu_{j(i)}=-\vu_i$.
\end{definition}

We conjecture that the optimal solutions for $N=2d$ are \emph{symmetric} and \emph{antipodal}. Note that the symmetric property holds for $N\leq d+1$ case, and the antipodality is frequently assumed in geometric problems such as the sphere covering problem in \citep{borodachov2022optimal}.

Thm.~\ref{thm:standard-full-batch-case2} shows that when $N=2d$, the optimal solution for the full-batch loss minimization, under a \emph{symmetric} and \emph{antipodal} configuration, form a cross-polytope which is defined as the following.

\begin{definition}[Simplex cross-polytope]
     We call a set of $N$ vectors $\{\vu\}_{i=1}^N$ form a simplex cross-polytope if, for all $i$, the following three conditions hold: $\|\vu_i\| = 1$; there exists a unique $j$ such that $\vu_i^\intercal\vu_j = -1$; and $\vu_i^\intercal\vu_k = 0$ for all $k \notin \{i, j\}$.
\label{def:cross-polytope}
\end{definition}
\begin{restatable}[Optimal solution when $N=2d$]{theorem}{thmPoly}    
\label{thm:standard-full-batch-case2}
Let 
\begin{align}
(\vU^{\star}, \vV^{\star}) :=\arg\min_{(\vU,\vV)\in\gA}\lclip(\vU,\vV)\quad\text{s.t.}\quad\|\vu_i\|=1, \|\vv_i\|=1\quad\forall i\in[N],    
\end{align}
where $\gA:=\{(\vU,\vV): \vU, \vV\text{ are symmetric and antipodal%
}\}$.
Then, the columns of $\vU^{\star}$ form a simplex cross-polytope for $N=2d$.
\end{restatable}

\emph{Proof Outline.} By the antipodality assumption, we can apply Jensen's inequality to $N-2$ indices without itself ${\vu}_i$ and antipodal point $-\vu_i$ for a given $i \in [N]$. Then we show that the simplex cross-polytope also minimizes this lower bound while satisfying the conditions that make the applications of Jensen's inequality tight. 

For the general case of $N>d+1$, excluding $N=2d,$ we still leave it as an open problem.
\subsection{
Mini-batch Contrastive Loss Optimization
}\label{sec:main_theory}

\begin{figure}
  \centering
  \subfigure[]{
    \centering
    \includegraphics[width=0.4\columnwidth]{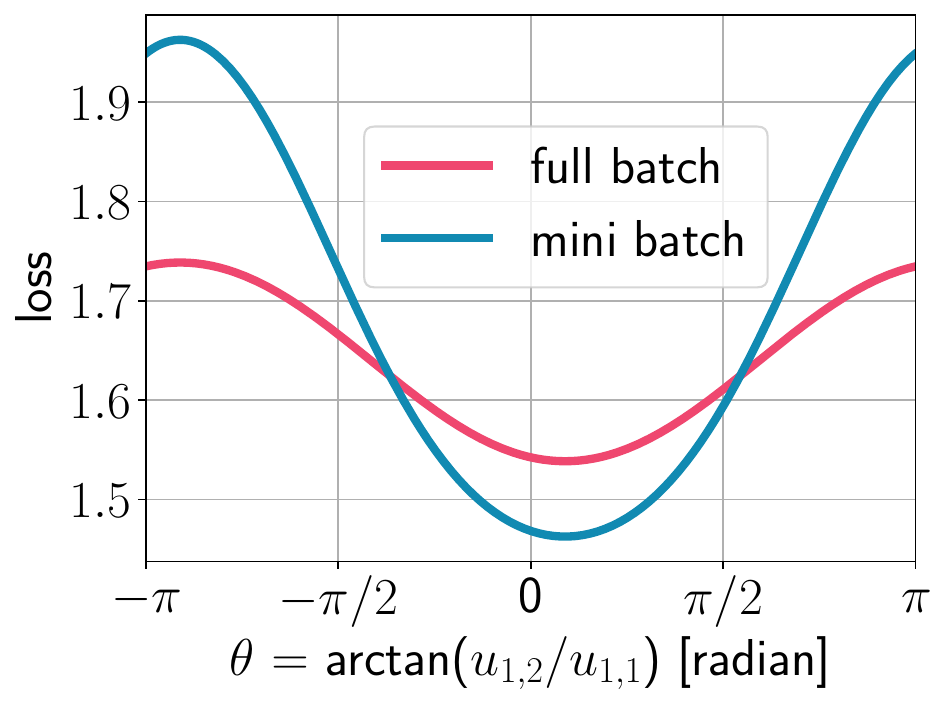}
    \label{fig:ful_mini_loss_compare}
  }
  \subfigure[]{
    \centering
    \includegraphics[width=0.55\columnwidth]{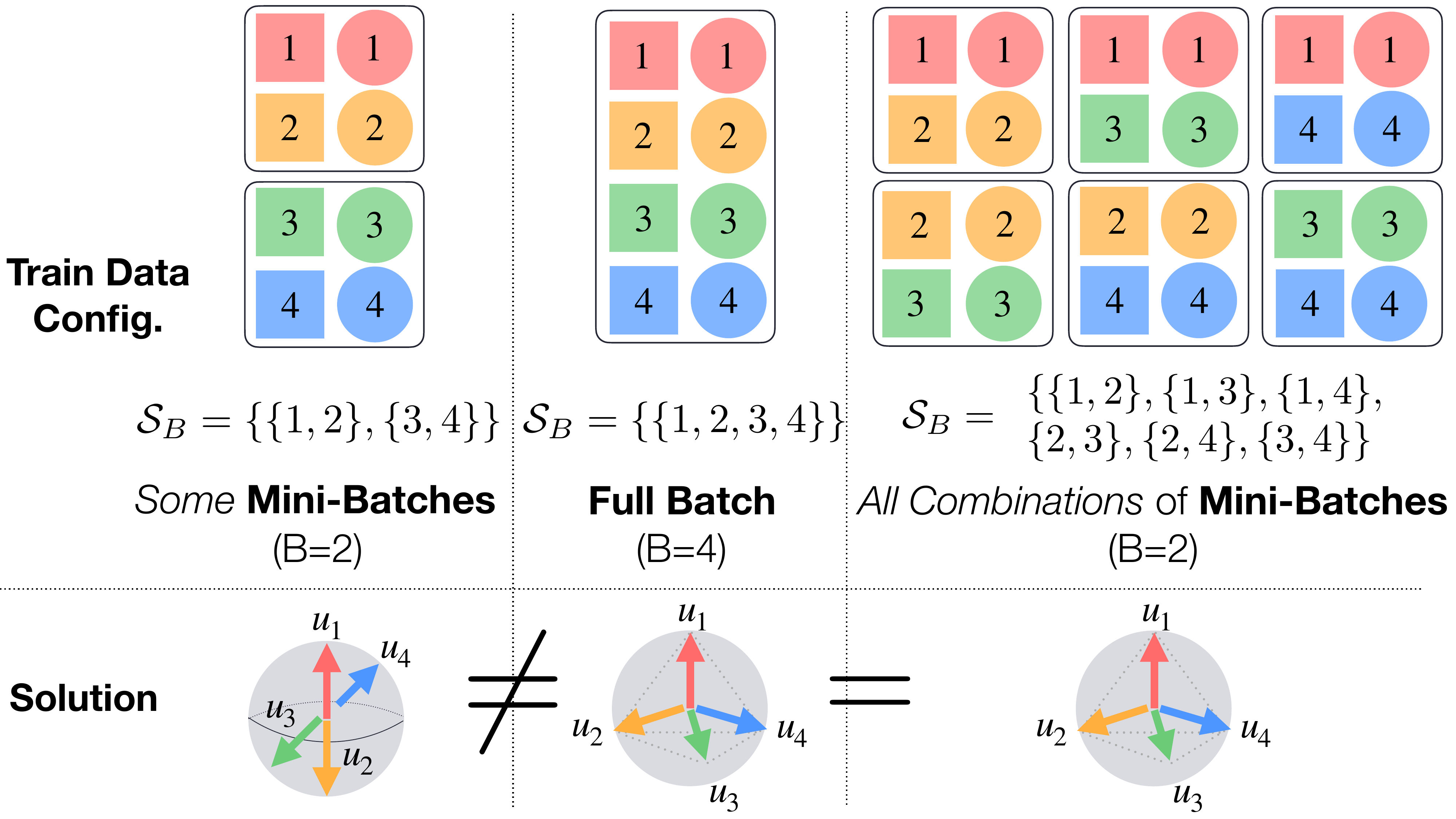}
    \label{fig:concept}
  }
  \caption{(a) Comparing mini-batch loss and full-batch loss when $N=10, B=2$, and $d=2$. We illustrate this by manipulating a single embedding vector $\vu_1$ while maintaining all other embeddings  ($\vv_1$ and $\{\vu_i, \vv_i\}_{i=2}^{10}$) at their optimal solutions. Specifically, $\vu_1 = [u_{1,1}, u_{1,2}]$ is varied as $[\cos(\theta), \sin(\theta)]$ for $\theta\in[-\pi,\pi]$. While the two loss functions are not identical, corroborating Prop.\ref{prop:prop1}, their minimizers align, providing empirical support for Thm.~\ref{thm:NcB-reaches-etf}; (b) The relationship between full-batch and mini-batch optimization in contrastive learning. Consider optimizing $N=4$ pairs of $d=3$ dimensional embedding vectors $\{(\vu_i, \vv_i)\}_{i=1}^N$ where $\vu_i$ and $\vv_i$ are shown as colored square and circle, respectively. The index $i$ is written in the square/circle. The black rounded box represents a batch. We compare three batch selection options: (i) full batch, \ie $B=4$, (ii) all $\binom{N}{B}=6$ mini-batches with size $B=2$, and (iii) some mini-batches. Here, $\gS_B$ is the set of mini-batches where each mini-batch is represented by the set of constituent samples' indices. Our theoretical/empirical findings are: the optimal embedding that minimizes full-batch loss and the one that minimizes the sum of $\binom{N}{B}$ mini-batch losses are identical, while the one that minimizes the mini-batch losses summed over only a strict \textit{subset} of $\binom{N}{B}$ batches does not guarantee the negative correlation between $\vu_i$ and $\vu_j$ for $i \ne j$.
  This illustration is supported by our mathematical results in Thms.~\ref{thm:NcB-reaches-etf} and ~\ref{thm:sub-batch}. 
  }
  \label{fig:full-vs-mini}
\end{figure}

Here we consider the mini-batch contrastive loss optimization problem, 
where we first choose multiple mini-batches of size $B$ and then find $\vU,\vV$ that minimize the sum of contrastive losses computed for the chosen mini-batches. 
Note that this is the loss that is typically considered in the contrastive learning since computing the full-batch loss is intractable in practice.
Let us consider a subset of all possible $\binom{N}{B}$ mini-batches and denote their indices by
$\gS_{B}\subseteq\left[\binom{N}{B}\right]$. For a fixed $\gS_{B}$, the mini-batch loss optimization problem is formulated as:
\begin{equation}
\label{prob:mini-batch}
\min_{\vU, \vV}\ \lclip_{\op{mini}}(\vU,\vV;\gS_B)\quad 
\text{s.t.} \quad \lVert \vu_i \rVert = 1, \lVert \vv_i \rVert = 1\; \quad \forall i \in [N],
\end{equation}
where the loss of given mini-batches is
$    \lclip_{\op{mini}}(\vU,\vV;\gS_B)
:=\frac{1}{|\gS_B|}\sum_{i\in \gS_B} \lclip (\vU_{\gB_i}, \vV_{\gB_i}).$
To analyze the relationship between the full-batch loss minimization in Eq.~(\ref{prob:standard-full-batch}) and the mini-batch loss minimization in Eq.~(\ref{prob:mini-batch}), 
we first compare the objective functions of two problems as below.

\begin{restatable}{proposition}{propOne}
\label{prop:prop1}
The mini-batch loss and full-batch loss are not identical, nor is one a simple scaling of the other by a constant factor. In other words, when $\gS_B=\left[\binom{N}{B}\right]$,
for all $B \ge 2$, there exists no constant $c$ such that $\lclip_{\op{mini}}(\vU,\vV; \gS_B)= c\cdot\lclip(\vU,\vV) \quad \text{for all} \quad \vU, \vV$.
\end{restatable}
\vspace{-2mm}

We illustrate this proposition by visualizing the two loss functions in Fig.~\ref{fig:ful_mini_loss_compare} when $N=10, B=2$, and $d=2$. We visualize it along a single embedding vector $\vu_1$ by freezing all other embeddings ($\vv_1$ and $\{\vu_i, \vv_i\}_{i=2}^{10}$) at the optimal solution and varying $\vu_1 = [u_{1,1}, u_{1,2}]$ as $[\cos(\theta), \sin(\theta)]$ for $\theta \in [-\pi, \pi]$. One can confirm that two losses are not identical (even up to scaling).

Interestingly, the following result shows that the optimal solutions of both problems are identical. 

\begin{restatable}[Optimization with all possible ${\binom{N}{B}}$ mini-batches]{theorem}{thmTwo}    
\label{thm:NcB-reaches-etf}
Suppose $B\geq 2$. The set of minimizers of the $\binom{N}{B}$ mini-batch problem in Eq.~(\ref{prob:mini-batch}) is \emph{the same} as that of the full-batch problem in Eq.~(\ref{prob:standard-full-batch}) for two cases: (i) $N\leq d+1$, and (ii) $N = 2d$ and the pairs ($\vU$, $\vV$) are restricted to those satisfying the conditions stated in Def.~\ref{def:sym_and_anti}.
In such cases, the solutions $(\vU,\vV)$ for the $N \choose B$ mini-batch optimization problem satisfies the following: Case (i) $\{{\vu_i}\}_{i=1}^N$ forms a simplex ETF and ${\vu}_i={\vv}_i$ for all $i\in [N]$; Case (ii): $\{{\vu_i}\}_{i=1}^N$ forms a simplex cross-polytope.

\end{restatable}

\emph{Proof Outline.} Similar to the proof of Lem.~\ref{thm:standard-full-batch-case1}, we bound the objective function from below using Jensen's inequality. Then, we show that this lower bound is equivalent to a scaling of the bound from the proof of Lem.~\ref{thm:standard-full-batch-case1}, by using careful counting arguments. Then, we can simply repeat the rest of the proof to show that the simplex ETF also minimizes this lower bound while satisfying the conditions that make the applications of Jensen's inequality tight. %

Now, we present mathematical results specifying the cases when the solutions of mini-batch optimization and full-batch optimization \emph{differ}. First, we show that when $B=2$, minimizing the mini-batch loss over any strict subset of $\binom{N}{B}$ batches, is not equivalent to minimizing the full-batch loss.

\begin{restatable}[Optimization with fewer than ${\binom{N}{B}}$ mini-batches]{theorem}{thmThree}
\label{thm:sub-batch} Suppose $B=2$ and $N\leq d+1$. Then, the minimizer of Eq.~(\ref{prob:mini-batch}) for $\gS_B \subsetneq \left[{\binom{N}{B}}\right]$ is not the minimizer of the full-batch optimization in Eq.~(\ref{prob:standard-full-batch}).
\end{restatable}
\emph{Proof Outline.} We show that  there exist embedding vectors that are not the simplex ETF, and have a strictly lower objective value.
This implies that the optimal solution of any set of mini-batches that does not contain all $\binom{N}{2}$ mini-batches is not the same as that of the full-batch problem.%

The result of Thm.~\ref{thm:sub-batch} is extended to the general case of $B \geq 2$, under some mild assumption; please check Prop.~\ref{thm:ncb-is-necessary1} and~\ref{thm:ncb-is-necessary2} in Appendix~\ref{sec4: theory}.
Fig.~\ref{fig:concept} summarizes the main findings in this section.%
\section{Ordered Stochastic Gradient Descent for Mini-Batch Contrastive Learning}\label{sec:osgd}

Recall that the optimal embeddings for the full-batch optimization problem in Eq.~(\ref{prob:standard-full-batch}) can be obtained by minimizing the sum of $\binom{N}{B}$ mini-batch losses, according to Thm.~\ref{thm:NcB-reaches-etf}. An easy way of approximating the optimal embeddings is using gradient descent (GD) on the sum of losses for $\binom{N}{B}$ mini-batches, or to use a stochastic approach which applies GD on the loss for a randomly chosen mini-batch. Recent works found that applying GD on selective batches outperforms SGD in some cases~\citep{kawaguchi2020ordered, lu2021general, loshchilov2015online}.
A natural question arises: does this hold for mini-batch \emph{contrastive} learning? Specifically, (i) Is SGD enough to guarantee good convergence on mini-batch contrastive learning?, and (ii) Can we come up with a batch selection method that outperforms vanilla SGD?  
To answer this question:
\begin{itemize}[leftmargin=0.3cm] 
    \item We show that Ordered SGD (OSGD)~\cite{kawaguchi2020ordered} can potentially accelerate convergence compared to vanilla SGD in a demonstrative toy example (Sec.~\ref{subsec:toy_example}). 
    We also show that the convergence results from \citet{kawaguchi2020ordered} can be extended to mini-batch contrastive loss optimization (Sec.~\ref{subsec:osgd_convergence}).
    \item We reformulate the batch selection problem into a min-cut problem in graph theory~\citep{cormen2022introduction}, by considering a graph with $N$ nodes where each node is each positive pair and each edge represents a proxy to the contrastive loss between two nodes. 
    This allows us to devise an efficient batch selection algorithm by leveraging spectral clustering~\citep{ng2001spectral} (Sec.~\ref{subsec:spectral_clustering}). 
\end{itemize}

\subsection{Convergence Comparison in a Toy Example: OSGD vs. SGD}\label{subsec:toy_example}

This section investigates the convergence of two gradient-descent-based methods, OSGD and SGD. The below lemma shows that the contrastive loss is geodesic non-quasi-convex, which implies the hardness of proving the convergence of gradient-based methods for contrastive learning in Eq.~(\ref{prob:standard-full-batch}). 
\begin{restatable}{lemma}{lclipnonquasiconvex} %
\label{lem:lclipnonquasiconvex}
Contrastive loss $\lclip(\vU,\vV)$ %
is a geodesic non-quasi-convex function of $\vU, \vV$ on $\gT = \{ (\vU, \vV): \lVert\vu_i\rVert = \lVert\vv_i\rVert=1, \forall i\in[N]\}$. 
\end{restatable}
We provide the proof in Appendix~\ref{sec5:theory}.

In order to compare the convergence of OSGD and SGD, we focus on a toy example where convergence to the optimal solution is achievable with appropriate initialization. 
Consider a scenario where we have $N=4$ embedding vectors $\{\vu_i\}_{i=1}^N$ with $\vu_i \in \sR^2$. Each embedding vector is defined as 
$\vu_1 =(\cos\theta_1, \sin \theta_1); \vu_2 =(\cos\theta_2, -\sin \theta_2); \vu_3 =(-\cos\theta_3, -\sin \theta_3); \vu_4 =(-\cos\theta_4, \sin \theta_4)$ for parameters $\{ \theta_i \}_{i=1}^n$. Over time step $t$, we consider updating the parameters $\vtheta^{(t)}:=[\theta_1^{(t)},\theta_2^{(t)},\theta_3^{(t)},\theta_4^{(t)}]$ using gradient descent based methods. For all $i$, the initial parameters are set as 
$\theta_i^{(0)}=\epsilon  > 0$, and the other embedding vectors are initialized as $\vv_i^{(0)} = \vu_i^{(0)}$. This setting is illustrated in Fig.~\ref{fig:toy_example_fig}.

\begin{figure}
  \centering
  \subfigure[]{
    \centering
    \includegraphics[width=0.5\columnwidth]{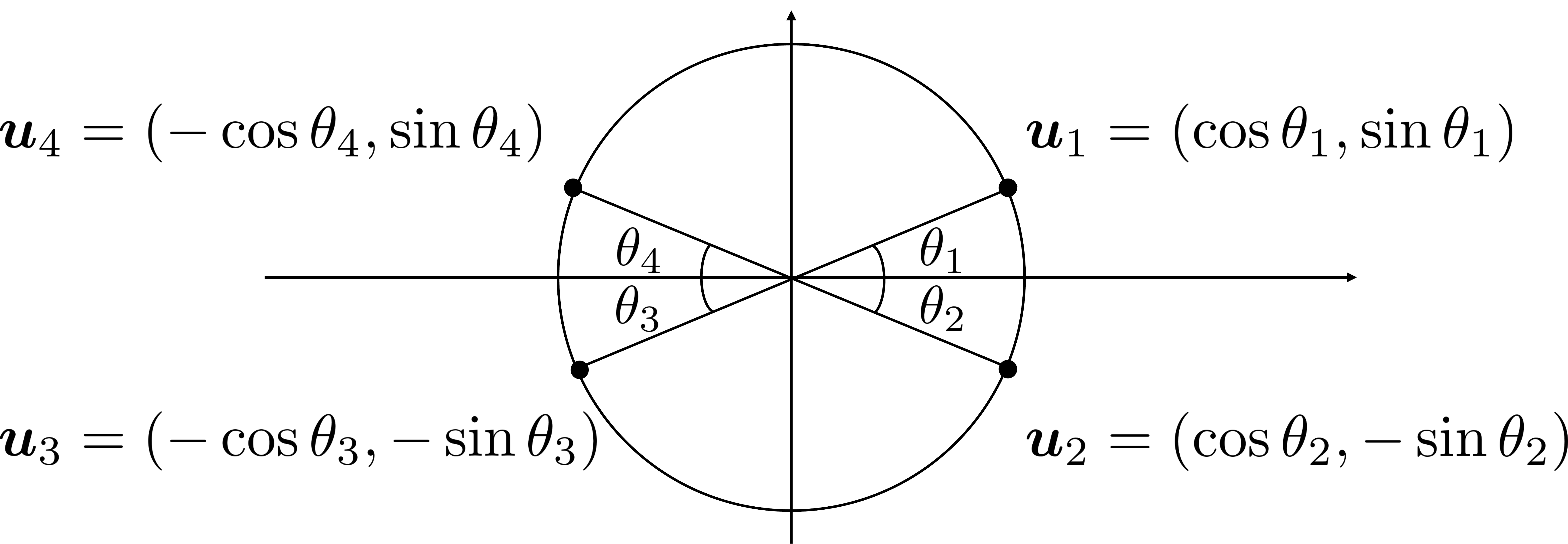}
    \label{fig:toy_example_fig}
  }
  \subfigure[]{
    \centering
    \includegraphics[width=0.45\columnwidth]{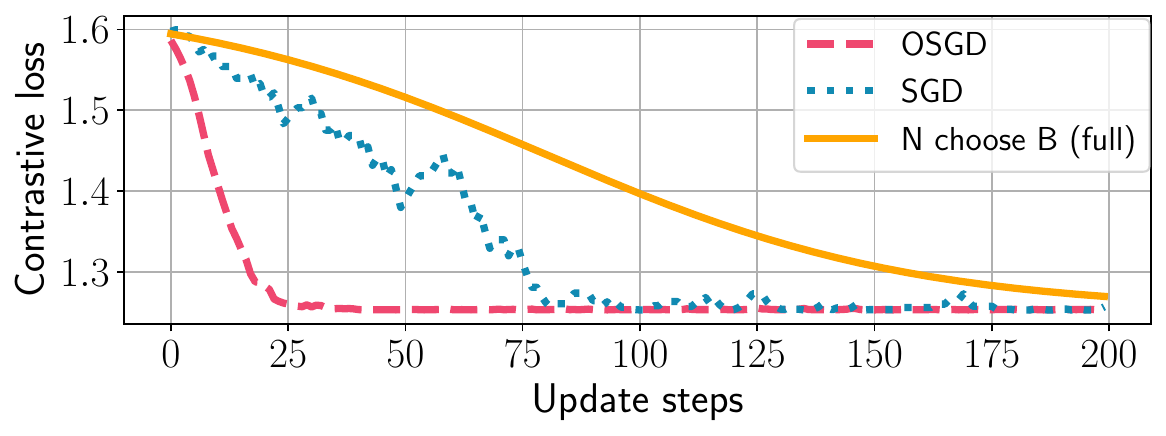}
    \label{fig:toy_example_loss}
  }
  \caption{(a) Toy example considered in Sec.~\ref{subsec:toy_example}; (b) The training loss curves of three algorithms (OSGD, SGD, and $\binom{N}{B}$ full-batch gradient descent) applied on the toy example when $N=4$ and $B=2$. The x-axis represents the number of update steps, while the y-axis displays the loss in Eq.~(\ref{eq:full_CLIP_loss}). OSGD converges the fastest among the three methods.}
  \label{fig:toy_example}
\end{figure}

At each time step $t$, each learning algorithm
begins by selecting a mini-batch $\gB^{(t)}\subset \left\{1,2,3,4\right\}$ with batch size $|\gB^{(t)}|=2$. SGD \emph{randomly} selects a mini-batch, while OSGD selects a mini-batch as follows: $\gB^{(t)} =  \arg \max \limits_{\gB\in \gS} \lclip (\vU_{\gB}(\vtheta^{(t)}), \vV_{\gB} (\vtheta^{(t)}))$.
Then, the algorithms update $\vtheta^{(t)}$ using gradient descent on $\lclip(\vU_{\gB},\vV_{\gB})$ with a learning rate $\eta$: $\vtheta^{(t+1)} =\vtheta^{(t)} - \eta \nabla_{\vtheta} \lclip (\vU_{\gB^{(t)}},\vV_{\gB^{(t)}})$.
For a sufficiently small margin $\rho > 0$, let $T_{\textnormal{OSGD}}, T_{\textnormal{SGD}}$ be the minimal time required for the algorithms to reach the condition $\mathbb{E}[\vtheta^{(T)}] \in (\pi/4-\rho, \pi/4)^{N}$. 
Under this setting, the following theorem compares OSGD and SGD, in terms of the lower bound on the time required for the convergence to the optimal solution.

\begin{restatable}{theorem}{thmToy}
\label{thm:Toy-example}
   Consider the described setting where the parameters $\vtheta^{(t)}$ of embedding vectors are updated, as shown in Fig.~\ref{fig:toy_example_loss}. 
   Suppose there exist $\tilde{\epsilon}$, $\overline{T}$ such that for all $t$ satisfying  $\gB^{(t)}=\left\{1,3\right\}$ or $\left\{2,4\right\}$, $\|\nabla_{\vtheta^{(t)}} \lclip (\vU_{\gB^{(t)}}, \vV_{\gB^{(t)}})\| \leq \tilde{\epsilon}$, and $T_{\textnormal{OSGD}}, \ T_{\textnormal{SGD}}< \overline{T}.$ 
    Then, we have the following inequalities:
    \[
    T_{\textnormal{OSGD}} \geq {\pi/4 - \rho - \epsilon +O(\eta^2 \epsilon + \eta \epsilon^3) \over \eta \epsilon}, \quad
    T_{\textnormal{SGD}} \geq {3(e^2+1) \over e^2-1} {\pi/4-\rho-\epsilon+O(\eta^2 \epsilon+\eta^2 \tilde{\epsilon}) \over \eta \epsilon+O(\eta \epsilon^3+\eta \tilde{\epsilon})}.
    \]
\end{restatable}
\begin{corollary}\label{coro:convergence_comparison}
    Suppose lower bounds of $T_{\textnormal{OSGD}}$, $T_{\textnormal{SGD}}$ in Thm.~\ref{thm:Toy-example} are tight, and the learning rate $\eta$ is small enough. Then, $T_{\textnormal{OSGD}}/T_{\textnormal{SGD}}=(e^2-1)/3(e^2+1)\approx 1/4$.
\end{corollary}
In Fig.~\ref{fig:toy_example_loss}, we present training loss curves 
of the full-batch contrastive loss in Eq.~(\ref{eq:full_CLIP_loss}) for various algorithms implemented on the toy example. One can observe that the losses of all algorithms eventually converge to 1.253, 
the optimal loss achievable when the solution satisfies $\vu_i=\vv_i$ and $\{\vu_i\}_{i=1}^N$ form simplex cross-polytope. As shown in the figure, OSGD converges faster than SGD to the optimal loss. 
This empirical evidence corroborates our theoretical findings in Corollary~\ref{coro:convergence_comparison}. %

\subsection{Convergence of OSGD in Mini-batch Contrastive Learning Setting}\label{subsec:osgd_convergence}
Recall that it is challenging to prove the convergence of gradient-descent-based methods for contrastive learning problem in Eq.~(\ref{prob:standard-full-batch}) due to the non-quasi-convexity of the contrastive loss $\mathcal{L}^{\text{con}}$. Instead of focusing on the contrastive loss, we consider a proxy, the \emph{weighted} contrastive loss defined as $\lcliptilde(\vU, \vV) \coloneqq \frac{1}{q} \sum_{j=1}^{{N \choose B}} \gamma_j \lclip(\vU_{\gB_{(j)}}, \vV_{\gB_{(j)}})$ with $\gamma_j = {\sum_{l=0}^{q-1} {j - 1 \choose l}{{N \choose B}-j \choose k - l - 1}}/{{{N \choose B} \choose k}}$
for two arbitrary natural numbers $k, q \le \binom{N}{B}$ where $\gB_{(j)}$ is a mini-batch with $j$-th largest loss among batches of size $B$.
Indeed, this is a natural objective obtained by applying OSGD to our problem, and we show the convergence of such an algorithm by extending the results in~\citet{kawaguchi2020ordered}.
OSGD updates the embedding vectors using the gradient averaged over $q$ batches that have \emph{the largest losses} among randomly chosen $k$ batches (see Algo.~\ref{alg:osgd-naive} in Appendix~\ref{sec5:theory}).
Let $\vU^{(t)}$, $\vV^{(t)}$ be the updated embedding matrices when applying OSGD for $t$ steps starting from $\vU^{(0)}$, $\vV^{(0)}$, using the learning rate $\eta_t$. 
Then the following theorem, proven in Appendix~\ref{sec5:theory}, holds. 
\begin{restatable}[Convergence results]{theorem}{osgdconvergence}\label{thm:osgd_convergence}
    Consider sampling $t^{\star}$ from $[T-1]$ with probability proportional to $\{\eta_t\}_{t=0}^{T-1}$, that is, $\prob(t^{\star} = t) = {\eta_t}/{(\sum_{i=0}^{T-1} \eta_i)}$. 
    Then $\forall\rho > \rho_0 = 2\sqrt{2/B} + 4e^2 / B$, we have
    \[\expect\left[\left\|\nabla \lcliptilde(\vU^{(t^{\star})}, \vV^{(t^{\star})})\right\|^2 \right] 
          \leq \frac{(\rho + \rho_0)^2}{\rho(\rho-\rho_0)} 
          \frac{\left( \lcliptilde(\vU^{(0)}, \vV^{(0)}) - \lclipstar \right) + 8{\rho}\sum_{t=0}^{T-1}\eta_t^2}{\sum_{t=0}^{T-1}\eta_t},
    \]    
    where $\lclipstar$ denotes the minimized value of $\lcliptilde$.
\end{restatable}
Given sufficiently small learning rate $\eta_t \sim O(t^{-1/2}),$ $\mathbb{E}\|\nabla \lcliptilde\|^2$ decays at the rate of $\widetilde{O}(T^{-1/2}).$ Therefore, this theorem guarantees the convergence of OSGD for mini-batch contrastive learning.

\subsection{Suggestion: Spectral Clustering-based Approach}\label{subsec:spectral_clustering}

\begin{figure}[!t]
\centering
\includegraphics[width=0.98\columnwidth]{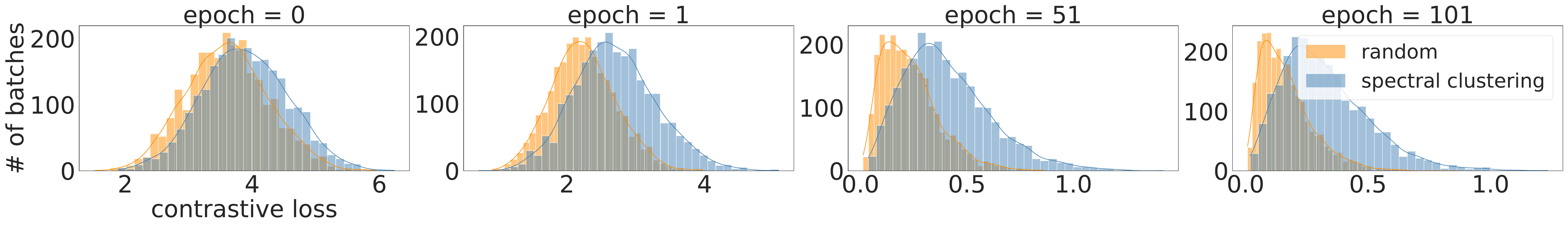}
\caption{Histograms of batch counts for $N/B$ batches, for the contrastive loss measured from ResNet-18 models trained on CIFAR-100 using SGD, where $N$=50,000 and $B$=20. Each plot is derived from a distinct training epoch. Here we compare two batch selection methods: 
(i) randomly shuffling $N$ samples and partition them into $N/B$ batches of size $B$, (ii) our SC method given in Algo.~\ref{alg:spectral_clustering}. 
The histograms show that batches generated through the proposed spectral clustering method tend to contain a higher proportion of large loss values when compared to random batch selection. Similar results are observed in different settings, details of which are given in 
Appendix~\ref{sec:batch_counts_appendix}. %
}
\label{fig:loss_histogram}
\end{figure}
\begin{wrapfigure}{R}{0.47\textwidth}
\vspace{-0.5cm}
\hspace{0.05cm}
\begin{minipage}{0.47\textwidth}
\begin{algorithm}[H]
\footnotesize
   \caption{Spectral Clustering Method}
   \label{alg:spectral_clustering}
   \DontPrintSemicolon
   \KwIn{the number of positive pairs $N$, batch size $B$, 
   embedding matrices: $\vU$, $\vV$}
   \KwOut{selected batches $\{\gB_j\}_{j=1}^{N/B}$}
   \BlankLine
   Construct the affinity matrix $A$:
   $A_{ij}=\mathbbm{1}\{i\neq j\}\times w(i,j)$\;
   Construct the degree matrix $D$ from $A$:
   $D_{ij} = \mathbbm{1}\{i=j\}\times(\sum_{j=1}^NA_{ij})$\;
   $L \leftarrow D-A$; $k \leftarrow N/B$\; 
   $\{\gB_j\}_{j=1}^{N/B}\leftarrow$Apply the even-sized spectral clustering algorithm with $L$ and $k$\; 
   \Return $\{\gB_j\}_{j=1}^{N/B}$
\end{algorithm}
\end{minipage}
\end{wrapfigure}
Applying OSGD to mini-batch contrastive learning has a potential benefit as shown in Sec.~\ref{subsec:toy_example}, but it also has some challenges. Choosing the best $q$ batches with high loss in OSGD is only doable after we evaluate losses of all $\binom{N}{B}$ combinations, which is computationally infeasible for large $N$. A naive solution to tackle this challenge is to first randomly choose $k$ batches and then select $q$ high-loss batches among $k$ batches. However, this naive random batch selection method does not guarantee that the chosen $q$ batches are having the highest loss among all $\binom{N}{B}$ candidates. 
Motivated by these issues of OSGD, we suggest an alternative batch selection method inspired by graph theory. 
Note that the contrastive loss $\mathcal{L}^{\sf{con}}(U_{\gB}, V_{\gB})$ for a given batch $\gB$  is lower bounded as follows: 
{
\begin{align}
    &\begin{aligned}
        & 
        \frac{1}{B(B-1)}\left\{\sum_{i\in\gB}\sum_{j \in \gB \setminus \{i\}}\log\left(1+(B-1)e^{\vu_i^\intercal(\vv_j-\vv_i)}\right)+\log\left(1+(B-1)e^{\vv_i^\intercal(\vu_j-\vu_i)}\right)\right\}.
    \end{aligned}
\end{align}
}This lower bound is derived using Jensen's inequality. Detailed derivation is provided in Appendix~\ref{sec:alg_detail_sc}. A nice property of this lower bound is that it can be expressed as a summation of terms over a pair $(i,j)$ of samples within batch $\gB$. 
Consider a graph $\gG$ with $N$ nodes, where the weight between node $k$ and $l$ is defined as $w(k,l):= \sum_{(i,j)\in\{(k,l), (l,k)\}}\log\left(1+(B-1)e^{\vu_i^\intercal(\vv_j-\vv_i)}\right)+\log\left(1+(B-1)e^{\vv_i^\intercal(\vu_j-\vu_i)}\right)$. 
Recall that our goal is to choose $q$ batches having the highest contrastive loss among $N \choose B$ batches. We relax this problem by reducing our search space such that the $q=N/B$ chosen batches $\gB_1, \cdots, \gB_q$ form a partition of $N$ samples, \ie $\gB_i \cap \gB_j = \varnothing$ and $\cup_{i \in [q]} \gB_i = [N]$. In such scenario, our target problem is equivalent to the problem of clustering $N$ nodes in graph $\gG$ into $q$ clusters with equal size, where the objective is to minimize the sum of weights of inter-cluster edges.
This problem is nothing but the min-cut problem~\citep{cormen2022introduction}, and we can employ even-sized spectral clustering algorithm which solves it efficiently. 
The pseudo-code of our batch selection method\footnote{Our algorithm finds $N/B$ good clusters \emph{at once}, instead of only finding a single best cluster. Compared with such alternative approach, our method is (i) more efficient when we update models for multiple iterations, and (ii) guaranteed to load all samples with $N/B$ batches, thus expected to have better convergence~\citep{bottou2009curiously, haochen2019random, GurbuzbalabanOzdaglarParrilo2021_why}.} is provided in Algo.~\ref{alg:spectral_clustering}, and further details of the algorithm are provided in Appendix~\ref{sec:alg_detail}. Fig.~\ref{fig:loss_histogram} shows the histogram of contrastive loss for $N/B$ batches chosen by the random batch selection method and the proposed \emph{spectral clustering} (SC) method. One can observe that the SC method favors batches with larger loss values.
\vspace{-2mm}
\section{Experiments}\label{sec:exp}
\vspace{-2mm}
We validate our theoretical findings and the effectiveness of our proposed batch selection method by providing experimental results on synthetic and real datasets. 
We first show that our experimental results on synthetic dataset coincide with two main theoretical results: (i) the relationship between the full-batch contrastive loss and the mini-batch contrastive loss given in Sec.~\ref{sec:relationship_full_mini}, (ii) the analysis on the convergence of OSGD and the proposed SC method given in Sec.~\ref{sec:osgd}. 
To demonstrate the practicality of our batch selection method, we provide experimental results on CIFAR-100~\cite{krizhevsky2009learning} and Tiny ImageNet~\cite{le2015tiny}. Details of the experimental setting can be found in Appendix~\ref{sec:exp_detail}, and our code is available at \href{https://github.com/krafton-ai/mini-batch-cl.git}{https://github.com/krafton-ai/mini-batch-cl}.

\begin{figure*}[!t]
\centering

\raisebox{2.4mm}[0mm][0mm]{\rotatebox[origin=l]{90}{\scriptsize{\textsc{$d=2N$}}}}\hspace{1mm}
    \includegraphics[height=.12\columnwidth]{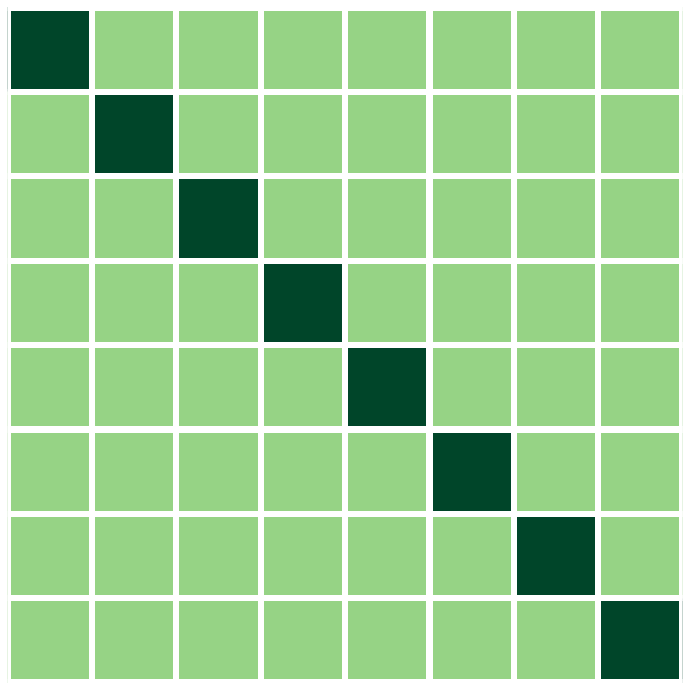}
    \label{fig:N8_d16_B2_ETF_wo_cbar}
    \hspace{1mm}
    \includegraphics[height=.12\columnwidth]{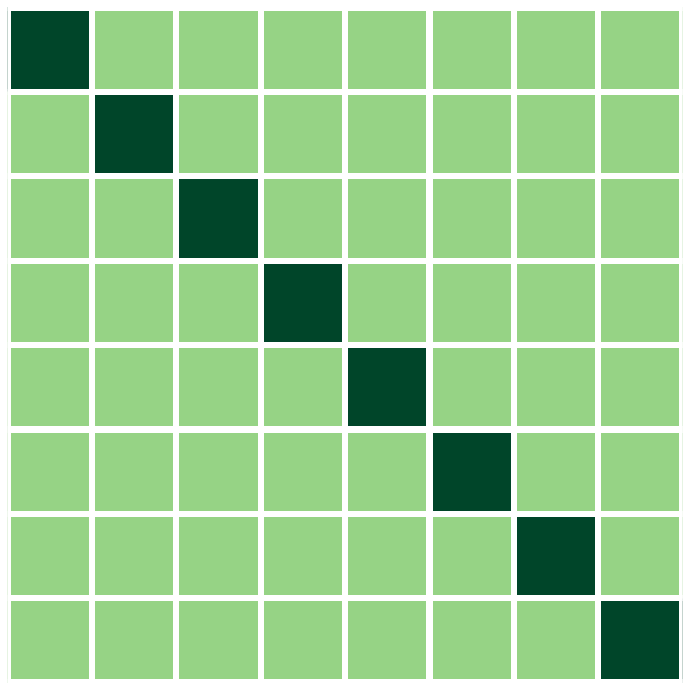}
    \label{fig:d=2N_N8_d16_lr0.5_s20000_z_fixed_mini_batch_B2_full_wo_cbar}
    \hspace{1mm}
    \includegraphics[height=.12\columnwidth]{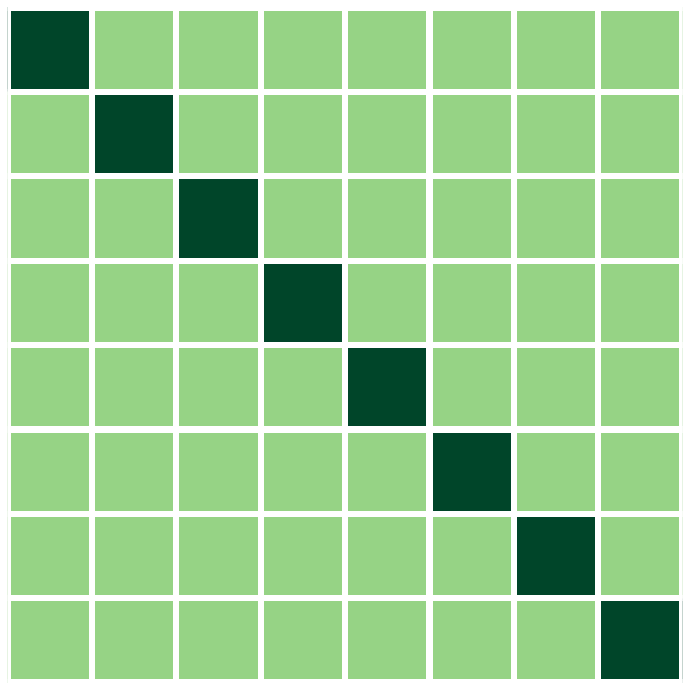}
    \label{fig:d=2N_N8_d16_lr0.5_s20000_z_fixed_mini_batch_B2_NcB_wo_cbar}
    \hspace{1mm}
    \includegraphics[height=.12\columnwidth]{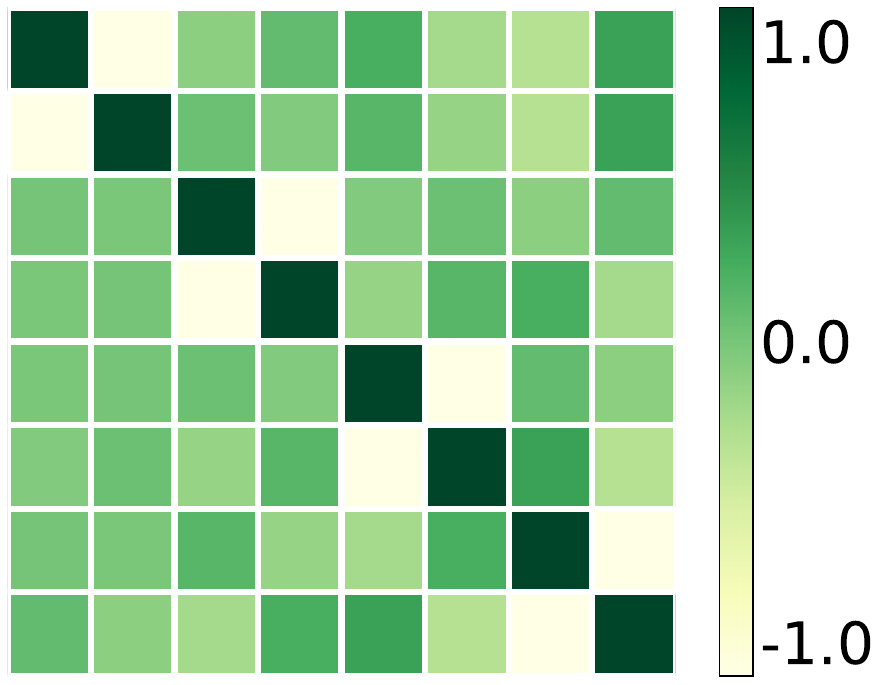}
    \label{fig:d=2N_N8_d16_lr0.5_s20000_z_fixed_mini_batch_B2_f_w_cbar}
    \hspace{1mm}
    \includegraphics[height=.13\columnwidth]{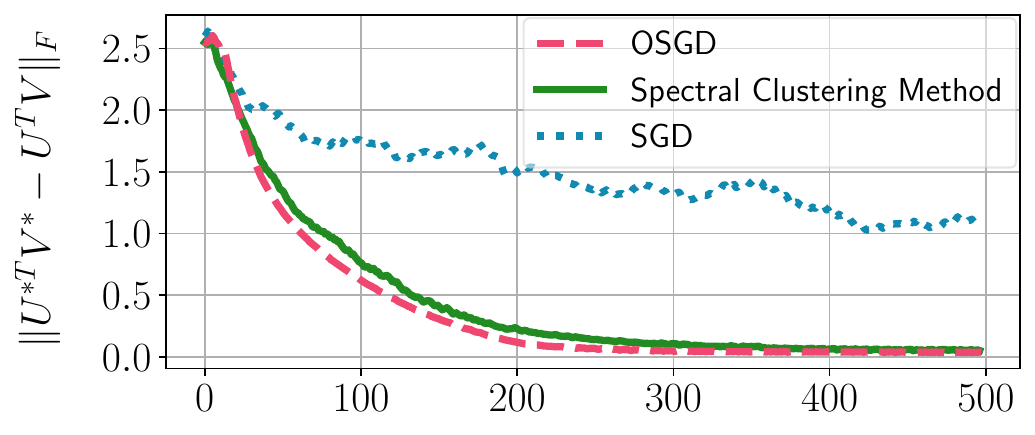}
    \label{fig:d=2N_norm_difference_d_2N}

\raisebox{2.2mm}[0mm][0mm]{\rotatebox[origin=l]{90}{\scriptsize{\textsc{$d=N/2$}}}
\hspace{-2.2mm}}\hspace{1mm}
\subfigure[optimal]{
    \includegraphics[height=.12\columnwidth]{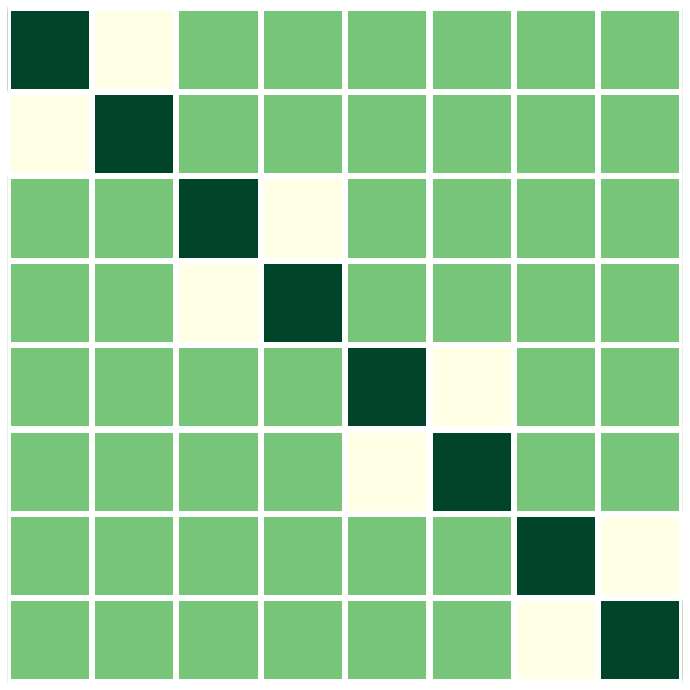}
    \label{fig:N8_d4_B2_CP_wo_cbar}
}
\subfigure[full-batch]{
    \includegraphics[height=.12\columnwidth]{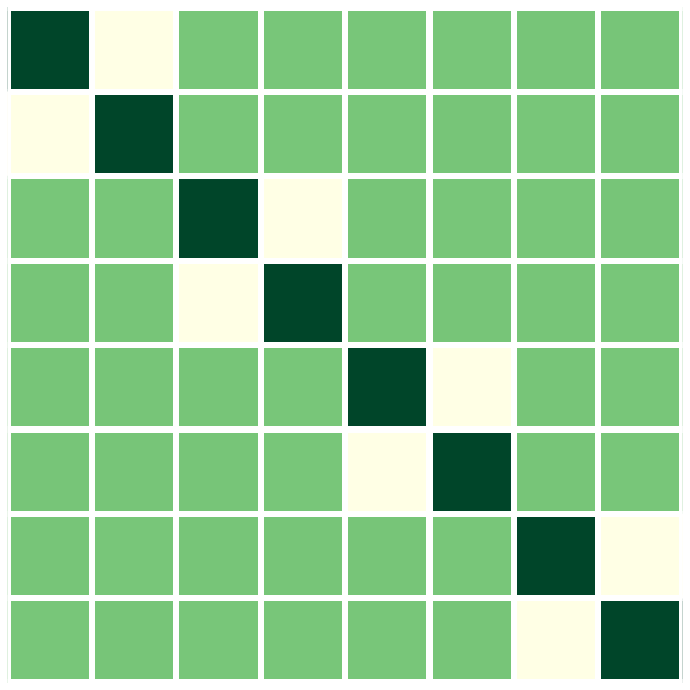}
    \label{fig:d=N_d_2_N8_d4_lr0.5_s20000_z_fixed_mini_batch_B2_full_wo_cbar}
}
\subfigure[${N\choose B}$-all]{
    \includegraphics[height=.12\columnwidth]{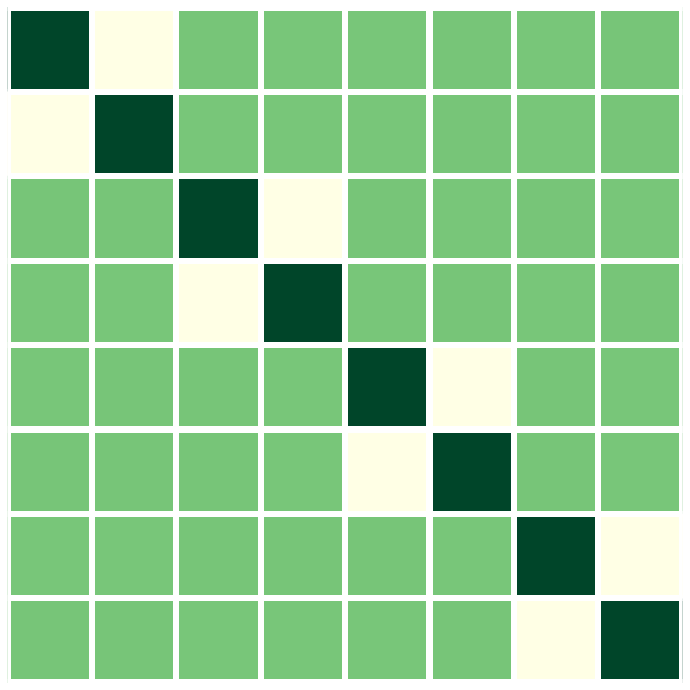}
    \label{fig:d=N_d_2_N8_d4_lr0.5_s20000_z_fixed_mini_batch_B2_NcB_wo_cbar}
}
\subfigure[${N\choose B}$-sub]{
    \includegraphics[height=.12\columnwidth]{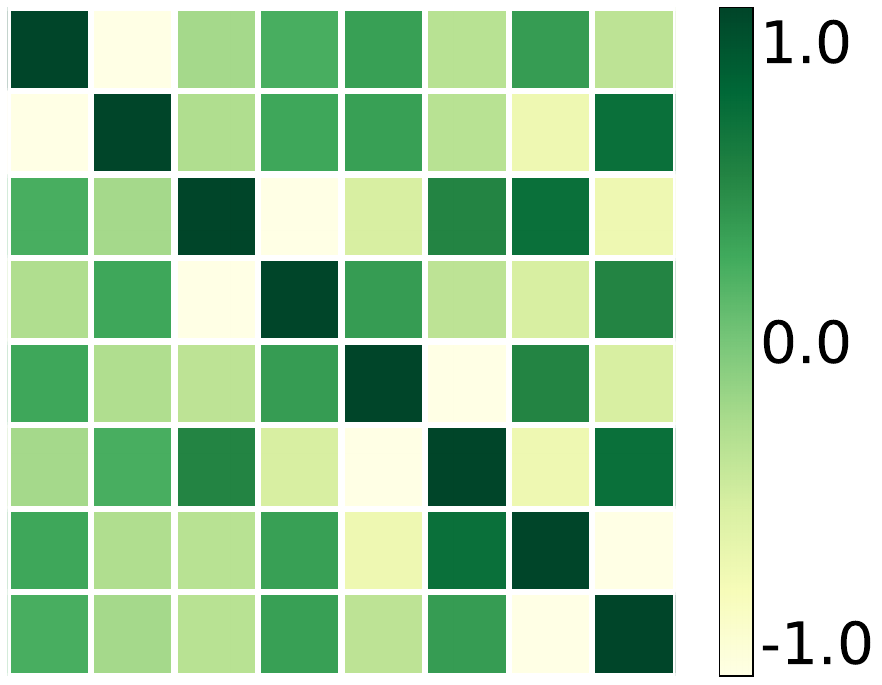}
    \label{fig:d=N_d_2_N8_d4_lr0.5_s20000_z_fixed_mini_batch_B2_f_w_cbar}
}
\subfigure[norm difference]{
    \includegraphics[height=.13\columnwidth]{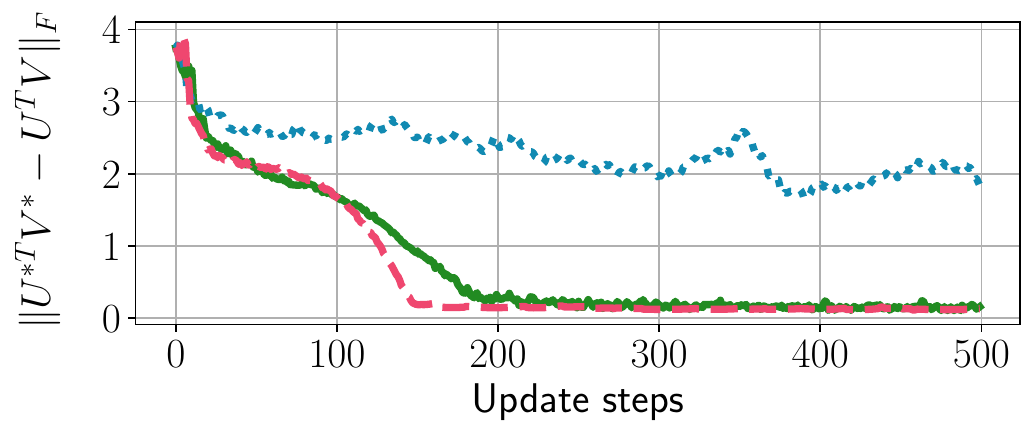}
    \label{fig:d=2N_norm_difference_d_2N}
}

\caption{
The behavior of embedding matrices $\vU, \vV$ optimized by different batch selection methods for $N=8$ and $B=2$ (Top: $d=2N$, Bottom: $d=N/2$). 
(a)-(d): Heatmap of $N \times N$ matrix visualizing the pairwise inner products $\vu_i^\intercal\vv_j$, where
(a): ground-truth solution (ETF for $d=2N$, cross-polytope for $d=N/2$), (b): optimized the full-batch loss with GD, (c): optimized the sum of $\binom{N}{B}$ mini-batch losses with GD, (d): optimized a partial sum of $\binom{N}{B}$ mini-batch losses with GD. 
Note that both (b) and (c) reach the ground-truth solution in (a), while (d) does not, supporting our theoretical results in Sec.~\ref{sec:main_theory}. 
Further, (e) compares the convergence of three mini-batch selection algorithms: 1) SGD, 2) OSGD, and 3) our spectral clustering method, when updating embeddings for 500 steps. OSGD and our method nearly converge to the optimal solution, while SGD does not. Here, $y$-axis represents the Frobenius norm of the difference between the heatmaps of the optimal solution and the updated embeddings, denoted by $\|\vU^{\star \intercal}\vV^{\star}-\vU^{\intercal}\vV\|_F$. 
}
\label{fig:sim_theorem_base_N_8}
\end{figure*}

\subsection{Synthetic Dataset}\label{sec:synthetic_exp}
Consider the problem of optimizing the embedding matrices $\vU, \vV$ using GD, where each column of $\vU, \vV$ is initialized as a multivariate normal vector and then normalized as $\lVert \vu_i \rVert = \lVert \vv_i \rVert = 1$, $\forall i$. We use learning rate $\eta=0.5$, and apply the normalization step at every iteration. 

First, we compare the minimizers of three optimization problems: 
(i) full-batch optimization in Eq.(\ref{prob:standard-full-batch}); 
(ii) mini-batch optimization in Eq.~(\ref{prob:mini-batch}) with $\gS_B=\left[\binom{N}{B}\right]$;
(iii) mini-batch optimization with $\gS_B\subsetneq\left[\binom{N}{B}\right]$.
We apply GD algorithm to each problem for $N=8$ and $B=2$, obtain the updated embedding matrices, and then show the heatmap plot of $N\times N$ gram matrix containing all the pairwise inner products $\vu_i^{\intercal} \vv_j$ in Fig.~\ref{fig:sim_theorem_base_N_8}(b)-(d). Here, we plot for two regimes: $d=2N$ for the top row, and $d=N/2$ for the bottom row. In Fig.~\ref{fig:sim_theorem_base_N_8}(a), we plot the gram matrix for the \emph{optimal} solution obtained in Sec.~\ref{sec:main_theory}. 
One can observe that when either full-batch or all $\binom{N}{B}$ mini-batches are used for training, the trained embedding vectors reach a simplex ETF and simplex cross-polytope solutions for $d=2N$ and $d=N/2$, respectively, as proved in Thm~\ref{thm:NcB-reaches-etf}. In contrast, when a strict subset of $\binom{N}{B}$ mini-batches are used for training,
these solutions are not achieved. 

Second, we compare the convergence speed of three algorithms in mini-batch optimization:
(i) OSGD; (ii) the proposed SC method; and (iii) SGD (see details of the algorithms in Appendix~\ref{sec:alg_detail}). 
Fig.~\ref{fig:d=2N_norm_difference_d_2N} shows the $\|\vU^{\star \intercal}\vV^{\star}-\vU^{(t)\intercal}\vV^{(t)}\|_F$ which is the Frobenius norm of the difference between heatmaps of the ground-truth solution ($\vU^{\star}, \vV^{\star}$) and the embeddings at each step $t$. We restrict the number of updates for all algorithms, specifically 500 steps. We observe that both OSGD and the proposed method nearly converge to the ground-truth solutions proved in Thm.~\ref{thm:NcB-reaches-etf} within 500 steps, while SGD does not. We obtain similar results for other values of $N$ and $d$, given
in Appendix~\ref{sec:synthetic_exp_appendix}.

\begin{table}[t]
    \begin{center}
    \caption{Top-1 retrieval accuracy on CIFAR-100-C (or Tiny ImageNet-C)~\citep{hendrycks2019benchmarking}, when each algorithm uses CIFAR-100 (or Tiny ImageNet) to pretrain ResNet-18 with SimCLR and SogCLR objective. 
    SC algorithm proposed in Sec.~\ref{subsec:spectral_clustering} outperforms all baselines.}
    \label{tab:performance_comparison_corrupted}
    \resizebox{0.55\textwidth}{!}{
    \begin{tabular}{@{}l*{4}{c}@{}}
    \toprule
     & \multicolumn{2}{c}{CIFAR-100} & \multicolumn{2}{c}{Tiny ImageNet} \\
    \cmidrule(lr){2-3} \cmidrule(lr){4-5}
     & SimCLR & SogCLR & SimCLR & SogCLR \\
    \midrule
    OSGD & 31.4 $\pm$ 0.03 & 23.8 $\pm$ 0.02 & 33.6 $\pm$ 0.04 & 29.7 $\pm$ 0.04 \\
    SGD  & 31.3 $\pm$ 0.02 & 23.6 $\pm$ 0.05 & 33.2 $\pm$ 0.03 & 28.6 $\pm$ 0.03 \\
    SC & $\bm{32.5}$ $\pm$ 0.05 & $\bm{30.0}$ $\pm$ 0.04 & $\bm{33.8}$ $\pm$ 0.04 & $\bm{33.3}$ $\pm$ 0.03 \\
    \bottomrule
    \end{tabular}}
    \end{center}
\end{table}

\subsection{Real Datasets}\label{sec:real_exp}
Here we show that the proposed SC method is effective in more practical settings where the embedding is learned by a parameterized encoder, and can be easily applied to existing uni-modal frameworks, such as SimCLR~\cite{chen2020simple} and SogCLR~\cite{yuan2022provable}. We conduct mini-batch contrastive learning on CIFAR-100 and Tiny ImageNet datasets and report the performances in the image retrieval downstream task on corrupted datasets, the results of which are in Table~\ref{tab:performance_comparison_corrupted}. Due to the page limit, we provide detailed experimental information in the Appendix~\ref{sec:real_exp_appendix}.

\section{Conclusion}\label{sec:conclusion}
We provided a thorough theoretical analysis of mini-batch contrastive learning. First, we showed that the solution of mini-batch optimization and that of full-batch optimization are identical if and only if all ${N\choose B}$ mini-batches are considered. Second, we analyzed the convergence of OSGD and devised spectral clustering (SC) method, a new batch selection method which handles the complexity issue of OSGD in mini-batch contrastive learning. Experimental results support our theoretical findings and the efficacy of SC. 

\section*{Limitations}\label{sec:limitations}
We note that our theoretical results have two major limitations:
\begin{enumerate}[leftmargin=0.7cm]
    \item While we would like to extend our results to the general case of $N > d+1$, we were only able to characterize the optimal solution for the specific case of $N=2d$. Furthermore, our result for the case of $N=2d$ in Thm.~\ref{thm:NcB-reaches-etf} requires the use of the conjecture that the optimal solution is symmetric and antipodal. However, as mentioned by \citet{lu2020neural}, the general case of $N > d+1$ seems quite challenging in the non-asymptotic regime.
    \item In practice, the embeddings are usually the output of a shared neural network encoder. However, our results are for the case when the embeddings only have a norm constraint. 
    Thus, our results do not readily indicate any generalization to unseen data. We expect however, that it is possible to extend our results to the shared encoder setting by assuming sufficient overparameterization.
\end{enumerate}

{
\small

\bibliography{main.bib}
\bibliographystyle{icml2022}
}

\appendix

\newpage

\onecolumn

\section*{Organization of the Appendix}
\begin{enumerate}[leftmargin=0.7cm]
    \item In Appendix~\ref{sec: add def}, we introduce an additional definition for posterity.
    \item In Appendix~\ref{sec:theory-appendix}, we provide detailed proofs of the theoretical results as well as any intermediate results/lemmas that we found useful.
    \begin{enumerate}
        \item Appendix~\ref{sec4: theory} provides proofs of the results from Section~\ref{sec:relationship_full_mini} which focuses on the relationship between the optimal solutions for minimizing the mini-batch and full-batch constrastive loss.
        \item Appendix~\ref{sec5:theory} contains the proofs of results from Section~\ref{sec:osgd} which concern the application of Ordered SGD to mini-batch contrastive learning.
        \item Appendix~\ref{sec:aux-lem} is intended to supplement Appendix~\ref{sec5:theory}. It contains auxiliary notation and proofs required in the proof of Theorem~\ref{thm:osgd_convergence}.
    \end{enumerate}
    \item Appendix~\ref{sec:alg_detail} specifies the pseudo-code and details for the three algorithms: (i) Spectral Clustering; (ii) Stochastic Gradient Descent (SGD) and (iii) Ordered SGD (OSGD).
    \item Appendix~\ref{sec:exp_detail} describes the details of the experimental settings from Section~\ref{sec:exp} while also providing some additional results.
\end{enumerate}

\section{Additional Definition} \label{sec: add def}
\begin{definition}[\citet{sustik2007existence}]
\label{def:original-etf}
    A set of $N$ vectors $\{\vu_i\}_{i=1}^N$ in the $\mathbb{R}^d$ form an equiangular tight frame (ETF) if (i) they are all unit norm: $\lVert \vu_i \rVert = 1$ for every $i \in [N]$,
    (ii) they are equiangular: $\lVert \vu_i^{\intercal} \vu_j \rVert = \alpha \ge 0$ for all $i \ne j$ and some $\alpha \geq 0$, and
     (iii) they form a tight frame: $\vU \vU^{\intercal} = (N/d) \mathbb{I}_d$ where $\vU$ is a $d\times N$ matrix whose columns are $\vu_1, \vu_2, \dots, \vu_N$, and $\mathbb{I}_d$ is the $d\times d$ identity matrix.
\end{definition}

\section{Proofs}\label{sec:theory-appendix}
\subsection{Proofs of Results From Section~\ref{sec:relationship_full_mini}} \label{sec4: theory}
\thmOne*

\begin{proof}
First, we define the contrastive loss as the sum of two symmetric one-sided contrastive loss terms to simplify the notation.
We denote the following term as the one-sided contrastive loss
\begin{equation}
\label{eq:one-side-cliploss}
    \gL(\vU, \vV)  = \frac{1}{N} \sum_{i=1}^N -\log \left(\frac{e^{\vu_i^{\intercal} \vv_i}}{\sum_{j=1}^N e^{\vu_i^{\intercal} \vv_j}}\right).
\end{equation}

Then, the overall contrastive loss is given by the sum of the two one-sided contrastive losses:
\begin{equation}
    \label{eq:sum-one-side-cliplosses}
    \gL^{\op{con}}(\vU, \vV)  = \gL(\vU, \vV) + \gL(\vV, \vU).
\end{equation}
Since $\gL^{\op{con}}$ is symmetric in its arguments, results pertaining to the optimum of $\gL(\vU, \vV)$ readily extend to $\gL^{\op{con}}$.
Now, let us consider the simpler problem of minimizing the one-sided contrastive loss from Eq.~(\ref{eq:one-side-cliploss}) which reduces the problem to exactly the same setting as \citet{lu2020neural}:
    \begin{align*}
        \gL(\vU, \vV) &= \frac{1}{N} \sum_{i=1}^N -\log \left(\frac{e^{\vu_i^{\intercal} \vv_i}}{\sum_{j=1}^N e^{\vu_i^{\intercal} \vv_j}}\right)\\
        &= \frac{1}{N} \sum_{i=1}^N \log \left(1 + \sum_{{j=1,j\neq i}}^N e^{(\vv_j - \vv_i)^{\intercal} \vu_i}\right).
    \end{align*}
Note that, we have for any fixed $1\leq i \leq N,$
\begin{align}
    \sum_{{j=1,j\neq i}}^N e^{(\vv_j - \vv_i)^{\intercal} \vu_i} &= e^{-(\vv_i^{\intercal} \vu_i)} \sum_{{j=1,j\neq i}}^N e^{\vv_j^{\intercal} \vu_i} \nonumber\\
    &= (N-1) e^{-(\vv_i^{\intercal} \vu_i)} \left(\frac{1}{N-1}\right) \sum_{{j=1,j\neq i}}^N e^{\vv_j^{\intercal} \vu_i} \nonumber\\
    &\overset{(a)}{\geq} (N-1)e^{-(\vv_i^{\intercal} \vu_i)}\exp\left(\frac{1}{N-1} \sum_{{j=1,j\neq i}}^N \vv_j^{\intercal} \vu_i \right) \nonumber\\
    &\overset{(b)}{=} (N-1)e^{-(\vv_i^{\intercal} \vu_i)} \exp\left(\frac{\vv^{\intercal} \vu_i - \vv_i^{\intercal} \vu_i}{N-1}\right)\nonumber\\
    &= (N-1) \exp\left(\frac{\vv^{\intercal} \vu_i - N(\vv_i^{\intercal} \vu_i)}{N-1}\right) \label{eq:jensen1},
\end{align}
where $(a)$ follows by applying Jensen inequality for $e^t$ and $(b)$ follows from $\vv:=\sum_{i=1}^N\vv_i$. Since $\log(\cdot)$ is monotonic, we have that $x > y \Rightarrow \log(x) > \log(y)$ and therefore,
\begin{align}
    \gL(\vU, \vV) &\geq {1 \over N} \sum_{i=1}^N \log \left[1 + (N-1)\exp\left(\frac{\vv^{\intercal} \vu_i}{N-1} - \frac{N(\vv_i^{\intercal} \vu_i)}{N-1} \right)\right] \nonumber\\
    &\overset{(c)}{\geq} \log \left[ 1 + (N-1)\exp\left(\frac{1}{N} \sum_{i=1}^N \left( \frac{\vv^{\intercal} \vu_i}{N-1} - \frac{N(\vv_i^{\intercal} \vu_i)}{N-1}\right) \right)\right] \nonumber\\
    &\overset{(d)}{=} \log\left[1 + (N-1)\exp\left(\frac{1}{N}\left(\frac{\vv^{\intercal} \vu}{N-1} - \frac{N}{N-1}\sum_{i=1}^N(\vv_i^{\intercal} \vu_i)\right)\right) \right], \label{eq:jensen2}
\end{align}
where $(c)$ follows by applying Jensen inequality to the convex function $\phi(t) = \log(1 + ae^{bt})$ for $a, b > 0$, and $(d)$ follow from $\vu := \sum_{i=1}^N \vu_i $. 

Note that for equalities to hold in Eq.~(\ref{eq:jensen1}) and~(\ref{eq:jensen2}), we need constants $c_i, c$ such that
\begin{align}
    &\vv_j^{\intercal} \vu_i = c_i \quad \forall j \neq i \label{eq:cond1},\\
    &\frac{\vv^{\intercal} \vu_i}{N-1} - \frac{N(\vv_i^{\intercal} \vu_i)}{N-1} = c \quad \forall i \in [N] \label{eq:cond2}.
\end{align}
Since $\log(\cdot)$ and $\exp(\cdot)$ are both monotonic, minimizing the lower bound in Eq.~(\ref{eq:jensen1}) is equivalent to
\begin{align}
    \min \quad &\frac{\vv^{\intercal} \vu}{N-1} - \frac{N}{N-1}\sum_{i=1}^N \vv_i^{\intercal} \vu_i \nonumber\\
    \Leftrightarrow \max \quad &N\sum_{i=1}^N \vv_i^{\intercal} \vu_i - \Big(\sum_{i=1}^N \vv_i\Big)^{\intercal} \Big(\sum_{i=1}^N \vu_i\Big). \label{eq:final-upper-bound}
\end{align}
All that remains is to show that the solution that maximizes Eq~\ref{eq:final-upper-bound} also satisfies the conditions in Eq.~(\ref{eq:cond1}) and~(\ref{eq:cond2}). To see this, first note that the maximization problem can be written as
\begin{align*}
    \max \quad \vv_{\text{stack}}^{\intercal} ((N\mathbb{I}_N - \mathbf{1}_N \mathbf{1}_N^{\intercal}) \otimes \mathbb{I}_d) \vu_{\text{stack}}
\end{align*}
where $\vv_{\text{stack}} = (\vv_1, \vv_2, \dots, \vv_n)$ is a vector in $\mathbb{R}^{Nd}$ formed by stacking the vectors $\vv_i$ together. $\vu_{\text{stack}}$ is similarly defined. $\mathbb{I}_N$ denotes the $N\times N$ identity matrix, $\mathbf{1}_N$ denotes the all-one vector in $\mathbb{R}^n$, and $\otimes$ denotes the Kronecker product. It is easy to see that $\lVert\vu_{\text{stack}}\rVert=\lVert\vv_{\text{stack}}\rVert = \sqrt{N}$ since each $\lVert \vu_i \rVert=\lVert \vv_i \rVert = 1$. Since the eigenvalues of $A \otimes B$ are the product of the eigenvalues of $A$ and $B$, in order to analyze the spectrum of the middle term in the above maximization problem, it suffices to just consider the eigenvalues of $(N\mathbb{I}_N - \mathbf{1}_N \mathbf{1}_N^{\intercal})$. As shown by the elegant analysis in \citet{lu2020neural}, $(N\mathbb{I}_N - \mathbf{1}_N \mathbf{1}_N^{\intercal}) \vp = N\vp$ for any $\vp \in \mathbb{R}^N$ such that $\sum_{i=1}^N \vp_i = 0$ and $(N\mathbb{I}_N - \mathbf{1}_N \mathbf{1}_N^{\intercal}) \vq = 0$ for any $\vq \in \mathbb{R}^N$ such that $\vq = k \mathbf{1}_N$ for some $k \in \mathbb{R}$. Therefore it follows that its eigenvalues are $N$ with multiplicity $(N-1)$ and $0$. Since its largest eigenvalue is $N$ and since $\lVert \vu_{\text{stack}} \rVert = \lVert \vv_{\text{stack}} \rVert = \sqrt{N}$, applying cauchy schwarz inequality, we have that
\begin{align*}
    &\max \quad \vv_{\text{stack}}^{\intercal} (N\mathbb{I}_N - \mathbf{1}_N \mathbf{1}_N^{\intercal}) \otimes \mathbb{I}_d) \vu_{\text{stack}}^{\intercal}\\
    &= \lVert \vv_{\text{stack}} \rVert \cdot \lVert (N\mathbb{I}_n - \mathbf{1}_n \mathbf{1}_n^{\intercal}) \otimes \mathbb{I}_d) \rVert \cdot \lVert \vu_{\text{stack}} \rVert\\
    &= \sqrt{N} (N) \sqrt{N}\\
    &= N^2.
\end{align*}
Moreover, we see that setting $\vu_i = \vv_i$ and setting $\{\vu_i\}_{i=1}^N$ to be the simplex ETF attains the maximum above while also satisfying the conditions in Eq.~(\ref{eq:cond1}) and~(\ref{eq:cond2}) with $c_i = -1/(N-1)$ and $c=-N/(N-1)$. Therefore, the inequalities in Eq.~(\ref{eq:jensen1}) and~(\ref{eq:jensen2}) are actually equalities for $\vu_i = \vv_i$ when they are chosen to be the simplex ETF in $\mathbb{R}^d$ which is attainable since $d \geq N-1$. Therefore, we have shown that if $\vU^\star = \{\vu_i^{\star}\}_i=1^N$ is the simplex ETF and $\vu_i^{\star} = \vv_i^{\star}\; \forall i \in [N]$, then
$\vU^\star, \vV^\star = arg\min_{\vU, \vV} \gL(\vU, \vV)$ over the unit sphere in $\mathbb{R}^n$. All that remains is to show that this is also the minimizer for $\lclip$.

First note that $\vU^\star, \vV^\star$ is also the minimizer for $\gL(\vV, \vU)$ through symmetry. One can repeat the proof exactly by simply exchanging $\vu_i$ and $\vv_i$ to see that this is indeed true. Now recalling Eq.~(\ref{eq:sum-one-side-cliplosses}), we have
\begin{align}
    \min \lclip &= \min{(\gL(\vU, \vV) + \gL(\vU, \vV))} \nonumber\\
    &\geq \min{(\gL(\vU, \vV))} + \min{(\gL(\vU, \vV))} \label{eq:sum-one-side-loss-ineq}\\
    &= \gL(\vU^\star, \vV^\star) + \gL(\vV^\star, \vU^\star). \nonumber
\end{align}
However, since the minimizer of both terms in Eq.~(\ref{eq:sum-one-side-loss-ineq}) is the same, the inequality becomes an equality. Therefore, we have shown that ($\vU^\star, \vV^\star)$ is the minimizer of $\lclip$ completing the proof.
\end{proof}

\begin{remark}
In the proof of the above Lemma, we only show that the simplex ETF attains the minimum loss  in Eq.~(\ref{prob:standard-full-batch}), but not that it is the only minimizer. The proof of \citet{lu2020neural} can be extended to show that this is indeed true as well. We omit it here for ease of exposition.
\end{remark}

\thmPoly*
\begin{proof}
By applying the logarithmic property that allows division to be represented as subtraction,
\begin{align*}
    \gL(\vU, \vV)&= - {1 \over N} \sum_{i=1}^{N} \log \left(\frac{e^{\vu_i^{\intercal}\vv_i}}{\sum_{j=1}^{N} e^{\vu_i^{\intercal} \vv_j}}\right) \\
    &=-{1 \over N} \sum_{i=1}^{N} \left[ \vu_i^{\intercal}\vv_i - \log \Big( \sum_{j=1}^{N} e^{\vu_i^{\intercal} \vv_j}\Big)\right].
\end{align*}
Since $\vU= \vV$ (\emph{symmetric} property), the contrastive loss satisfies 
\begin{align}
     \lclip(\vU, \vV)
     &= 2\gL(\vU, \vU) \nonumber \\
     &=-{2 \over N} \sum_{i=1}^{N} \left[ \vu_i^{\intercal}\vu_i - \log \Big( \sum_{j=1}^{N} e^{\vu_i^{\intercal} \vu_j}\Big)\right] \nonumber \\
     &= -2+ {2 \over N} \sum_{i=1}^{N} \log\big(\sum_{j=1}^{N}
     e^{\vu_i^{\intercal}\vu_j}\big) \label{lclip_summary}.
\end{align}
Since $\norm{\vu_i}=1$ for any $i \in [N]$, we can derive the following relations:
\begin{align*}
    \norm{\vu_i-\vu_j}^2 = 2-2\vu_i^{\intercal}\vu_j, \quad
    \vu_i^{\intercal}\vu_j = 1-{\norm{\vu_i-\vu_j}^2 \over 2}.
\end{align*}
We incorporate these relations into Eq.~(\ref{lclip_summary}) as follows:
\begin{align*}
    \lclip(\vU, \vV)
    &=-2+{2 \over N} \sum_{i=1}^{N} \log \big(\sum_{j=1}^{N} e^{1-\norm{\vu_i-\vu_j}^2/2}\big)\\
    &={2 \over N}  \sum_{i=1}^{N} \log \big(\sum_{j=1}^{N} e^{-\norm{\vu_i-\vu_j}^2/2}\big).
\end{align*}
The antipodal property of $\vU$ indicates that for each $i \in[N]$, there exists a $j(i)$ such that $u_{j(i)}=-u_i$. By applying this property, we can manipulate the summation of $e^{-\norm{\vu_i-\vu_j}^2/2}$ over $j$ as the following: 
\begin{align*}
    \sum_{j=1}^{N} e^{-\norm{\vu_i-\vu_j}^2/2} &= e^{-\norm{\vu_i-\vu_i}^2/2}+e^{-\norm{\vu_i-\vu_{j(i)}}^2/2}+\sum_{j \neq i, j(i)} e^{-\norm{\vu_i-\vu_j}^2/2}\\
    &=1+e^{-2}+\sum_{j \neq i, j(i)} e^{-\norm{\vu_i-\vu_j}^2/2}.
\end{align*}
Therefore,
 \begin{align}
    &\lclip(\vU, \vV) = {2 \over N} \sum_{i=1}^{N} \log \Big( 1+e^{-2}+\sum_{j \neq i, j(i)} e^{-\norm{\vu_i-\vu_j}^2/2} \Big) \nonumber\\ 
    &\overset{(a)}{\geq} {2 \over N(N-2)} \sum_{i=1}^{N} \sum_{j \neq i, j(i)} \log\big(1+e^{-2}+(N-2)e^{-\norm{\vu_i-\vu_j}^2/2}\big) \nonumber \\
    &= {2 \over N(N-2)} \sum_{i=1}^{N} \sum_{j \neq i} \log\big(1+e^{-2}+(N-2)e^{-\norm{\vu_i-\vu_j}^2/2}\big) - {2 \over N-2} \log(1+(N-1)e^{-2}) \nonumber\\
    &\overset{(b)}{\geq} {2 \over N(N-2)} \sum_{i=1}^{N} \sum_{j \neq i} \log\big(1+e^{-2}+(N-2)e^{-\norm{\vu_i^{\star}-\vu_j^{\star}}^2/2}\big) - {2 \over N-2} \log(1+(N-1)e^{-2}), \nonumber
\end{align}
where (a) follows by applying Jensen's inequality to the concave function $f(t)=\log(1+e^{-2}+t)$; and (b) follows by Lem.~\ref{lemma:elementary_cohn-kumar}, and the fact that function $g(t)=\log[1+e^{-2}+(N-2)e^{-t/2}]$ is convex and monotonically decreasing. $\{\vu^{\star}_1, \cdots, \vu^{\star}_N\}$ denotes a set of vectors which forms a cross-polytope. 

Both inequalities in $(a)$ and $(b)$ are equalities only when the columns of $\vU$ form a cross-polytope. Therefore, the columns of $\vU^{\star}$ form a cross-polytope.  
\end{proof}

\begin{lemma} \label{lemma:elementary_cohn-kumar}
    Given a function $g(t)$ is convex and monotonically decreasing, let
     \begin{align} \label{g-energy}
     \vU^* := \arg\min \limits_{\vU \in \gA} \sum_{i=1}^{N} \sum_{j \neq i} g(\|\vu_i - \vu_j\|^2)\quad\text{s.t.}\quad\|\vu_i\|=1, \|\vv_i\|=1\quad\forall i\in[N],
     \end{align}
     where $\gA:=\{\vU: \vU\text{ is antipodal}\}$. Then, the columns of $\vU^*$ form a simplex cross-polytope for $N=2d$.
\end{lemma}
\begin{proof}
    Suppose $N=2d$ and $\vU\in\gA$. Given a function $g(t)$ is convex and monotonically decreasing. 
    $j(i)$ denotes the corresponding index for $i$ such that $\vu_{j(i)}=-\vu_i$, and $\|\vu_i-\vu_{j(i)}\|^2=4$. Under these conditions, we derive the following: 
    \begin{align*}
        \sum_{i=1}^{N} \sum_{j \neq i} g(\|\vu_i - \vu_j\|^2) &\overset{}{=} N g(4)+ \sum_{i=1}^{N} \sum_{j \neq i, j(i)} g(\|\vu_i - \vu_j\|^2) \\
        &\overset{(a)}{\geq} N g(4)+N(N-2)g\Big(\frac{1}{N(N-2)}\sum_{i=1}^{N} \sum_{j \neq i, j(i)}\|\vu_i - \vu_j\|^2\Big) \\
        &\overset{}{=} N g(4)+N(N-2)g\Big(\frac{1}{N(N-2)}\Big(-4N+\sum_{i=1}^{N} \sum_{j =1}^N\|\vu_i - \vu_j\|^2\Big)\Big) \\
        &\overset{}{=} N g(4)+N(N-2)g\Big(\frac{1}{N(N-2)}\Big(-4N+\sum_{i=1}^{N} \sum_{j =1}^N(2-2\vu_i^{\intercal} \vu_j )\Big)\Big) \\
        &\overset{}{=} N g(4)+N(N-2)g\Big(\frac{1}{N(N-2)}\Big(-4N+2N^2-\Big\|\sum_{i=1}^N\vu_i\Big\|^2\Big)\Big) \\
        &\overset{(b)}{\geq} N g(4)+N(N-2)g\Big(\frac{1}{N(N-2)}\Big(-4N+2N^2\Big)\Big)\\
        &=Ng(4)+N(N-2)g(2), 
    \end{align*}
    where $(a)$ follows by Jensen's inequality; and (b) follows from the fact that $\norm{\sum_{i=1}^N\vu_i}^2 \geq 0$ and the function $g(t)$ is monotonically decreasing. The equality conditions for $(a)$ and $(b)$ only hold when the columns of $\vU$ form a cross-polytope. We can conclude that the columns of $\vU^{\star}$ form a cross polytope.
\end{proof}
\propOne*
\begin{proof}
Consider $\widetilde{\vU},\widetilde{\vV}$ defined such that $\tilde{\vu}_i = \tilde{\vv}_i = \ve_i\; \forall i \in [N],$ where $\ve_i$ is $i$-th unit vector in $\mathbb{R}^N.$ First note that $\tilde{\vu}_i^{\intercal} \tilde{\vv}_i = 1\; \forall i \in [N]$ and $\tilde{\vu}_i^{\intercal} \tilde{\vv}_j = 0 \; \forall{i \neq j}$. Then,
\begin{align}
    &\gL(\widetilde{\vU}, \widetilde{\vV}) 
    = \log(e + N - 1) -1 \label{eq:full-loss-utilde},\\
    &\frac{1}{{N\choose B}}\sum_{i=1}^{N\choose B} \gL(\widetilde{\vU}_{\gB_i}, \widetilde{\vV}_{\gB_i}) 
    = \log(e + B - 1) - 1 \label{eq:ncb-loss-utilde}.
\end{align}

We now consider the second part of the statement. For contradiction, assume that there exists some $c \in \mathbb{R}$ such that $\lclip_{\op{mini}}(\vU,\vV; \gS_B)= c\cdot\lclip(\vU,\vV) \quad \text{for all} \quad \vU, \vV$. Let $\widehat{\vU}, \widehat{\vV}$ be defined such that $\hat{\vu}_i = \hat{\vv}_i = \ve_1\; \forall i \in [N]$, where $\ve_1 = (1, 0, \cdots, 0).$ Note that $\hat{\vu}_i^{\intercal} \hat{\vv}_j = 1\; \forall i, j \in [N]$. Then,
\begin{align}
    &\gL(\widehat{\vU}, \widehat{\vV}) 
    = \log(N) \label{eq:full-loss-uhat},\\
    &\frac{1}{{N\choose B}}\sum_{i=1}^{N\choose B} \gL(\widehat{\vU}_{\gB_i}, \widehat{\vV}_{\gB_i}) 
    = \log(B). \label{eq:ncb-loss-uhat}
\end{align}

From Eq.~(\ref{eq:full-loss-utilde}) and~(\ref{eq:ncb-loss-utilde}), we have that $c = \frac{\log(e+B-1)-1}{\log(e+N-1)-1}$. Whereas from Eq.~(\ref{eq:full-loss-uhat}) and~(\ref{eq:ncb-loss-uhat}), we have that $c = \frac{\log(B)}{\log(N)}$ which is a contradiction. Therefore, there exists no $c \in \mathbb{R}$ satisfying the given condition.
\end{proof}

\thmTwo*
\begin{proof}
\textbf{Case (i)}: Suppose $N \leq d+1.$
\newline \newline
For simplicity, first consider just one of the two terms in the two-sided loss. Therefore, the optimization problem becomes
\begin{align*}
    \min_{\vU, \vV}\quad  &\frac{1}{\binom{N}{B}} \sum_{i=1}^{\binom{N}{B}}  \gL(\vU_{{\gB}_i}, \vV_{{\gB}_i})  \quad
    s.t. \quad  \lVert \vu_i \rVert = 1, \lVert \vv_i \rVert = 1\; \forall i \in [N].
\end{align*}

Similar to the proof of Lem.~\ref{thm:standard-full-batch-case1}, we have that
\begin{align*}
    &\sum_{i=1}^{N\choose B} \gL(\vU_{{\gB}_i}, \vV_{{\gB}_i})
    = {1 \over B} \sum_{i=1}^{N\choose B} \sum_{j \in {\gB}_i} \log \left(1 + \sum_{\substack{k\in \gB_{i}\\ k\neq j}} e^{{\vu_j}^{\intercal} (\vv_k - \vv_j)}\right)\\
    &\overset{(a)}{\geq} {1 \over B} \sum_{i=1}^{N\choose B} \sum_{j \in {\gB}_i} \log \left(1 + (B-1)\exp \left(\frac{\sum_{k\in \gB_{i}, k\neq j} \vu_j^{\intercal} (\vv_k - \vv_j)}{B-1}\right)\right) \\ 
    &= {1 \over B} \sum_{i=1}^{N\choose B} \sum_{j \in {\gB}_i} \log \left(1 + (B-1) \exp \left(\frac{ \sum_{k \in \gB_{i}} \left(\vu_j^{\intercal} \vv_k - B \vu_j^{\intercal} \vv_j\right)}{B-1} \right)\right)\\
    &\overset{(b)}{\geq} {N\choose B}  \log\left(1 + (B-1) \exp \left(\frac{\sum_{i=1}^{N\choose B} \sum_{j \in \gB_i} \sum_{k \in \gB_i} \vu_j^{\intercal} \vv_k - \sum_{i=1}^{N\choose B}\sum_{j \in \gB_i}B \vu_j^{\intercal} \vv_j}{{N\choose B}\cdot B \cdot (B-1)}\right)\right), 
\end{align*}
where $(a)$ and $(b)$ follows by applying Jensen's inequality to $e^t$ and $\log(1+ae^{bt})$ for $a,b>0$, respectively. Note that for equalities to hold in Jensen's inequalities, we need constants $c_j, c$ such that
\begin{align}
    &\vu_j^{\intercal} \vv_k = c_j \quad \forall k \neq j \label{eq:cond3},\\
    &\frac{\vu^{\intercal} \vv_i}{N-1} - \frac{N(\vu_i^{\intercal} \vv_i)}{N-1} = c \quad \forall i \in [N] \label{eq:cond4}.
\end{align}

Now, we carefully consider the two terms in the numerator:
\begin{align*}
 A_1:= \sum_{i=1}^{N\choose B} \sum_{j \in \gB_i} \sum_{k \in \gB_i} \vu_j^{\intercal} \vv_k, \quad
 A_2 := \sum_{i=1}^{N\choose B}\sum_{j \in \gB_i} B \vu_j^{\intercal} \vv_j.
\end{align*}
To simplify $A_1$, first note that for any fixed $l, m \in [N]$ such that $l\neq m$, there are ${{N-2}\choose {B-2}}$ batches that contain $l$ and $m$. And for $l=m$, there are ${{N-1}\choose {B-1}}$ batches that contain that pair. Since these terms all occur in $A_1$, we have that

\begin{align*}
    A_1 &= {{N-2}\choose {B-2}} \sum_{l=1}^N \sum_{m=1}^N \vu_l^{\intercal} \vv_m + \left[{{N-1}\choose {B-1}} - {{N-2}\choose {B-2}}\right]\sum_{l=1}^N \vu_l^{\intercal} \vv_l\\
    &= {{N-2}\choose {B-2}} \sum_{l=1}^N \sum_{m=1}^N \vu_l^{\intercal} \vv_m + {{N-2}\choose {B-2}}\left(\frac{N-B}{B-1}\right) \sum_{l=1}^N \vu_l^{\intercal} \vv_l.
\end{align*}

Similarly, we have that
\begin{align*}
    A_2 = {{N-1}\choose {B-1}} B \sum_{l=1}^N \vu_l^{\intercal} \vv_l.
\end{align*}
Plugging these back into the above inequality, we have that
\begin{align*}
    \sum_{i=1}^{N\choose B} \gL(\vU_{{\gB}_i}, \vV_{{\gB}_i}) &\geq {N\choose B}  \log \left(1 + (B-1)\exp\left(\frac{\sum_{l=1}^N\sum_{m=1}^N \vu_l^{\intercal} \vv_m - N\sum_{l=1}^N \vu_l^{\intercal}\vv_l}{N(N-1)} \right)\right)\\
    &= {N\choose B}  \log \left(1 + (B-1)\exp\left(\frac{\vu^{\intercal} \vv - N\sum_{i=1}^N \vu_i^{\intercal}\vv_i}{N(N-1)} \right)\right).
\end{align*}
Observe that the term inside the exponential is identical to Eq.~(\ref{eq:jensen2}) and therefore, we can reuse the same spectral analysis argument to show that the simplex ETF also minimizes $\sum_{i=1}^{N\choose B} \gL(\vU_{{\gB}_i}, \vV_{{\gB}_i})$. Once again, since the proof is symmetric the simplex ETF also minimizes $\sum_{i=1}^{N\choose B} \gL(\vV_{{\gB}_i}, \vU_{{\gB}_i})$.

\bigskip

\noindent \textbf{Case (ii)}: Suppose $N = 2d,$ and $\vU$, $\vV$ are symmetric and antipodal. Next, we consider the following optimization problem
\begin{align}
\min_{(\vU,\vV)\in\gA} \quad &\frac{1}{\binom{N}{B}} \sum_{i=1}^{\binom{N}{B}}  \lclip(\vU_{{\gB}_i}, \vV_{{\gB}_i}) \quad
    s.t. \quad  \lVert \vu_i \rVert = 1, \lVert \vv_i \rVert = 1\; \forall i \in [N], \label{mini-batch opt}    
\end{align}
where $\gA:=\{(\vU,\vV): \vU, \vV\text{ are symmetric and antipodal%
}\}$. 
Since $\vU= \vV$ (\emph{symmetric} property) the contrastive loss satisfies 
    \begin{align}
         \lclip(\vU_{{\gB}_i}, \vV_{{\gB}_i})
         &= 2\gL(\vU_{{\gB}_i}, \vU_{{\gB}_i}) \nonumber \\
         &=-{2 \over B} \sum_{j\in\gB_i} \left[ \vu_j^{\intercal}\vu_j - \log \Big( \sum_{k\in\gB_i} e^{\vu_j^{\intercal} \vu_k}\Big)\right] \nonumber \\
         &= -2+ {2 \over B} \sum_{j\in\gB_i} \log\big(\sum_{k\in\gB_i}
         e^{\vu_j^{\intercal}\vu_k}\big) \label{lclip_summary}.
    \end{align}
Therefore, the solution of the optimization problem in Eq.~(\ref{mini-batch opt}) is identical to the minimizer of the following optimization problem:
\begin{equation*}
\vU^{\star}:=\arg\min_{\vU} \quad \sum_{i=1}^{N \choose B} \sum_{j \in \gB_i} \log \Big(\sum_{k \in \gB_i} e^{\vu_j^{\intercal} \vu_k}\Big).  
\end{equation*}
The objective of the optimization problem can be rewritten by reorganizing summations as
\begin{equation}
\label{eq:summation_term}
\sum_{j=1}^{N} \sum_{i\in \gI_j} \log \Big(\sum_{k \in \gB_i} e^{\vu_j^{\intercal} \vu_k}\Big),
\end{equation}
where $\gI_j:=\{i: j \in \gB_i \}$ represents the set of batch indices containing $j$. We then divide the summation term in Eq.~(\ref{eq:summation_term}) into two terms:
\begin{equation}
\label{eq:two_terms}
\sum_{j=1}^{N} \sum_{i\in \gI_j} \log \Big(\sum_{k \in \gB_i} e^{\vu_j^{\intercal} \vu_k}\Big)
= \sum_{j=1}^{N} \sum_{i\in \gA_j} \log \Big(\sum_{k \in \gB_i} e^{\vu_j^{\intercal} \vu_k}\Big)
+\sum_{j=1}^{N} \sum_{i\in \gA_j^c} \log \Big(\sum_{k \in \gB_i} e^{\vu_j^{\intercal} \vu_k}\Big),
\end{equation}
by partitioning the set $\gI_j$ for each $j \in [N]$ into as the following with $k(j)$ being the index for which $u_{k(j)}=-u_j$: 
\begin{align*}
    \gA_j := \{i:  j \in \gB_i,
    \text{ and }k(j) \in \gB_i \};\quad \gA_j^c := \{i: j\in \gB_i, \text{ and }k(j) \notin \gB_i\}.
\end{align*}
We will prove that the columns of $\vU^*$ form a cross-polytope by showing that the minimizer of each term of the RHS in Eq.~(\ref{eq:two_terms}) also forms a cross-polytope. Let us start with the first term of the RHS in Eq.~(\ref{eq:two_terms}). Starting with applying Jensen's inequality to the concave function $f(x) :=\log(e+e^{-1}+ x)$, we get:
\begin{align*}
        &\sum_{j=1}^{N}  \sum_{i \in \gA_j} \log \Big( \sum_{k \in \gB_i} e^{\vu_j^{\intercal} \vu_k}\Big) 
        =\sum_{j=1}^{N}\sum_{i \in \gA_j} \log \Big( e+e^{-1}+\sum_{k \in \gB_i \setminus \{j, k(j)\}} e^{\vu_j^{\intercal} \vu_k} \Big)\\
        &\overset{}{\geq} {1 \over B-2} \sum_{j=1}^{N} \sum_{i \in \gA_j} \sum_{k \in \gB_i \setminus \{j, k(j)\}} \log \big(e+e^{-1}+(B-2)e^{\vu_j^{\intercal} \vu_k} \big)\\
        &={1 \over B-2} \sum_{j=1}^{N} \sum_{k \notin \{j, k(j)\}} {N-3 \choose B-3} \log \big(e+e^{-1}+(B-2)e^{\vu_j^{\intercal} \vu_k} \big)\\
        &={{N-3 \choose B-3} \over B-2} \Big[\sum_{j=1}^{N} \sum_{k \neq j}  \log \big(e+e^{-1}+(B-2)e^{\vu_j^{\intercal} \vu_k} \big) - N \log\big(e+(B-1)e^{-1}\big) \Big]\\
        &={{N-3 \choose B-3} \over B-2} \Big[\sum_{j=1}^{N} \sum_{k \neq j}  \log \big(e+e^{-1}+(B-2)e\cdot e^{-\frac{\|\vu_j- \vu_k\|^2}{2}} \big) - N \log\big(e+(B-1)e^{-1}\big) \Big]\\
        &\overset{(a)}{\geq}{{N-3 \choose B-3} \over B-2} \Big[\sum_{j=1}^{N} \sum_{k \neq j}  \log \big(e+e^{-1}+(B-2)e\cdot e^{-\frac{\|\vu_j^{\star}- \vu_k^{\star}\|^2}{2}} \big) - N \log\big(e+(B-1)e^{-1}\big) \Big],
\end{align*} 
where $(a)$ follows by Lem.~\ref{lemma:elementary_cohn-kumar} and the fact that $g(t)=\log(a+be^{-\frac{t}{2}})$ for $a,b>0$ is convex and monotonically decreasing. $\{\vu^{\star}_1, \cdots, \vu^{\star}_N\}$ denotes a set of vectors which forms a cross-polytope. All equalities hold only when the columns of $\vU$ form a cross-polytope.

Next consider the second term of the RHS in Eq.~(\ref{eq:two_terms}). By following a similar procedure above, we get: 
\begin{align*}
   &\sum_{j=1}^{N} \sum_{i \in \gA^c_j} \log \Big( \sum_{k \in \gB_i} e^{\vu_j^{\intercal} \vu_k}\Big) 
   \geq {1 \over B-1} \sum_{j=1}^{N} \sum_{i \in \gA_j} \sum_{k \in \gB_i \setminus \{j\}} \log \Big(e+ (B-1)e^{\vu_j^{\intercal} \vu_k}\Big)\\
   &= {1 \over B-1} \sum_{j=1}^{N}\sum_{k \notin \{j, k(j)\}} \binom{N-3}{B-2}\log \Big(e+ (B-1)e^{\vu_j^{\intercal} \vu_k}\Big)\\
   &={{N-3 \choose B-2} \over B-1} \Big[ \sum_{j=1}^{N} \sum_{k \neq j} \log \big(e+(B-1)e^{\vu_j^{\intercal} \vu_k}\big) - N\log\big(e+(B-1)e^{-1}\big)\Big]\\
   &\geq{{N-3 \choose B-2} \over B-1} \Big[\sum_{j=1}^{N}\sum_{k \neq j} \log \big(e+(B-1)e\cdot e^{-\frac{\|\vu_j^{\star}- \vu_k^{\star}\|^2}{2}}\big) - N\log\big(e+(B-1)e^{-1}\big)\Big],
\end{align*} 
where $\{\vu^{\star}_1, \cdots, \vu^{\star}_N\}$ denotes a set of vectors which forms a cross-polytope.

Both terms of RHS in Eq.~(\ref{eq:two_terms}) have the minimum value when $\vU$ forms a cross-polytope. Therefore, we can conclude that the columns of $\vU^{\star}$ form a cross-polytope.
\end{proof}
\thmThree*
\begin{proof}
    Consider a set of batches $\gS_{B} \subset\left[{N\choose 2}\right]$ with the batch size $B=2$. Without loss of generality, assume that $(1, 2) \notin \bigcup_{i \in \gS_B} \{\gB_i\}$. For contradiction, assume that the simplex ETF - $(\vU^\star, \vV^\star)$ is indeed the optimal solution of the loss over these $\gS_B$ batches. Then, by definition, we have that for any $(\vU, \vV) \neq (\vU^\star, \vV^\star),$
    \begin{align}
        \frac{1}{|\gS_B |}\sum_{i \in \gS_B} \gL(\vU^\star_{\gB_i}, \vV^\star_{\gB_i}) &\leq \frac{1}{|\gS_B |}\sum_{i \in \gS_B} \gL(\vU_{\gB_i}, \vV_{\gB_i}) \nonumber\\
        \Rightarrow \sum_{i \in \gS_B} \gL(\vU^\star_{\gB_i}, \vV^\star_{\gB_i}) &\leq \sum_{i \in \gS_B} \gL(\vU_{\gB_i}, \vV_{\gB_i}), \label{eq:optimality-defn-thm3}
    \end{align}
where $(\vU^\star, \vV^\star)$ is defined such that $\vu_i^\star = \vv_i^\star$ for all $i \in [N]$ and ${\vu_i^\star}^{\intercal} \vv_j^\star = -1/(N-1)$ for all $i \neq j$. Also recall that $\lVert \vu_i \rVert = \lVert \vv_i \rVert = 1$ for all $i \in [N]$.
Therefore, we also have
\begin{align}
    \sum_{i \in \gS_B} \gL (\vU^\star_{\gB_i}, \vV^\star_{\gB_i}) &= \sum_{i \in \gS_B}\sum_{j \in \gB_i} \log\left(1 + \sum_{k \in \gB_i, k\neq j} \exp \left({\vu_j^\star}^{\intercal} (\vv_k^\star - \vv_j^\star)\right)\right) \nonumber\\
    &= \sum_{i \in \gS_B}\sum_{j \in \gB_i} \log\left(1 + \sum_{k \in \gB_i, k\neq j} \exp \left(-\frac{1}{N-1} - 1\right)\right) \nonumber\\
    &= \sum_{i \in \gS_B}\sum_{j \in \gB_i} \log\left(1 + \exp \left(-\frac{1}{N-1} - 1\right)\right), \label{eq:etf-loss-val}
\end{align}
where the last equality is due to the fact that $|\gB_i| = 2$.

Now, let us consider $(\widetilde{\vU}, \widetilde{\vV})$ defined such that $\tilde{\vu}_i = \tilde{\vv}_i$ for all $i \in [N]$, and $\tilde{\vu}_i^{\intercal} \tilde{\vv}_j = -1/(N-2)$ for all $i\neq j, (i, j) \notin \{(1,2), (2,1)\}$. Intuitively, this is equivalent to placing $\tilde{\vu}_2, \dots, \tilde{\vu}_N$ on a simplex ETF of $N-1$ points and setting $\tilde{\vu}_1 = \tilde{\vu}_2$. This is clearly possible because $d > N-1 \Rightarrow d > N-2,$ which is the condition required to place $N-1$ points on a simplex ETF in $\mathbb{R}^d$.
Therefore,
\begin{align}
    \sum_{i \in \gS_B} \gL (\widetilde{\vU}_{\gB_i}, \widetilde{\vV}_{\gB_i}) &= \sum_{i \in \gS_B}\sum_{j \in \gB_i} \log\left(1 + \sum_{k \in \gB_i, k\neq j} \exp \left(\tilde{\vu}_j^{\intercal} (\tilde{\vv}_k - \tilde{\vv}_j)\right)\right) \nonumber\\
    &= \sum_{i \in \gS_B}\sum_{j \in \gB_i} \log\left(1 + \sum_{k \in \gB_i, k\neq j} \exp \left(-\frac{1}{N-2} - 1\right)\right) \nonumber\\
    &= \sum_{i \in \gS_B}\sum_{j \in \gB_i} \log\left(1 + \exp \left(-\frac{1}{N-2} - 1\right)\right), \label{eq:utilde-loss-val}
\end{align}
where the last equality follows since $(1, 2) \notin \bigcup_{i \in \gS_B} \{\gB_i\}$.
It is easy to see from Eq.~(\ref{eq:etf-loss-val}) and~(\ref{eq:utilde-loss-val}) that $\sum_{i \in \gS_B} \gL (\widetilde{\vU}_{\gB_i}, \widetilde{\vV}_{\gB_i}) < \sum_{i \in \gS_B} \gL (\vU^\star_{\gB_i}, \vV^\star_{\gB_i})$ which contradicts Eq.~(\ref{eq:optimality-defn-thm3}). Therefore, the optimal solution of minimizing the contrastive loss over any $\gS_B \subset \left[{N\choose 2}\right]$ batches is not the simplex ETF completing the proof.
\end{proof}

\begin{restatable}{proposition}{propTwo}
\label{thm:ncb-is-necessary1}
Suppose $B\ge 2$, and let $\gS_B\subseteq \left[{\binom{N}{B}}\right]$ be a set of mini-batch indices. If there exist two data points that never belong together in any mini-batch, \ie $\exists i,j\in[N]$ \text{s.t.} $\{i,j\}\not\subset\gB_k$ for all $k\in\gS_B$, then the optimal solution of Eq.~(\ref{prob:mini-batch}) is not the minimizer of the full-batch problem in Eq.~(\ref{prob:standard-full-batch}).
\end{restatable}

\begin{proof}
The proof follows in a fairly similar manner to that of Thm.~\ref{thm:sub-batch}. Consider a set of batches of size $B \geq 2$, $\gS_B \subset [{N\choose B}]$. Without loss of generality, assume that $\{1, 2\} \not\subset \gB_k$ for any $k \in \gS_B$. For contradiction, assume that the simplex ETF - $(\vU^\star, \vV^\star)$ is the optimal solution of the loss over these $\gS_B$ batches. Then, by definition, we have that for any $(\vU, \vV) \neq (\vU^\star, \vV^\star)$

Once again, for contradiction assume that the simplex ETF - $(\vU^\star, \vV^\star)$ is indeed the optimal solution of the loss over these $\gS_B$ batches. Then, by definition for any $(\vU, \vV) \neq (\vU^\star, \vV^\star),$

\begin{align}
        \frac{1}{|\gS_B |}\sum_{i \in \gS_B} \gL(\vU^\star_{\gB_i}, \vV^\star_{\gB_i}) &\leq \frac{1}{|\gS_B |}\sum_{i \in \gS_B} \gL(\vU_{\gB_i}, \vV_{\gB_i}) \nonumber\\
        \Rightarrow \sum_{i \in \gS_B} \gL(\vU^\star_{\gB_i}, \vV^\star_{\gB_i}) &\leq \sum_{i \in \gS_B} \gL(\vU_{\gB_i}, \vV_{\gB_i}), \label{eq:optimality-defn-corr1}
    \end{align}
where $(\vU^\star, \vV^\star)$ is defined such that $\vu_i^\star = \vv_i^\star$ for all $i \in [N]$ and ${\vu_i^\star}^{\intercal} \vv_j^\star = -1/(N-1)$ for all $i \neq j$. Also recall that $\lVert \vu_i \rVert = \lVert \vv_i \rVert = 1$ for all $i \in [N]$. Therefore, we also have

\begin{align}
    \sum_{i \in \gS_B} \gL (\vU^\star_{\gB_i}, \vV^\star_{\gB_i}) &= {1 \over B} \sum_{i \in \gS_B}\sum_{j \in \gB_i} \log\left(1 + \sum_{k \in \gB_i, k\neq j} \exp \left({\vu_j^\star}^{\intercal} (\vv_k^\star - \vv_j^\star)\right)\right) \nonumber\\
    &= {1 \over B} \sum_{i \in \gS_B}\sum_{j \in \gB_i} \log\left(1 + \sum_{k \in \gB_i, k\neq j} \exp \left(-\frac{1}{N-1} - 1\right)\right) \nonumber\\
    &= {1 \over B} \sum_{i \in \gS_B}\sum_{j \in \gB_i} \log\left(1 + (B-1)\exp \left(-\frac{1}{N-1} - 1\right)\right). \label{eq:etf-loss-val-corr1}
\end{align}

Now, let us consider $(\widetilde{\vU}, \widetilde{\vV})$ defined such that $\tilde{\vu}_i = \tilde{\vv}_i$ for all $i \in [N]$, $\tilde{\vu}_2 = \tilde{\vv}_2$ and $\tilde{\vu}_i^{\intercal} \tilde{\vv}_j = -1/(N-2)$ for all $i\neq j, (i, j) \notin \{(1,2), (2,1)\}$. Once again, note that this is equivalent to placing $\tilde{\vu}_2, \dots, \tilde{\vu}_N$ on a simplex ETF of $N-1$ points and setting $\tilde{\vu}_1 = \tilde{\vu}_2$. Hence,
\begin{align}
    \sum_{i \in \gS_B} \gL (\widetilde{\vU}_{\gB_i}, \widetilde{\vV}_{\gB_i}) &= {1 \over B} \sum_{i \in \gS_B}\sum_{j \in \gB_i} \log\left(1 + \sum_{k \in \gB_i, k\neq j} \exp \left(\tilde{\vu}_j^{\intercal} (\tilde{\vv}_k - \tilde{\vv}_j)\right)\right) \nonumber\\
    &= {1 \over B} \sum_{i \in \gS_B}\sum_{j \in \gB_i} \log\left(1 + \sum_{k \in \gB_i, k\neq j} \exp \left(-\frac{1}{N-2} - 1\right)\right) \nonumber\\
    &= {1 \over B} \sum_{i \in \gS_B}\sum_{j \in \gB_i} \log\left(1 + (B-1)\exp \left(-\frac{1}{N-2} - 1\right)\right), \label{eq:utilde-loss-val-corr1}
\end{align}
where for the final equality note that following. The only pair for which $\tilde{\vu}_j^{\intercal} \tilde{\vv}_k \neq -1/(N-2)$ is $(j, k) = (1,2)$. Since there is no $i \in \gS_B$ such that $\{1, 2\} \in \gB_i$, this term never appears in our loss. From Eq.~(\ref{eq:etf-loss-val-corr1}) and Eq.~(\ref{eq:utilde-loss-val-corr1}), we have that $\sum_{i \in \gS_B} \gL (\widetilde{\vU}_{\gB_i}, \widetilde{\vV}_{\gB_i}) < \sum_{i \in \gS_B} \gL (\vU^\star_{\gB_i}, \vV^\star_{\gB_i})$ which contradicts Eq.~(\ref{eq:optimality-defn-corr1}). Therefore, we conclude that the optimal solution of the contrastive loss over any $\gS_B \subset \left[{N\choose 2}\right]$ batches is not the simplex ETF.
\end{proof}

\begin{restatable}{proposition}{propThree}
\label{thm:ncb-is-necessary2}
Suppose $B \ge 2$, and let $\gS_B \subseteq \left[{\binom{N}{B}}\right]$ be a set of mini-batch inidices satisfying $\gB_i\bigcap \gB_j = \varnothing, \forall i,j\in\gS_B $ and $\bigcup_{i\in\gS_B}\gB_i = [N]$, \ie $\{\gB_i\}_{i \in \gS_B}$ forms non-overlapping mini-batches that cover all data samples.
Then, the minimizer of the mini-batch loss optimization problem in Eq.~(\ref{prob:mini-batch}) 
is different from the minimizer of the full-batch loss optimization problem in Eq.~(\ref{prob:standard-full-batch}). 
\end{restatable}

\begin{proof}
    Applying Lem.~\ref{thm:standard-full-batch-case1} specifically to a single batch $\gB_i$ gives us that the optimal solution for just the loss over this batch is the simplex ETF over $B$ points. In the case of non-overlapping batches, the objective function can be separated across batches and therefore the optimal solution for the sum of the losses is equal to the solution of minimizing each term independently. More precisely, we have
\begin{align}
    \min_{\vU, \vV} \sum_{i=1}^{N/B}\gL^{\op{con}}(\vU_{{\gB}_i}, \vV_{{\gB}_i}) 
    =\sum_{i=1}^{N/B} \min_{\vU_{\gB_i}, \vV_{\gB_i}} \gL^{\op{con}}(\vU_{{\gB}_i}, \vV_{{\gB}_i}), \nonumber
\end{align}
where ${\vU}_{\gB_i} = \{\vu_j: j\in\gB_i\}$ and ${\vV}_{\gB_i} = \{\vv_j: j\in\gB_i\}$, respectively, and the equality follows from the fact that $\gB_i$'s are disjoint. 
\end{proof}

\subsection{Proofs of Results From Section~\ref{sec:osgd}}\label{sec5:theory}
\lclipnonquasiconvex*
\begin{proof}
The contrastive loss function $\lclip$ is geodesic quasi-convex if for any two points $(\vU, \vV)$ and $(\vU', \vV')$ in the domain and for all $t$ in $[0,1]$: 
$$\lclip(t(\vU, \vV)+(1-t)(\vU', \vV'))\leq\max\{\lclip(\vU, \vV), \lclip(\vU', \vV')\}.$$

We provide a counter-example for geodesic quasi-convexity, which is a triplet of points $(\vU^1, \vV^1)$, $(\vU^2, \vV^2)$, $(\vU^3, \vV^3)$ where $(\vU^3, \vV^3)$ is on the geodesic between other two points and satisfies $\lclip(\vU^3, \vV^3) > \max\{\lclip(\vU^1, \vV^1), \lclip(\vU^2, \vV^2)\}$.
Let $N = 2$ and 
\[  
    \vU^1 = \begin{bmatrix}
        \sqrt{\frac{1}{2}} & \sqrt{\frac{2}{5}} \\
        \sqrt{\frac{1}{2}} & \sqrt{\frac{1}{5}}
    \end{bmatrix}, 
    \vU^2 = \begin{bmatrix}
        \sqrt{\frac{1}{2}} & \sqrt{\frac{1}{2}} \\
        \sqrt{\frac{1}{2}} & \sqrt{\frac{1}{2}}
    \end{bmatrix}, 
    \vV^1 = \begin{bmatrix}
        \sqrt{\frac{1}{2}} & \sqrt{\frac{1}{2}} \\
        \sqrt{\frac{1}{2}} & \sqrt{\frac{1}{2}}
    \end{bmatrix}, 
    \vV^2 = \begin{bmatrix}
        \sqrt{\frac{2}{5}} & \sqrt{\frac{1}{2}} \\
        \sqrt{\frac{1}{5}} & \sqrt{\frac{1}{2}}
    \end{bmatrix}.
\]
Now, define $\vU^3 = \mathrm{normalize}((\vU^1 + \vU^2)/2)$ and $\vV^3 = \mathrm{normalize}((\vV^1 + \vV^2)/2)$, which is the ``midpoint'' of the geodesic between $(\vU^1, \vV^1)$ and $(\vU^2, \vV^2)$.
By direct calculation, we obtain $\lclip(\vU^3, \vV^3) \approx 2.798 > 2.773 \approx \max(\lclip(\vU^1, \vV^1), \lclip(\vU^2, \vV^2))$, which indicates $\lclip$ is geodesic non-quasi-convex.
\end{proof} 

\begin{theorem}[Theorem~\ref{thm:Toy-example} restated]\label{theorem 9} 
Consider $N=4$ samples and their embedding vectors $\{\vu_i\}_{i=1}^N$, $\{\vv_i\}_{i=1}^N$ with dimension $d=2$. 
Suppose $\vu_i$'s are parametrized by $\vtheta^{(t)} = [\theta_1^{(t)}, \theta_2^{(t)}, \theta_3^{(t)}, \theta_4^{(t)}]$ as in the setting described in Sec.~\ref{subsec:toy_example} (see Fig.~\ref{fig:toy_example_fig}).
Consider initializing $\vu_i^{(0)} = \vv_i^{(0)}$ and $\theta_i^{(0)} = \epsilon > 0$ for all $i$, then updating $\vtheta^{(t)}$ via OSGD and SGD with the batch size $B=2$ as described in Sec.~\ref{subsec:toy_example}.
Let $T_{\textnormal{OSGD}}$, $T_{\textnormal{SGD}}$ be the minimal time required for OSGD, SGD algorithm to have $\mathbb{E}[\vtheta^{(T)}] \in (\pi/4 - \rho, \pi/4)^N$. 
Suppose there exist $ \tilde{\epsilon}$, $\overline{T}$ such that for all $t$ satisfying  $\gB^{(t)}=\left\{1,3\right\}$ or $\left\{2,4\right\}$, $\|\nabla_{\vtheta^{(t)}} \lclip (\vU_{\gB^{(t)}}, \vV_{\gB^{(t)}})\| \leq \tilde{\epsilon}$, 
and $T_{\textnormal{OSGD}}, \ T_{\textnormal{SGD}}< \overline{T}.$ 
    Then,
    \[
    T_{\textnormal{OSGD}} \geq {\pi/4 - \rho - \epsilon +O(\eta^2 \epsilon + \eta \epsilon^3) \over \eta \epsilon}, \quad
    T_{\textnormal{SGD}} \geq {3(e^2+1) \over e^2-1} {\pi/4-\rho-\epsilon+O(\eta^2 \epsilon+\eta^2 \tilde{\epsilon}) \over \eta \epsilon+O(\eta \epsilon^3+\eta \tilde{\epsilon})}.
    \]
\end{theorem}

\begin{proof}
We begin with the proof of 
\[ T_{\textnormal{OSGD}} \geq {\pi/4 - \rho - \epsilon +O(\eta^2 \epsilon + \eta \epsilon^3) \over \eta \epsilon}.\]
 Assume that the parameters are initialized at $\big(\theta^{(0)}_1, \theta^{(0)}_2, \theta^{(0)}_3, \theta^{(0)}_4\big) = (\epsilon, \epsilon, \epsilon, \epsilon)$. 
 Then, there are six batches with the batch size $B=2$, and we can categorize the batches according to the mini-batch contrastive loss: 
 \begin{enumerate}[leftmargin=0.7cm]
     \item $\gB=\{1, 2\} \ \textnormal{or} \ \{3, 4\}$: $\lclip(\vU_\gB, \vV_\gB) = -2+2 \log(e+e^{ \cos 2\epsilon});$\\
     \item  $\gB=\{1, 3\} \ \textnormal{or} \ \{2, 4\}$: $\lclip(\vU_\gB, \vV_\gB) = -2+2 \log(e+e^{ -1});$\\
     \item $\gB=\{1, 4\} \ \textnormal{or} \ \{2, 3\}$: $\lclip(\vU_\gB, \vV_\gB) = -2+2 \log(e+e^{- \cos 2\epsilon}).$
 \end{enumerate}
    \noindent Without loss of generality, we assume that OSGD algorithm described in Algo.~\ref{alg:osgd} chooses the mini-batch $\gB = \{1, 2\}$ corresponding to the highest loss at time $t=0,$ and updates the parameter as  \[
        \theta_1^{(1)} = \epsilon - \eta \nabla_{\theta_1} \lclip(\vU_\gB, \vV_\gB), \ \theta_2^{(1)} =  \epsilon - \eta \nabla_{\theta_2} \lclip(\vU_\gB, \vV_\gB) .
    \] Then, for the next update, OSGD choose $\vu_3, \vu_4$ which is now closer than updated $\vu_1, \vu_2$. And $\vu_3, \vu_4$ would be updated as same as what previously $\vu_1, \vu_2$ have changed. Thus, $\theta_{1}$ updates only at the even time, and stays at the odd time, i.e. $$\theta_{1}^{(t+1)}=\begin{cases}
        \theta_{1}^{(t)}- \eta \nabla_{\theta_1} \lclip(\vU_\gB, \vV_\gB) & \text{if }t \ \text{is even,} \\
        \theta_{1}^{(t)} & \text{if }t \ \text{is odd.}
    \end{cases}$$  Iterating this procedure, we can view OSGD algorithm as one-parameterized algorithm of parameter $\phi^{(t)}=\theta_{1}^{(2t)}$ as:
    \begin{align*}
        \phi^{(0)} = \epsilon, \quad
        \phi^{(t)} = \phi^{(t-1)} + \eta \ g\big(\phi^{(t-1)}\big), \quad
        \phi^{(T_{\textnormal{half}})} \in \big({\pi \over 4}-\rho, \ {\pi \over 4}\big),
    \end{align*}
    where $g(\phi) = {2 \sin (2\phi) / (1+ e^{1- \cos(2\phi)})}$, and $T_{\textnormal{half}}:=T_{\textnormal{OSGD}}/2.$ In the procedure of updates, we may assume that $\phi^{(t)} \in (0, {\pi \over 4})$ for all $t$. To analyze the drift of $\phi^{(t)}$, we firstly study smoothness of $g$;
    \begin{align*}
        g' (\phi)&={4e^{\cos 2\phi}(\cos 2\phi (e+e^{\cos 2\phi})-e \sin^2 2\phi) \over (e+e^{\cos 2\phi})^2}.
    \end{align*}
    We can observe that $\max \limits_{\phi \in [0, {\pi \over 4}]} | 
    g' (\phi)| = 2$, hence ${g(\phi)}$ has Lipschitz constant $2,$ i.e. \[\Big|{g}\big(\phi^{(t-1)}\big) - {g} \big(\phi^{(0)}\big)\Big| \leq 2 \big|\phi^{(t-1)}-\phi^{(0)}\big|.\]
    Therefore, 
    \begin{align*}
    \phi^{(t)} - \phi^{(t-1)} &= \eta \big| g (\phi^{(t-1)}) \big| \\
    &\leq \eta | g (\epsilon) | + 2\eta (\phi^{(t-1)}-\epsilon)\\
    &=2 \eta \phi^{(t-1)}+O(\eta \epsilon^3),
    \end{align*}
    where the first inequality is from Lipschitz-continuity of $g(\phi)$, and the second equality is from Taylor expansion of $g$ at $\epsilon =0$ as;\[
    g(\epsilon)=2\epsilon - \frac{10}{3}\epsilon^3+\frac{34}{15}\epsilon^5 + \cdots.
    \]  Hence, $\phi^{(t)} \leq (1+2\eta) \phi^{(t-1)} + O(\eta \epsilon^3)$ indicates that \begin{align*}
        \phi^{(T_{\textnormal{half}})} &\leq (1+2\eta)^{T_{\textnormal{half}}} \epsilon + \overline{T} \ O(\eta \epsilon^3)\\
        &\leq (1+2\eta T_{\textnormal{half}})\epsilon + O(\eta^2 \epsilon+\eta \epsilon^3),
    \end{align*} for some constant $\overline{T}>T_{\textnormal{OSGD}}.$ Moreover ${\pi \over 4}-\rho < \phi^{(T_{\textnormal{half}})}$ implies that \begin{align*}
        T_{\textnormal{half}} \geq {1 \over 2} {\pi/4 - \rho - \epsilon +O(\eta\epsilon^3 + \eta^2 \epsilon) \over \eta \epsilon}.
    \end{align*}
    So, we obtain the lower bound of ${T}_{\textnormal{OSGD}}$ by doubling $T_{\textnormal{half}}.$
\bigskip
We estimate of $T_{\textnormal{OSGD}}.$\\
Now, we study convergence rate of SGD algorithm. We claim that 
\[
T_{\textnormal{SGD}} \geq {3(e^2+1) \over e^2-1} {\pi/4-\rho-\epsilon+O(\eta^2 (\epsilon+ \tilde{\epsilon})) \over \eta \epsilon+O(\eta (\epsilon^3+ \tilde{\epsilon}))}.
\]
Without loss of generality, we firstly focus on the drift of $\theta_1$. Since batch selection is random, given $\vtheta^{(t)} = (\theta_1^{(t)}, \theta_2^{(t)}, \theta_3^{(t)}, \theta_4^{(t)})$:
\begin{enumerate}[leftmargin=0.7cm]
    \item $\gB=\{1, 2\}$ with probability ${1 / 6}$. Then, $\lclip(\vU_\gB, \vV_\gB) = -2+2 \log(e+e^{ \cos (\theta_1^{(t)} + \theta_2^{(t)})})$ implies \[
\theta_1^{(t+1)} = \theta_1^{(t)} + \eta {2 \sin (\theta_1^{(t)}+\theta_2^{(t)}) \over 1+ e^{1- \cos(\theta_1^{(t)} + \theta_2^{(t)})}}.\]
    \item $\gB=\{1, 3\}$  with probability ${1 / 6}$. 
    At $t=0$, the initial batch selection can be primarily categorized into three distinct sets; closely positioned vectors $\{\vu_1, \vu_2\}$ or $\{\vu_3, \vu_4\}$, vectors that form obtuse angles $\{\vu_1, \vu_4\}$ or $\{\vu_2, \vu_3\}$, and vectors diametrically opposed at $180^\circ,$ $\{\vu_1, \vu_3\}$ or $\{\vu_2, \vu_4\}$. Given that $\epsilon$ is substantially small, the possibility of consistently selecting batches from the same category for subsequent updates is relatively low. As such, it is reasonable to infer that each batch is likely to maintain its position within the initially assigned categories. From this, one can deduce that vector sets such as $\{\vu_1, \vu_3\}$ or $\{\vu_2, \vu_4\}$ continue to sustain an angle close to $180^\circ$. Given these conditions, it is feasible to postulate that if the selected batch $\gB$ encompasses either $\{1, 3\}$ or $\{2, 4\}$, the magnitude of the gradient of the loss function $\lclip(U_\gB, V_\gB)$, denoted by  $\|\nabla \lclip(U_\gB, V_\gB)\|$, would be less than a particular threshold $\tilde{\epsilon},$ i.e.
    \[
      \|\nabla \lclip(U_\gB, V_\gB)\|< \tilde{\epsilon}.
     \]  
    Then, 
    \[\theta_1^{(t+1)} = \theta_1^{(t)}+ \eta O(\tilde{\epsilon}).\]
 \item $\gB=\{1, 4\}$ with probability ${1 / 6}$. Then, $\lclip(\vU_\gB, \vV_\gB) = -2+2 \log(e+e^{- \cos (\theta_1+\theta_4)})$ implies \[
\theta_1^{(t+1)} = \theta_1^{(t)} - \eta {2 \sin (\theta_1^{(t)}+\theta_4^{(t)}) \over 1+ e^{1+ \cos(\theta_1^{(t)} + \theta_4^{(t)})}}. \]
\end{enumerate}
Since there is no update on $\theta_1$ for the other cases, taking expectation yields \begin{align*}
\mathbb{E}[\theta_1^{(t+1)} - \theta_1^{(t)}|\vtheta^{(t)}]={\eta \over 6} F_{1}(\vtheta^{(t)}) + O(\eta \tilde{\epsilon}),
\end{align*}
where $F_1(\vtheta)$ is defined as:
$$F_1(\vtheta)={2 \sin (\theta_1+\theta_2) \over 1+ e^{1- \cos(\theta_1 + \theta_2)}} - {2 \sin (\theta_1+\theta_4) \over 1+ e^{1+ \cos(\theta_1 + \theta_4)}}.$$ 
We study smoothness of $F_1$ by setting $F_{1}(\vtheta) = f_{-}(\theta_1+\theta_2)-{f}_{+}(\theta_1+\theta_4)$, where $$f_{-}(t):={2 \sin t \over 1+e^{1-cost}}, \quad {f}_{+}(t):={2\sin t \over 1+e^{1+cost}}.$$ 
Note that
\begin{align*}
\max_{t \in [0, {\pi / 2}]}  |f_{-}'(t)| = 1, \quad
\max_{t \in [0, {\pi / 2}]} |{f}_{+}'(t)|= C,
\end{align*}
for some constant $C\in (0, 1).$ Then for $\vtheta = (\theta_1, \theta_2, \theta_3, \theta_4), \vtheta' = (\theta'_1, \theta'_2, \theta'_3, \theta'_4)$, \begin{align*}
|F_{1}(\vtheta')-F_{1}(\vtheta)| &\leq |f_{-}(\theta'_1+\theta'_2)-f_{-}(\theta_1+\theta_2)|+|f_{+}(\theta'_1+\theta'_4)-f_{+}(\theta_1+\theta_4)|\\
&\leq 1 \cdot |\theta'_1+\theta'_2-\theta_1-\theta_2|+C \cdot |\theta'_1+\theta'_4-\theta_1-\theta_4|\\
&\leq 2(1+C) \|\vtheta' - \vtheta\|.
\end{align*}

        \noindent In the same way, we can define the functions $F_2, F_3, F_4$ all having Lipschitz constant $2(1+C)$. As we define $F(\vtheta)=(F_1(\vtheta), F_2(\vtheta), F_3(\vtheta), F_4(\vtheta))$, it has Lipschitz constant $4(1+C)$ satisfying that \[
        \mathbb{E}[\vtheta'-\vtheta|\vtheta]={\eta \over 6} F(\vtheta)+O(\eta \tilde{\epsilon}),
        \]
        where Big $O(\cdot)$ is applied elementwise to the vector, denoting that each element follows $O(\cdot)$ independently.
        From Lipschitzness of $F$, for any $t \geq 1,$ \begin{align*}
            \mathbb{E} [\|\vtheta^{(t)} - \vtheta^{(t-1)}\||\vtheta^{(t-1)}] &\leq {\eta \over 6} \|F(\vtheta^{(t-1)})\| + O(\eta \tilde{\epsilon} ) \\
            &\leq {\eta \over 6} \|F(\vtheta^{(0)})\|+ {\eta \over 6} \|F(\vtheta^{(t-1)})-F(\vtheta^{(0)})\|+O(\eta\tilde{\epsilon}) \\
            &\leq {\eta \over 6} \|F(\vtheta^{(0)})\|+ {2\eta(1+C) \over 3} \|\vtheta^{(t-1)}-\vtheta^{(0)}\|+O(\eta\tilde{\epsilon}).
        \end{align*}
        By taking expecations for both sides,
        \begin{align*}
        \mathbb{E}[ \|\vtheta^{(t)} - \vtheta^{(t-1)}\|]
        \leq {\eta \over 6} \|F(\vtheta^{(0)})\|+ {2\eta(1+C) \over 3} \mathbb{E}[\|\vtheta^{(t-1)}-\vtheta^{(0)}\|] +O(\eta\tilde{\epsilon}).          
        \end{align*}
        Applying the triangle inequality, $\|\vtheta^{(t)}-\vtheta^{(0)}\| \leq \|\vtheta^{(t)}-\vtheta^{(t-1)}\|+\|\vtheta^{(t-1)}-\vtheta^{(0)}\|$, we further deduce that 
        \begin{align*}
        \mathbb{E}[\|\vtheta^{(t)} - \vtheta^{(0)}\|] \leq \Big(1+{2\eta (1+C) \over 3}\Big)  \mathbb{E}[ \|\vtheta^{(t-1)}-\vtheta^{(0)}\|] +\Big( \frac{\eta \|F(\vtheta^{(0)})\|}{6} + O(\eta \tilde{\epsilon}) \Big).
        \end{align*}
        Setting $\Gamma = \frac{3}{2\eta(1+C)} \Big( \frac{\eta \|F(\vtheta^{(0)})\|}{6} + O(\eta \tilde{\epsilon}) \Big),$ we can write 
        \begin{align*}
            \mathbb{E} [ \|\vtheta^{(t)} - \vtheta^{(0)}\| + \Gamma ] \leq \Big(1+{2\eta (1+C) \over 3}\Big) \mathbb{E}[\|\vtheta^{(t-1)}-\vtheta^{(0)}\| + \Gamma],
        \end{align*}
        Thus, with constant $\overline{T}>T_{\textnormal{SGD}},$
        \begin{align*}
            \mathbb{E} [ \|\vtheta^{(T_{\textnormal{SGD}})} - \vtheta^{(0)}\| + \Gamma ] &\leq \Big(1+{2\eta (1+C) \over 3}\Big)^{T_{\textnormal{SGD}}} \Gamma \\
            &\leq \Big(1+{2\eta (1+C) \over 3} T_{\textnormal{SGD}} \Big) \Gamma + \overline{T} \ O(\eta^2 \Gamma).
        \end{align*}
        By Taylor expansion of $F_1$ near $\epsilon \approx 0$: \begin{align*}
        F_1(\epsilon, \epsilon, \epsilon, \epsilon)={2(e^2-1) \over e^2+1}\epsilon+O(\epsilon^3), \quad
        \|F(\vtheta^0)\| = {4(e^2-1) \over 1+e^2} \epsilon + O(\epsilon^3),
        \end{align*}
        we get \[
        \Gamma = \frac{e^2-1}{(1+C)(e^2+1)} \epsilon +O(\epsilon^3+\tilde{\epsilon}) = O(\epsilon+\tilde{\epsilon}).
        \]
        Since $\mathbb{E}[\|\vtheta^{(T_{\textnormal{SGD}})} - \vtheta^{(0)}\|] \geq 2({\pi \over 4} - \rho - \epsilon)$, 
        \begin{align*}
        \frac{2\eta(1+C) \Gamma}{3} T_{\textnormal{SGD}} &\geq \mathbb{E}[\|\vtheta^{(T_{\textnormal{SGD}})} - \vtheta^{(0)}\|]+O(\eta^2 (\epsilon+\tilde{\epsilon}))\\
        &\geq 2({\pi \over 4}-\rho-\epsilon) + O(\eta^2 (\epsilon+\tilde{\epsilon})).
        \end{align*}
        Therefore,
        \begin{align*}
        T_{\textnormal{SGD}} \geq {3(e^2+1) \over e^2-1} {\pi/4-\rho-\epsilon+O(\eta^2 (\epsilon+ \tilde{\epsilon})) \over \eta \epsilon+O(\eta (\epsilon^3+ \tilde{\epsilon}))}.
        \end{align*}

\medskip

\end{proof}

\begin{remark}
    To simply compare the convergence rates of two algorithms, we assumed that there is some constant $\overline{T}$ such that $T_{\textnormal{SGD}}$, $ T_{\textnormal{OSGD}} <\overline{T}$ in Theorem ~\ref{theorem 9}. However, without this assumption, we could still obtain lower bounds of both algorithms as;
    \begin{align*}
        &T_{\textnormal{OSGD}} \geq \frac{2}{\log(1+2\eta)} \log \left[ \frac{{\pi \over 4}-\rho +O(\epsilon^3)}{\epsilon+O(\epsilon^3)} \right], \\
        &T_{\textnormal{SGD}} \geq \frac{1}{\log\big(1+{2(1+C) \over 3}\eta\big)} \log \left[ {1 \over \tilde{C}} \frac{{\pi \over 4}-\rho -(1-\tilde{C})\epsilon+O(\epsilon^3+\tilde{\epsilon})}{\epsilon+O(\epsilon^3+\tilde{\epsilon})} \right],
    \end{align*}
    where $\tilde{C}= {(e^2-1)}/{2(C+1)(e^2+1)}$, $C := \max \limits_{x \in [0, {\pi \over 2}]}[2 \sin x/(1+e^{1+ \cos x})]',$ and their approximations are $\tilde{C} \approx 0.265, C \approx 0.436.$ For small enough $\eta, \epsilon, \tilde{\epsilon},$ we can observe OSGD algorithm converges faster than SGD algorithm, if the inequalities are tight. 
\end{remark}

\paragraph{Direct Application of OSGD and its Convergence}
We now focus exclusively on the convergence of OSGD. We prove Theorem~\ref{thm:osgd_convergence}, which establishes the convergence of an application of OSGD to the mini-batch contrastive learning problem, with respect to the loss function $\lcliptilde$.
\begin{restatable}{algorithm}{osgdalg}
\caption{The direct application of OSGD to our problem}
    \label{alg:osgd-naive}
    \begin{algorithmic}[1]
    \DontPrintSemicolon
    \State \textbf{Parameters: } $k$: the number of batches to be randomly chosen at each iteration; 
    $q$: the number of batches of the largest losses to be chosen among $k$ batches at each iteration; $T$: the number of iterations.
    \State \textbf{Inputs:} an initial feature vector $(\vU^{(0)}, \vV^{(0)})$, the set of learning rates $\{\eta_t\}_{t=0}^{T-1}$. \\
    \For{$t=0$ \KwTo $T-1$}{%
        Randomly choose $S \subset [{N \choose B}]$ with $|S| = k$ \;
        Choose $i_1, \dots, i_q \in S$ having the largest losses, i.e., $\lclip(\vU_{\gB_i}^{(t)}, \vV_{\gB_i}^{(t)})$ \;
        $g \leftarrow \frac{1}{q} \sum_{i \in S} \nabla_{\vU, \vV} \lclip(\vU_{\gB_i}^{(t)}, \vV_{\gB_i}^{(t)})$ \;
        $(\vU^{(t+1)}, \vV^{(t+1)}) \gets (\vU^{(t)}, \vV^{(t)}) - \eta_t g$ \;
        $(\vU^{(t+1)}, \vV^{(t+1)}) \gets \mathrm{normalize}(\vU^{(t+1)}, \vV^{(t+1)})$
    }
    \end{algorithmic}
\end{restatable}
\medskip

For ease of reference, we repeat the following definition:
\begin{equation}\label{eq:osgd-loss}
    \lcliptilde(\vU, \vV) \coloneqq \frac{1}{q} \sum_{j=1}^{{N \choose B}} \gamma_j \lclip(\vU_{\gB_{(j)}}, \vV_{\gB_{(j)}}), \quad \gamma_j = \frac{\sum_{l=0}^{q-1} {j - 1 \choose l}{{N \choose B} - j \choose k - l -1}}{{{N \choose B} \choose k }},
\end{equation}
where $\gB_{(j)}$ represents the batch with the $j$-th largest loss among all possible $\binom{N}{B}$ batches, and $q$, $k$ are parameters for the OSGD.

\osgdconvergence*

\begin{proof}
    Define $(\widehat{\vU}^{(t^\star)}, \widehat{\vV}^{(t^\star)}) = \underset{\vU', \vV'}{\mathrm{argmin}}\left\{
        \lcliptilde(\vU', \vV') + \frac{\rho}{2}\norm{(\vU', \vV') - (\vU^{(t^\star)}, \vV^{(t^\star)})}^2
    \right\}$. We begin by reffering to Lemma 2.2. in \cite{davis2019stochastic}, which provides the following equations:
    \begin{align*}
        \norm{(\vU^{(t^\star)}, \vV^{(t^\star)}) - (\widehat{\vU}^{(t^\star)}, \widehat{\vV}^{(t^\star)})} &= 
        \frac{1}{\rho}\norm{\nabla\lcliptilde_{\rho}(\vU^{(t^\star)}, \vV^{(t^\star)})}, \\
        \norm{\nabla\lcliptilde(\widehat{\vU}^{(t^\star)}, \widehat{\vV}^{(t^\star)})} &\leq
        \norm{\nabla\lcliptilde_{\rho}(\vU^{(t^\star)}, \vV^{(t^\star)})}.
    \end{align*}
    Furthermore, we have that $\nabla\lcliptilde$ is $\rho_0$-Lipschitz in $((B_d(0, 1))^N)^2$ by Thm.~\ref{thm:lipschtiz-of-loss-grad}. 
    This gives
    \[  
        \norm{\nabla\lcliptilde(\vU^{(t^\star)}, \vV^{(t^\star)}) - \nabla\lcliptilde(\widehat{\vU}^{(t^\star)}, \widehat{\vV}^{(t^\star)})}
        \leq
        \rho_0 \norm{(\vU^{(t^\star)}, \vV^{(t^\star)}) - (\widehat{\vU}^{(t^\star)}, \widehat{\vV}^{(t^\star)})}
    \]
    Therefore,
    \begin{align*}
        \norm{\nabla\lcliptilde(\vU^{(t^\star)}, \vV^{(t^\star)})}
        &\leq  \norm{\nabla\lcliptilde(\widehat{\vU}^{(t^\star)}, \widehat{\vV}^{(t^\star)})}
        + \norm{\nabla\lcliptilde(\vU^{(t^\star)}, \vV^{(t^\star)}) - \nabla\lcliptilde(\widehat{\vU}^{(t^\star)}, \widehat{\vV}^{(t^\star)})} \\
        &\leq \norm{\nabla\lcliptilde(\widehat{\vU}^{(t^\star)}, \widehat{\vV}^{(t^\star)})} 
        + \rho_0 \norm{(\vU^{(t^\star)}, \vV^{(t^\star)}) - (\widehat{\vU}^{(t^\star)}, \widehat{\vV}^{(t^\star)})} 
        \\
        &\leq \frac{\rho + \rho_0}{\rho} \norm{\nabla\lcliptilde_{\rho}({\vU}^{(t^\star)}, {\vV}^{(t^\star)})}.
    \end{align*}
    As a consequence of Thm~\ref{thm:osgd_convergence_rho},
    \begin{align*}
        \expect\left[\left\|\nabla \lcliptilde(\vU^{(t^{\star})}, \vV^{(t^{\star})})\right\|^2 \right] 
        &\leq
        \frac{(\rho + \rho_0)^2}{\rho^2}\expect\left[\left\|\nabla \lcliptilde_{\rho}(\vU^{(t^{\star})}, \vV^{(t^{\star})})\right\|^2 \right] \\
        &\leq
        \frac{(\rho + \rho_0)^2}{\rho(\rho-\rho_0)} 
          \frac{\left( \lcliptilde_{\rho}(\vU^{(0)}, \vV^{(0)}) - \lclipstar_{\rho} \right) + 8{\rho}\sum_{t=0}^T\eta_t^2}{\sum_{t=0}^T\eta_t}\\
        &\leq
        \frac{(\rho + \rho_0)^2}{\rho(\rho-\rho_0)} 
          \frac{\left( \lcliptilde(\vU^{(0)}, \vV^{(0)}) - \lclipstar \right) + 8{\rho}\sum_{t=0}^T\eta_t^2}{\sum_{t=0}^T\eta_t}.
    \end{align*}
    Note that $\lclipstar_{\rho}$ is the minimized value of $\lcliptilde_{\rho}$, and the last inequality is due to $\lcliptilde_{\rho}(\vU^{(0)}, \vV^{(0)}) - \lclipstar_{\rho} \leq \lcliptilde(\vU^{(0)}, \vV^{(0)}) - \lclipstar$, because
    \begin{align*}
        \lcliptilde_{\rho}(\vU^{(0)}, \vV^{(0)}) 
        &= \min_{\vU', \vV'} \left\{
            \lcliptilde(\vU', \vV') + \frac{\rho}{2} \norm{(\vU', \vV') - (\vU^{(0)}, \vV^{(0)})}^2
        \right\} \\
        &\leq \lcliptilde(\vU^{(0)}, \vV^{(0)}) 
    \intertext{by putting $(\vU', \vV') = (\vU^{(0)}, \vV^{(0)})$, and }
        \lclipstar &= \min_{\vU', \vV'}\left\{ \lcliptilde(\vU', \vV') \right\} \\
        &\leq \min_{\vU', \vV'}\left\{
        \lcliptilde(\vU', \vV') + \frac{\rho}{2} \norm{(\vU', \vV') - (\vU, \vV)}^2
        \right\} \\
        &= \lcliptilde_{\rho}(\vU, \vV)
    \end{align*}
    for any $\vU$, $\vV$, implying that $\lclipstar \leq \lclipstar_{\rho}$.
\end{proof}

We provide details, including proof of theorems and lemmas in the sequel. 

\begin{theorem}\label{thm:osgd_convergence_rho}
    Consider sampling $t^{\star}$ from $[T]$ with probability $\prob(t^{\star} = t) = {\eta_t}/{(\sum_{i=0}^{T} \eta_i)}$. 
    Then $\forall\rho > \rho_0 = 2\sqrt{2/B} + 4e^2 / B$, we have
    \[\expect\left[\left\|\nabla \lcliptilde_{{\rho}}(\vU^{(t^{\star})}, \vV^{(t^{\star})})\right\|^2 \right] 
          \leq \frac{{\rho}}{\rho-\rho_0} 
          \frac{\left( \lcliptilde_{{\rho}}(\vU^{(0)}, \vV^{(0)}) - \lclipstar_{{\rho}} \right) + 8{\rho}\sum_{t=0}^T\eta_t^2}{\sum_{t=0}^T\eta_t},
		\]
   where $\lcliptilde_{{\rho}}(\vU, \vV) \coloneqq \min \limits_{\vU', \vV'} \left\{ \lcliptilde(\vU', \vV') + \frac{{\rho}}{2} \norm{(\vU', \vV') - (\vU, \vV)}^2  \right\}$, and $\lclipstar_{\rho}$ denotes the minimized value of $\lcliptilde_{\rho}$.
\end{theorem}

\begin{proof}
    $\nabla \lcliptilde$ is $\rho_0$-Lipschitz in $((B_d(0, 1))^N)^2$ by Thm.~\ref{thm:lipschtiz-of-loss-grad}.
    Hence, it is $\rho_0$-weakly convex by Lem.~\ref{lem:lip-grad-weak-convex}.
    Furthermore, the gradient norm of a mini-batch loss, or $\norm{\nabla_{\vU, \vV} \lclip(\vU_{\gB_i}, \vV_{\gB_i})}$ is bounded by $L = 4$.
    Finally, \cite[Theorem~1]{kawaguchi2020ordered} states that the expected value of gradients of the OSGD algorithm is $\nabla_{\vU, \vV} \lcliptilde(\vU^{(t)}, \vV^{(t)})$ at each iteration $t$. 
    Therefore, we can apply \cite[Thm.~3.1]{davis2019stochastic} to the OSGD algorithm to obtain the desired result.
\end{proof}

Roughly speaking, Theorem~\ref{thm:osgd_convergence} shows that $(\vU^{(t^\star)}, \vV^{(t^\star)})$ are close to a stationary point of $\lcliptilde_{\rho}$.
We refer readers to \citet{davis2019stochastic} which illustrates the role of the norm of the gradient of the Moreau envelope, $\norm{\nabla \lcliptilde_{\rho}(\vU^{(t^\star)}, \vV^{(t^\star)})}$, being small in the context of stochastic optimization.

We leave the results of some auxiliary theorems and lemmas to Subsection~\ref{sec:aux-lem}.

\subsection{Auxiliaries for the Proof of Theorem~\ref{thm:osgd_convergence}} \label{sec:aux-lem}
	For a square matrix $A$, we denote its trace by $\mathrm{tr}(A)$.
	If matrices $A$ and $C$ are of the same shape, we define the canonical inner product $\dotp{A, C}$ by
	\[
		\dotp{A, C} = \sum_{i, j} A_{ij} C_{ij} = \mathrm{tr}(A^\intercal C).
	\]

Following a pythonic notation, we write $A_{i, :}$ and $A_{:, j}$ for the $i$-th row and $j$-th column of a matrix $A$, respectively.
The Cauchy--Schwarz inequality for matrices is given by
	\[
		\dotp{A, C} \leq \norm{A} \norm{C},
	\]
 where a norm $\norm{\cdot}$ is a Frobenius norm in matrix i.e. $\|A\|=\Big(\sum \limits_{i, j} A_{ij}^2\Big)^{1/2}.$
\begin{lemma}\label{lem:mat-mul-norm-bound}
    Let $A \in \sR^{m \times n}$, $C \in \sR^{n \times k}$. 
    Then, $\norm{AC} \leq \norm{A}\norm{C}$.
\end{lemma}
	\begin{proof}
		By a basic calculation, we have
		\begin{align*}
			\norm{AC}^2 = \mathrm{tr}(C^\intercal A^\intercal A C) = \mathrm{tr}(CC^\intercal A^\intercal A) 
			= \dotp{CC^\intercal, A^\intercal A} 
			\leq \norm{CC^\intercal} \norm{A^\intercal A}.
		\end{align*}
		Meanwhile, for any positive semidefinite matrix $D$, let $D = U\Lambda U^\intercal$ be a spectral decomposition of $D$.
		Then, we have
		\begin{align*}
			\mathrm{tr}(D^2) &= \mathrm{tr}(U \Lambda^2 U^\intercal) 
			= \mathrm{tr}(\Lambda^2 U^\intercal U) 
			= \mathrm{tr}(\Lambda^2) 
			\leq (\mathrm{tr}(\Lambda))^2 %
			= (\mathrm{tr}(D))^2,
		\end{align*}
		where $\lambda_i(D)$ denotes the $i$-th eigenvalue of a matrix $D$.
		Invoking this fact, we have
		\[
			\norm{CC^\intercal}^2 = \mathrm{tr}((CC^\intercal)^2) \leq (\mathrm{tr}(CC^\intercal))^2 = \norm{C}^4,
		\]
		or equivalently, $\norm{CC^\intercal} \leq \norm{C}^2$. Similarly, we have $\norm{A^\intercal A} = \norm{A}^2$. 
		Therefore, we obtain
		\[
			\norm{AC}^2 \leq \norm{CC^\intercal}\norm{A^\intercal A} \leq \norm{A}^2 \norm{C}^2,
		\]
		which means $\norm{AC} \leq \norm{A}\norm{C}$.
	\end{proof}
	
	If $\gL \colon \sR^{m \times n} \to \sR$ is a function of a matrix $X \in \sR^{m \times n}$, we write a gradient of $\gL$ with respect to $X$ as a matrix-valued function defined by
	\[
		(\nabla_X \gL)_{ij} = \bigg(\frac{\partial \gL}{\partial X}\bigg)_{ij} = \frac{\partial \gL}{\partial X_{ij}}.
	\]
	Then, the chain rule gives
	\[
		\frac{d}{dt}\gL(X) = \bigg\langle \frac{dX}{dt}, \nabla_X \gL \bigg\rangle
	\]
	for a scalar variable $t$.
        If $\gL(\vU, \vV)$ is a function of two matrices $\vU$, $\vV \in \sR^{m \times n}$, we define $\nabla_{\vU, \vV} \gL$ as a horizontal stack of two gradient matrices, i.e., $\nabla_{\vU, \vV} \gL = (\nabla_\vU \gL, \nabla_\vV \gL)$.

Now, we briefly review some necessary facts about Lipschitz functions.

\begin{lemma}[Rendering of weak convexity by a Lipschitz gradient]\label{lem:lip-grad-weak-convex}
		Let $f \colon \sR^d \to \sR$ be a $\rho$-smooth function, i.e., $\nabla f$ is a $\rho$-Lipschitz function.
		Then, $f$ is $\rho$-weakly convex.
	\end{lemma}

	\begin{proof}
		For the sake of simplicity, assume $f$ is twice differentiable.
		We claim that $\nabla^2 f \succeq -\rho \mathbb{I}_d$, where $\mathbb{I}_d$ is the $d \times d$ identity matrix and $A \succeq B$ means $A - B$ is a positive semidefinite matrix.
        It is clear that this claim renders $f + \frac{\rho}{2}\norm{\cdot}^2$ to be convex.
		
		Let us assume, contrary to our claim, that there exists $\vx_0 \in \sR^d$ with $\nabla^2 f(\vx_0) \not\succeq -\rho \mathbb{I}_d$.
		Therefore, $\nabla^2 f(\vx_0)$ has an eigenvalue $\lambda < -\rho$.
		Denote corresponding eigenvector by $\vu$, so we have $\nabla^2 f(\vx_0) \vu = \lambda \vu$,
		and consider $g(\epsilon) = \nabla f(\vx_0 + \epsilon \vu)$; 
		the (elementwise) Taylor expansion of $g$ at $\epsilon = 0$ gives
		\[
			\nabla f(\vx_0 + \epsilon \vu) = \nabla f(\vx_0) + \epsilon \nabla^2 f(\vx_0) \vu + o(\epsilon),
		\]
		which gives
		\[
			\frac{\norm{\nabla f(\vx_0 + \epsilon \vu) - \nabla f(\vx_0)}}{\epsilon}
			= \left\| \nabla^2 f(\vx_0) \vu + \frac{o(\epsilon)}{\epsilon} \right\|.
		\]
		Taking $\epsilon \to 0$, we obtain $\norm{\nabla f(\vx_0 + \epsilon \vu) - \nabla f(\vx_0)}/\epsilon \geq \abs{\lambda} > \rho$, which is contradictory to $\rho$-Lipschitzness of $\nabla f$.
	\end{proof}	
 
	For $X \in \sR^{B \times B}$, let us define
	\[
		\gL^M(X) = \frac{1}{B}
            \left( -2\mathrm{tr}(X) + \sum_{i=1}^B \log \sum_{j=1}^B \exp(X_{ij}) + \sum_{i=1}^B \log \sum_{j=1}^B \exp(X_{ji}) \right).
	\]
        Using this function, we can write the loss corresponding to a mini-batch $\gB$ of size $B$ by
        \[
            \mathcal{L}^M(\vU_\gB^\intercal \vV_\gB) = \lclip(\vU_\gB, \vV_\gB). 
        \]
        We now claim the following:
\begin{lemma}\label{lem:loss_M_is_lip_and_bdd}
		Consider $X \in \sR^{B \times B}$, where $\abs{X_{ij}} \leq 1$ for all $1 \leq i, j \leq B$.
		Then, $\nabla_X \gL^M(X)$ is bounded by $2\sqrt{2/B}$ and $2e^2 /B^2$-Lipschitz.
	\end{lemma}
	\begin{proof}
		With basic calculus rules, we obtain
		\begin{align}
			B \nabla_X \gL^M(X) &= -2\mathbb{I}_B + P_X + Q_X \label{eq:grad-contL}, 
		\end{align}
		where $\mathbb{I}_B$ is the $B \times B$ identity matrix and
		\[
			(P_X)_{ij} = \exp (X_{ij}) / \sum_{k=1}^B \exp (X_{ik}), \quad (Q_X)_{ij} = \exp (X_{ij}) / \sum_{k=1}^B \exp (X_{kj}).
		\]
		From $\sum_j P_{ij} = 1$ for all $i$, it is easy to see that $\norm{(\mathbb{I}_B - P)_{i, :}}^2 \leq 2$. 
		This gives $\norm{\mathbb{I}_B - P_X}^2 \leq 2B$, and similarly $\norm{\mathbb{I}_B - Q_X}^2 \leq 2B$.
		Therefore, we have
		\begin{equation}
			\norm{B \nabla_X \gL^M(X)} \leq \norm{\mathbb{I}_B - P_X} + \norm{\mathbb{I}_B - Q_X} \leq 2\sqrt{2B},
		\end{equation}
            or equivalently 
            \begin{equation}\label{eq:grad-loss-bound}
			\norm{\nabla_X \gL^M(X)} \leq 2\sqrt{2/B}.
		\end{equation}
            
            We now show that $\nabla_X \gL^M$ is $\frac{2e^2}{B^2}$-Lipschitz.
            Define $p \colon \R^B \to \R^B$ by 
            \[
                (p(x))_i = \frac{\exp(x_i)}{\sum_{k=1}^B \exp(x_k)}.
            \]
            Then, we have 
            \[
                \frac{\partial}{\partial x} p(x) = \mathrm{diag}(p(x)) - p(x)p(x)^\intercal.
            \]
            For $x \in [-1, 1]^B$, we have $p(x)_i \leq \frac{e^2}{B - 1 + e^2} < \frac{e^2}{B}$ for any $i$.
            Thus,
            \[
                0 \preceq \frac{\partial}{\partial x} p(x) \preceq \mathrm{diag}(p(x)) \preceq \frac{e^2}{B} \mathbb{I}_B, 
            \]
            which means $p(x)$ is $\frac{e^2}{B}$-Lipschitz, i.e., $\norm{p(x) - p(y)} \leq \frac{e^2}{B} \norm{x-y}$ for any $x$, $y \in [-1, 1]^B$.
            Using this fact, we can bound $\norm{P_X - P_Y}$ for $X$, $Y \in [-1, 1]^{B \times B}$ as follows:
            \[
                \norm{P_X - P_Y}^2
                = \sum_{i = 1}^B \norm{p(X_{i, :}) - p(Y_{i, :})}^2
                \leq \left(\frac{e^2}{B}\right)^2 \sum_{i = 1}^B \norm{X_{i, :} - Y_{i, :}}^2
                = \left(\frac{e^2}{B}\right)^2 \norm{X - Y}^2.
            \]
            Similarly, we have $\norm{Q_X - Q_Y} \leq \frac{e^2}{B} \norm{X - Y}$. 
            Summing up, 
            \[
                \norm{B\nabla_X \gL^M(X) - B\nabla_X \gL^M(Y)} \leq \norm{P_X - P_Y} + \norm{Q_X - Q_Y} \leq \frac{2e^2}{B}\norm{X - Y}.
            \]
            which renders
            \[
                \norm{\nabla_X \gL^M(X) - \nabla_X \gL^M(Y)} \leq \frac{2e^2}{B^2} \norm{X - Y}.
            \]
\end{proof}

	Recall that $\lclip(\vU_\gB, \vV_\gB) = \gL^M(\vU_\gB^\intercal \vV_\gB)$ for $\vU_\gB$, $\vV_\gB \in \sR^{d \times B}$ (They correspond to embeddings corresponding to a mini-batch $\gB$).
	Using this relation, we can calculate the gradient of $\lclip$ with respect to $\vU_\gB$. 
	Denote $E_{ij} \in \sR^{d \times B}$ a one-hot matrix, which is a matrix of zero entries except for $(i, j)$ indices being $1$, and write $G = \nabla_X \gL^M(\vU_\gB^\intercal \vV_\gB)$.
	Then, 
	\begin{align*}
		\frac{\partial}{\partial {(\vU_\gB)}_{ij}} \lclip(\vU_\gB, \vV_\gB)
		&= \bigg\langle \frac{\partial (\vU_\gB^\intercal \vV_\gB)}{\partial {\vU_\gB}_{ij}}, \nabla_X \gL^M(\vU_\gB^\intercal \vV_\gB) \bigg\rangle \\
		&= \bigg\langle E_{ij}^\intercal \vV_\gB, G \bigg\rangle \\
		&= \mathrm{tr}\bigg( \vV_\gB^\intercal E_{ij} G \bigg) \\
		&= \mathrm{tr}\bigg( E_{ij} (G \vV_\gB^\intercal)\bigg) \\ 
		&= (G \vV_\gB^\intercal)_{ji} \\
		&= (\vV_\gB G^\intercal)_{ij}.
	\end{align*}
	This elementwise relation means
	\begin{align}
		\frac{\partial}{\partial \vU_\gB} \lclip(\vU_\gB, \vV_\gB) &= \vV_\gB G^\intercal = \vV_\gB (\nabla_X \gL^M(\vU_\gB^\intercal \vV_\gB))^\intercal, \label{eq:lclip-grad-U}\\
		\intertext{and similarly,}
		\frac{\partial}{\partial \vV_\gB} \lclip(\vU_\gB, \vV_\gB) &= \vU_\gB \nabla_X \gL^M(\vU_\gB^\intercal \vV_\gB). \label{eq:lclip-grad-V}
	\end{align}

	We introduce a simple lemma for bounding the difference between two multiplication of matrices.
	\begin{lemma}\label{lem:matrix-mul-diff-bdd}
		For $A_1$, $A_2 \in \sR^{m \times n}$ and $B_1$, $B_2 \in \sR^{n \times k}$, we have 
		\[
			\norm{A_1 B_1 - A_2 B_2} \leq \norm{A_1 - A_2} \norm{B_1} + \norm{A_2} \norm{B_1 - B_2}.
		\]
	\end{lemma}

	\begin{proof}
		This follows from a direct calculation and Lemma~\ref{lem:mat-mul-norm-bound}
		\begin{align*}
			\norm{A_1 B_1 - A_2 B_2}
			&= \norm{A_1 B_1 - A_1 B_2 + A_1 B_2 - A_2 B_2} \\
			&\leq \norm{A_1 (B_1 - B_2)} + \norm{(A_1 - A_2) B_2} \\
			&\leq \norm{A_1 - A_2} \norm{B_1} + \norm{A_2} \norm{B_1 - B_2}.
		\end{align*}
	\end{proof}

	\begin{theorem}\label{thm:bdd-of-grad}
		For any $\vU$, $\vV \in (B_d(0, 1))^N$ and any batch $\gB$ of size $B$, we have
		$\norm{\nabla_{\vU, \vV} \lclip(\vU_\gB, \vV_\gB)} \leq 4$.
	\end{theorem}
	\begin{proof}
		Suppose $\vU_\gB$, $\vV_\gB \in (B_d(0, 1))^B$, we have
		\[
			\nabla_{\vU_\gB, \vV_\gB} \lclip(\vU_\gB, \vV_\gB) = (\vV_\gB(\nabla_X \gL^M (\vU_\gB^\intercal \vV_\gB))^\intercal, 
			\vU_\gB \nabla_X \gL^M (\vU_\gB^\intercal \vV_\gB))
		\]
		from Eq.~(\ref{eq:lclip-grad-U}) and~(\ref{eq:lclip-grad-V}).
By following the fact that $\norm{\vU_\gB}$, $\norm{\vV_\gB} \leq \sqrt{B}$ and $\nabla_X \gL^M (X)\leq 2\sqrt{2/B}$ (see Lem.~\ref{lem:loss_M_is_lip_and_bdd}), we get
\begin{align*}
    \norm{\vV_\gB(\nabla_X \gL^M (\vU_\gB^\intercal \vV_\gB))^\intercal}
    &\leq \norm{\vV_\gB} \norm{\nabla_X \gL^M (\vU_\gB^\intercal \vV_\gB)}
    \leq 2\sqrt{2}, \\
    \intertext{and}
    \norm{\vU_\gB \nabla_X \gL^M (\vU_\gB^\intercal \vV_\gB)} 
    &\leq \norm{\vU_\gB} \norm{\nabla_X \gL^M (\vU_\gB^\intercal \vV_\gB)}
    \leq 2\sqrt{2}.
\end{align*}		
Then, 
$$\norm{\nabla_{\vU_\gB, \vV_\gB}\lclip(\vU_\gB, \vV_\gB)}
=\sqrt{\norm{\vV_\gB(\nabla_X \gL^M (\vU_\gB^\intercal \vV_\gB))^\intercal}^2+\norm{\vU_\gB \nabla_X \gL^M (\vU_\gB^\intercal \vV_\gB)}^2}
\leq 4.$$
  
Since $\lclip(\vU_\gB, \vV_\gB)$ is independent of $\vU_{[N]\setminus\gB}$ and $\vV_{[N]\setminus\gB}$, we have 
$$\norm{\nabla_{\vU, \vV}\lclip(\vU_\gB, \vV_\gB)} = \norm{\nabla_{\vU_\gB, \vV_\gB}\lclip(\vU_\gB, \vV_\gB)} \leq 4.$$
\end{proof}

	\begin{theorem}\label{thm:lipschtiz-of-loss-grad}
		$\nabla \lcliptilde (\vU, \vV)$ is $\rho_0$-Lipschitz for $\vU$, $\vV \in (B_d(0, 1))^N$, or to clarify,
		\begin{align*}
			\norm{\nabla \lcliptilde (\vU^1, \vV^1) - \nabla \lcliptilde (\vU^2, \vV^2)} &\leq \rho_0 \norm{(\vU^1, \vV^1) - (\vU^2, \vV^2)} 
		\end{align*}
		for any $\vU^1$, $\vV^1$, $\vU^2$, $\vV^2 \in (B_d(0, 1))^N$, where $\rho_0 = 2\sqrt{2/B} + 4e^2/B$.
	\end{theorem}

	\begin{proof}
            Denoting $\vU_\gB^i$, $\vV_\gB^i$ as parts of $\vU^i$, $\vV^i$ that correspond to a mini-batch $\gB$,
            we first show $\norm{\nabla_{\vU_\gB, \vV_\gB} \lclip(\vU_\gB^1, \vV_\gB^1) - \nabla_{\vU_\gB, \vV_\gB} \lclip(\vU_\gB^2, \vV_\gB^2)} \leq \rho_0 \norm{(\vU_\gB^1, \vV_\gB^1) - (\vU_\gB^2, \vV_\gB^2)}$ holds.
            For any $\vU_\gB$, $\vV_\gB \in (B_d(0, 1))^B$, we have
		\[
			\nabla_{\vU_\gB, \vV_\gB} \lclip(\vU_\gB, \vV_\gB) = (\vV_\gB(\nabla_X \gL^M (\vU_\gB^\intercal \vV_\gB))^\intercal, 
			\vU_\gB \nabla_X \gL^M (\vU_\gB^\intercal \vV_\gB)).
		\]
		from Eq.~\ref{eq:lclip-grad-U} and Eq.~\ref{eq:lclip-grad-V}.
		Recall Lemma~\ref{lem:loss_M_is_lip_and_bdd}; 
		for any $\vU_\gB^i$, $\vV_\gB^i \in (B_d(0, 1))^B$ ($i = 1, 2$), we have
		\begin{align*}
			\norm{\nabla_X \gL^M ((\vU_\gB^i)^\intercal \vV_\gB^i)}
			&\leq 2\sqrt{2/B}
			\intertext{and}
			\norm{\nabla_X \gL^M ((\vU_\gB^1)^\intercal \vV_\gB^1) - \nabla_X \gL^M ((\vU_\gB^2)^\intercal \vV_\gB^2)}
			&\leq \frac{2e^2}{B^2} \norm{(\vU_\gB^1)^\intercal \vV_\gB^1 - (\vU_\gB^2)^\intercal \vV_\gB^2}.
		\end{align*}
		We invoke Lemma~\ref{lem:matrix-mul-diff-bdd} and obtain
		\begin{align*}
			&\norm{\vU_\gB^1 \nabla_X \gL^M ((\vU_\gB^1)^\intercal \vV_\gB^1) - \vU_\gB^2 \nabla_X \gL^M ((\vU_\gB^2)^\intercal \vV_\gB^2)} \\
			&\leq 
			\norm{\vU_\gB^1 - \vU_\gB^2} \norm{\nabla_X \gL^M ((\vU_\gB^1)^\intercal \vV_\gB^1)}
			+ \norm{\vU_\gB^2} \norm{\nabla_X \gL^M ((\vU_\gB^1)^\intercal \vV_\gB^1) - \nabla_X \gL^M ((\vU_\gB^2)^\intercal \vV_\gB^2)} \\
			&\leq 
			2\sqrt{2/B} \norm{\vU_\gB^1 - \vU_\gB^2} 
			+ \frac{2e^2}{B^{3/2}} \norm{(\vU_\gB^1)^\intercal \vV_\gB^1 - (\vU_\gB^2)^\intercal \vV_\gB^2} \\
			&\leq
			2\sqrt{2/B} \norm{\vU_\gB^1 - \vU_\gB^2} 
			+ \frac{2e^2}{B^{3/2}} (\norm{\vU_\gB^1 - \vU_\gB^2} \norm{\vV_\gB^1} + \norm{\vU_\gB^2} \norm{\vV_\gB^1 - \vV_\gB^2}) \\
			&\leq
			(2\sqrt{2/B} + 2e^2/B) \norm{\vU_\gB^1 - \vU_\gB^2} 
			+ (2e^2/B) \norm{\vV_\gB^1 - \vV_\gB^2}, \\
			\intertext{and similarly}
			&\norm{\vV_\gB^1 \nabla_X (\gL^M ((\vU_\gB^1)^\intercal \vV_\gB^1))^\intercal - \vV_\gB^2 \nabla_X (\gL^M ((\vU_\gB^2)^\intercal \vV_\gB^2))^\intercal} \\
			&\leq  (2e^2/B) \norm{\vU_\gB^1 - \vU_\gB^2} 
			+ (2\sqrt{2/B} + 2e^2/B) \norm{\vV_\gB^1 - \vV_\gB^2}.
		\end{align*}
		Using the fact that 
		\[
			(ax + by)^2 + (bx + ay)^2 = (a^2 + b^2)(x^2 + y^2) + 4ab xy
			\leq (a + b)^2 (x^2 + y^2)
		\]
		holds for any $a$, $b \geq 0$ and $x$, $y \in \sR$, we obtain
		\begin{align*}
			&\norm{\nabla \lclip(\vU_\gB^1, \vV_\gB^1) - \nabla \lclip(\vU_\gB^2, \vV_\gB^2)}^2 \\
			&= \norm{\vV_\gB^1 \nabla_X (\gL^M ((\vU_\gB^1)^\intercal \vV_\gB^1))^\intercal - \vV_\gB^2 \nabla_X (\gL^M ((\vU_\gB^2)^\intercal \vV_\gB^2))^\intercal}^2 \\
			&\quad + \norm{\vU_\gB^1 \nabla_X \gL^M ((\vU_\gB^1)^\intercal \vV_\gB^1) - \vU_\gB^2 \nabla_X \gL^M ((\vU_\gB^2)^\intercal \vV_\gB^2)}^2 \\
			&\leq (2\sqrt{2/B} + 4e^2/B)^2 (\norm{\vU_\gB^1 - \vU_\gB^2}^2 + \norm{\vV_\gB^1 - \vV_\gB^2}^2) \\
			&= (2\sqrt{2/B} + 4e^2/B)^2 \norm{(\vU_\gB^1, \vV_\gB^1) - (\vU_\gB^2, \vV_\gB^2)}^2.
		\end{align*}
		Restating this with $\rho_0 = 2\sqrt{2/B} + 4e^2/B$, we have
		\begin{equation}\label{eq:lclip-lips}
			\norm{\nabla \lclip(\vU_\gB^1, \vV_\gB^1) - \nabla \lclip(\vU_\gB^2, \vV_\gB^2)}
			\leq \rho_0 \norm{(\vU_\gB^1, \vV_\gB^1) - (\vU_\gB^2, \vV_\gB^2)}.
		\end{equation}
		
		Recall the definition of $\lcliptilde$:
		\[
			\lcliptilde(\vU, \vV) = \frac{1}{q} \sum_{j} \gamma_j \lclip(\vU_{\gB_{(j)}}, \vV_{\gB_{(j)}}),
		\]
		where $\gamma_j = \frac{\sum_{l=0}^{q-1} {j - 1 \choose l}{{N \choose B} - j \choose k - l -1}}{{{N \choose B} \choose k}}$ and $\sum_j \gamma_j = q$.
		For any $\vU$, $\vV \in (\sS^d)^N$, we can find a neighborhood of $(\vU, \vV)$ so that value rank of $\lclip(\vU_{\gB_i}, \vV_{\gB_i})$ over $i \in \{1, \ldots, {N \choose B}\}$ does not change, since $\lclip$ is $\rho_0$-Lipschitz.
		More precisely speaking, we can find a rank that can be accepted by all points in the neighborhood.
		Therefore, we have
		\[
			\nabla_{\vU, \vV} \lcliptilde (\vU, \vV) = \frac{1}{q} \sum_{j} \gamma_j \nabla_{\vU, \vV} \lclip(\vU_{\gB_{(j)}}, \vV_{\gB_{(j)}}),
		\]
		and since $\norm{\vU_{\gB_{(j)}} - \vV_{\gB_{(j)}}} \leq \norm{\vU - \vV}$, $\nabla_{\vU, \vV} \lcliptilde (\vU, \vV)$ is locally $\rho_0$-Lipschitz.
		Since $\lclip$ is smooth, such property is equivalent to $-\rho_0 \mathbb{I}_N \preceq \nabla^2_{\vU, \vV}\lcliptilde(\vU, \vV) \preceq \rho_0 \mathbb{I}_N$, where $\mathbb{I}_N$ is the $N \times N$ identity matrix.
		Therefore, $\lcliptilde$ is $\rho_0$-Lipschitz on $((B_d(0, 1))^N)^2$.
	\end{proof}
\section{Algorithm Details}\label{sec:alg_detail}
\subsection{Spectral Clustering Method}\label{sec:alg_detail_sc}
Here, we provide a detailed description of the proposed spectral clustering method (see Sec.~\ref{subsec:spectral_clustering}) from Algo.~\ref{alg:spectral_clustering}. Recall that the contrastive loss $\mathcal{L}^{\sf{con}}(U_{\gB}, V_{\gB})$ for a given mini-batch $\gB$ is lower bounded as the following by Jensen's inequality:
{\small
\begin{align*}
    &\begin{aligned}
    &\lclip(\vU_{\gB},\vV_{\gB}) 
    = -\frac{1}{B} \sum_{i\in \gB} \log \left(\frac{e^{\vu_i^{\intercal} \vv_i}}{\sum_{j=1}^N e^{{\vu}_i^{\intercal} {\vv}_j}}\right)
    -\frac{1}{B} \sum_{i=\in\gB} \log \left(\frac{e^{\vv_i^{\intercal} \vu_i}}{\sum_{j=1}^N e^{{\vv}_i^{\intercal} {\vu}_j}}\right)\\
    &= \frac{1}{B} \left\{\sum_{i\in \gB} \log \left(1+\sum_{j\in\gB\setminus \{i\}}e^{\vu_i^{\intercal}(\vv_j-\vv_i)})\right)
    +\sum_{i\in \gB} \log \left(1+\sum_{j\in\gB\setminus \{i\}}e^{\vv_i^{\intercal}(\vu_j-\vu_i)})\right)\right\}\\
        & \geq
        \frac{1}{B(B-1)}\left\{\sum_{i\in\gB}\sum_{j \in \gB \setminus \{i\}}\log\left(1+(B-1)e^{\vu_i^\intercal(\vv_j-\vv_i)}\right)+\log\left(1+(B-1)e^{\vv_i^\intercal(\vu_j-\vu_i)}\right)\right\},
    \end{aligned}
\end{align*}
}
and we consider the graph $\gG$ with $N$ nodes, where the weight between node $k$ and $l$ is defined as 
\begin{align*}
    w(k,l):= \sum_{(i,j)\in\{(k,l), (l,k)\}}\log\left(1+(B-1)e^{\vu_i^\intercal(\vv_j-\vv_i)}\right)+\log\left(1+(B-1)e^{\vv_i^\intercal(\vu_j-\vu_i)}\right).
\end{align*}
The proposed method employs the spectral clustering algorithm from~\cite{ng2001spectral}, which bundles $N$ nodes into $N/B$ clusters. We aim to assign an equal number of nodes to each cluster, but we encounter a problem where varying numbers of nodes are assigned to different clusters. To address this issue, we incorporate an additional step to ensure that each cluster (batch) has the equal number $B$ of positive pairs. This step is to solve an assignment problem~\citep{kuhn1955hungarian, crouse2016implementing}. We consider a minimum weight matching problem in a bipartite graph~\cite{crouse2016implementing}, where the first partite set is the collection of data points and the second set represents $B$ copies of each cluster center obtained after the spectral clustering. The edges in this graph are weighted by the distances between data points and centers. The goal of the minimum weight matching problem is to assign exactly $B$ data points to each center, minimizing the total cost of the assignment, where cost is the sum of the distances from each data point to its assigned center. This guarantees an equal number of data points for each cluster while minimizing the total assignment cost. A annotated procedure of the method is provided in Algo.~\ref{alg:sc_appendix}.
\begin{algorithm}[H]
\footnotesize
   \caption{Spectral Clustering Method}
   \label{alg:sc_appendix}
   \DontPrintSemicolon
   \KwIn{the number of positive pairs $N$, mini-batch size $B$, 
   embedding matrices: $\vU$, $\vV$}
   \KwOut{selected mini-batches $\{\gB_j\}_{j=1}^{N/B}$}
   \BlankLine
   Construct the affinity matrix $A$: 
   \quad\qquad $A_{ij} = \begin{cases} w(i,j)  & \text{if } i\neq j \\ 0 & \text{else} \end{cases}$\;
   Construct the degree matrix $D$ from $A$:
   $D_{ij} = \begin{cases} 0 & \text{if } i\neq j \\ \sum_{j=1}^N A_{ij} & \text{else} \end{cases}$\;
   $L \leftarrow D-A$; $k \leftarrow N/B$\; 
   Compute the first $k$ eigenvectors of $L$, denoted as $V_k \in \mathbb{R}^{N \times k}$\;
   Normalize the rows of $V_k$ to have unit $\ell_2$-norm\;
   Apply the $k$-means clustering algorithm on the rows of the normalized $V_k$ to get cluster centers $Z\in\mathbb{R}^{k\times k}$\;
   Construct a bipartite graph $\gG_{\sf assign}$: (i) the first partite set is $V_k$ and (ii) the second set is the collection of $B$ copies of each center in $Z$\;
   Compute distances between row vectors of $V_k$ and  $B$ copies of each center in $Z$, and assign these as edge weights in $\gG_{\sf assign}$\;
   Solve the minimum weight matching problem in $\gG_{\sf assign}$ using a method such as the Hungarian algorithm\;
   \Return $\{\gB_j\}_{j=1}^{N/B}$
\end{algorithm}
                
\vspace{10mm}
\subsection{Stochastic Gradient Descent (SGD)}\label{sec:alg_detail_sgd}
We consider two SGD algorithms: 
\begin{enumerate}[leftmargin=0.7cm]
    \item SGD \emph{with replacement} (Algo.~\ref{alg:sgd_w_replacement}) with $k=1$ for the theoretical analysis in Sec.~\ref{subsec:toy_example}.
    \item SGD \emph{without replacement} (Algo.~\ref{alg:sgd_wo_replacement}) for experimental results in Sec.~\ref{sec:exp}, which is widely employed in practical settings.
\end{enumerate}
In the more practical setting where $\vu_i = f_{\theta}(\vx_i)$ and $\vv_i = g_{\phi}(\vy_i)$, SGD updates the model parameters $\theta,\phi$ using the gradients $\frac{1}{k} \sum_{i \in S_{\gB}} \nabla_{\theta, \phi}\lclip(\vU_{\gB_i}, \vV_{\gB_i})$ instead of explicitly updating $\vU$ and $\vV$.
\begin{algorithm}[H]
\footnotesize
\label{alg:sgd_w_replacement}
   \caption{SGD with replacement}
   \DontPrintSemicolon
   \KwIn{the number of positive pairs $N$, mini-batch size $B$, the number of mini-batches $k$, the number of iterations $T$, the learning rate $\eta$, 
   initial embedding matrices: $\vU$, $\vV$}
   \BlankLine
   \For{$t=1$ \KwTo $T$}{
    Randomly select $k$ mini-batch indices $S_{\gB} \subset \left[\binom{N}{B}\right]$ $(|S_{\gB}|=k)$\;
    Compute the gradient: $g \leftarrow \frac{1}{k} \sum_{i \in S_{\gB}} \nabla_{\vU, \vV} \lclip(\vU_{\gB_i}, \vV_{\gB_i})$ \;
    Update the weights: $(\vU, \vV) \leftarrow (\vU, \vV) - \eta^{(t)} \cdot g$\;
    Normalize column vectors of embedding matrices $(\vU, \vV)$\;
    }
\end{algorithm}

\begin{algorithm}[H]
\footnotesize
\label{alg:sgd_wo_replacement}
   \caption{SGD without replacement}
   \DontPrintSemicolon
   \KwIn{the number of positive pairs $N$, mini-batch size $B$, the number of mini-batches $k$, the number of epochs $E$, the learning rate $\eta$, 
   initial embedding matrices: $\vU$, $\vV$}
   \BlankLine
   \For{$e=1$ \KwTo $E$}{
   Randomly partition the $N$ positive pairs into $N/B$ mini-batches: $\{\mathcal{B}_i\}_{i=1}^{N/B}$\;
   \For{$j=1$ \KwTo $N/Bk$}{
    Select $k$ mini-batch indices $S_{\gB}=\{k(j-1)+1,k(j-1)+2,\ldots,kj\}$\;
    Compute the gradient: $g \leftarrow \frac{1}{k} \sum_{i \in S_{\gB}} \nabla_{\vU, \vV} \lclip(\vU_{\gB_i}, \vV_{\gB_i})$ \;
    Update the weights: $(\vU, \vV) \leftarrow (\vU, \vV) - \eta \cdot g$\;
    Normalize column vectors of embedding matrices $(\vU, \vV)$\;
    }
    }
\end{algorithm}

\subsection{Ordered SGD (OSGD)}\label{sec:alg_detail_osgd}
We consider two OSGD algorithms: 
\begin{enumerate}[leftmargin=0.7cm]
    \item OSGD (Algo.~\ref{alg:osgd}) with $k=\binom{N}{B}$ for the theoretical analysis in Sec.~\ref{subsec:toy_example}.
    \item OSGD \emph{without replacement} (Algo.~\ref{alg:osgd_wo_replacement}) for experimental results in Sec.~\ref{sec:exp}, which is implemented for practical settings.
\end{enumerate}
In the more practical setting where $\vu_i = f_{\theta}(\vx_i)$ and $\vv_i = g_{\phi}(\vy_i)$, OSGD updates the model parameters $\theta,\phi$ using the gradients $\frac{1}{k} \sum_{i \in S_{\gB}} \nabla_{\theta, \phi}\lclip(\vU_{\gB_i}, \vV_{\gB_i})$ instead of explicitly updating $\vU$ and $\vV$.
\begin{algorithm}[H]
\footnotesize
\label{alg:osgd}
   \caption{OSGD}
   \DontPrintSemicolon
   \KwIn{the number of positive pairs $N$, mini-batch size $B$, the number of mini-batches $k$, the number of iterations $T$, the set of learning rates $\{\eta^{(t)}\}_{t=1}^{T}$, 
   initial embedding matrices: $\vU$, $\vV$}
   \BlankLine
   \For{$t=1$ \KwTo $T$}{
    Randomly select $k$ mini-batch indices $S_{\gB} \subseteq \left[\binom{N}{B}\right]$ $(|S_{\gB}|=k)$\;
    Choose $q$ mini-batch indices $S_q:=\{i_1, i_2,\ldots, i_q\}\subset S_{\gB}$ having the largest losses i.e., $\lclip(\vU_{\gB_i}, \vV_{\gB_i})$\;
    Compute the gradient: $g \leftarrow \frac{1}{q} \sum_{i \in S_{q}} \nabla_{\vU, \vV} \lclip(\vU_{\gB_i}, \vV_{\gB_i})$ \;
    Update the weights: $(\vU, \vV) \leftarrow (\vU, \vV) - \eta^{(t)} \cdot g$\;
    Normalize column vectors of embedding matrices $(\vU, \vV)$\;
    }
\end{algorithm}
\begin{algorithm}[H]
\footnotesize
\label{alg:osgd_wo_replacement}
   \caption{OSGD without replacement}
   \DontPrintSemicolon
   \KwIn{the number of positive pairs $N$, mini-batch size $B$, the number of mini-batches $k$, the number of epochs $E$, the set of learning rate $\eta$, 
   initial embedding matrices: $\vU$, $\vV$}
   \BlankLine
   \For{$e=1$ \KwTo $E$}{
   Randomly partition the $N$ positive pairs into $N/B$ mini-batches: $\{\mathcal{B}_i\}_{i=1}^{N/B}$\;
   \For{$j=1$ \KwTo $N/Bk$}{
    Select $k$ mini-batch indices $S_{\gB}=\{k(j-1)+1,k(j-1)+2,\ldots,kj\}$\;
    Choose $q$ mini-batch indices $S_q:=\{i_1, i_2,\ldots, i_q\}\subset S_{\gB}$ having the largest losses i.e., $\lclip(\vU_{\gB_i}, \vV_{\gB_i})$\;
    Compute the gradient: $g \leftarrow \frac{1}{k} \sum_{i \in S_{q}} \nabla_{\vU, \vV} \lclip(\vU_{\gB_i}, \vV_{\gB_i})$ \;
    Update the weights: $(\vU, \vV) \leftarrow (\vU, \vV) - \eta \cdot g$\;
    Normalize column vectors of embedding matrices $(\vU, \vV)$\;
    }
    }
\end{algorithm}

\section{Experiment Details}\label{sec:exp_detail}

In this section, we describe the details of the experiments in Sec.~\ref{sec:exp} and provide additional experimental results. First, we present histograms of mini-batch counts for different loss values from models trained with different batch selection methods. Next, we provide the results for $N\in\{4, 16\}$ on the synthetic dataset. Lastly, we explain the details of the experimental settings on real dataset, and provide the results of the retrieval downstream tasks.

\subsection{Batch Counts: SC method vs. Random Batch Selection}\label{sec:batch_counts_appendix}
We provide additional results comparing the mini-batch counts of two batch selection algorithms: the proposed SC method and random batch selection. The mini-batch counts are based on the mini-batch contrastive loss $\lclip(\vU_{\gB}, \vV_{\gB})$. We measure mini-batch losses from ResNet-18 models trained on CIFAR-100 using the gradient descent algorithm with different batch selection methods: (i) SGD (Algo.~\ref{alg:sgd_wo_replacement}), (ii) OSGD (Algo.~\ref{alg:osgd_wo_replacement}), and (iii) the SC method (Algo.~\ref{alg:sc_appendix}). Fig.~\ref{fig:loss_histogram_appendix} illustrates histograms of mini-batch counts for $N/B$ mini-batches, where $N=50000$ and $B=20$. The results show that mini-batches generated through the proposed spectral clustering method tend to contain a higher proportion of large loss values when compared to the random batch selection, regardless of the pre-trained models used. 

\begin{figure}[!t]
\label{fig:loss_histogram_appendix}
\centering
\includegraphics[width=0.98\columnwidth]{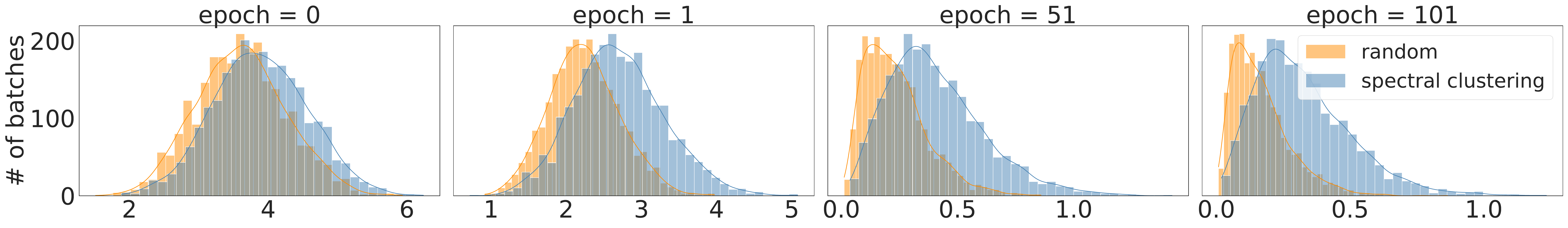}
\includegraphics[width=0.98\columnwidth]{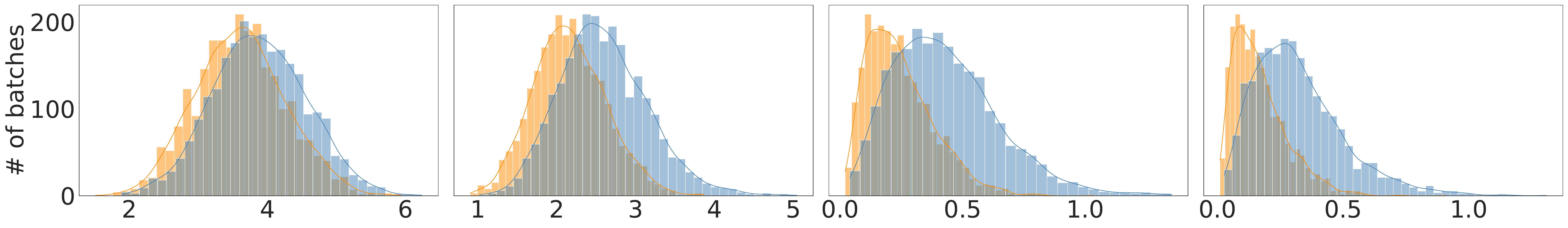}
\includegraphics[width=0.98\columnwidth]{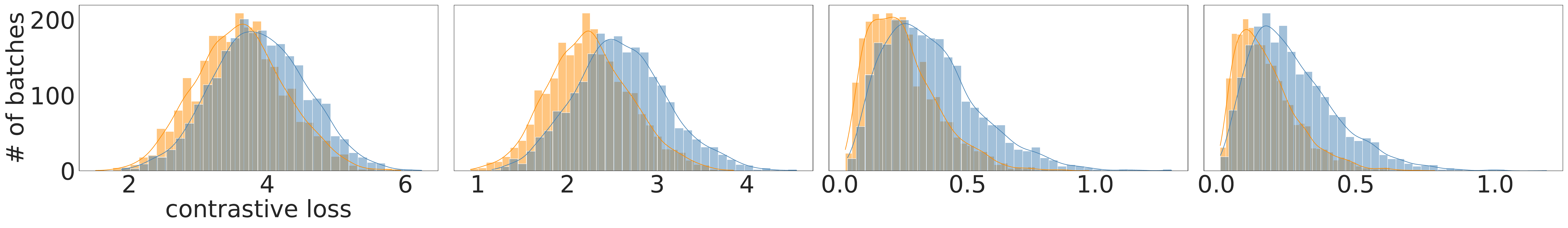}
\vspace{-3mm}
\caption{Histograms of mini-batch counts for $N/B$ mini-batches, for the contrastive loss measured from ResNet-18 models trained on CIFAR-100 using different batch selection methods: (i) SGD (Top), (ii) OSGD (Middle), (iii) SC method (Bottom), where $N$=50,000 and $B$=20. Each column of plots is derived from a distinct training epoch. 
Here we compare two batch selection methods: 
(i) randomly shuffling $N$ samples and partition them into $N/B$ mini-batches of size $B$, (ii) the proposed SC method given in Algo.~\ref{alg:spectral_clustering}. 
The histograms show that mini-batches generated through the proposed spectral clustering method tend to contain a higher proportion of large loss values when compared to random batch selection, regardless of the pre-trained models used.
}
\end{figure}

\subsection{Synthetic Dataset}\label{sec:synthetic_exp_appendix}

With the settings from Sec.~\ref{sec:synthetic_exp}, where each column of embedding matrices $\vU, \vV$ is initialized as a multivariate normal vector and then normalized as $\lVert \vu_i \rVert = \lVert \vv_i \rVert = 1$, for all $i$, we provide the results for $N\in\{4, 16\}$ and $d=2N$ or $d=N/2$. Fig.~\ref{fig:sim_theorem_base_N_4} and ~\ref{fig:sim_theorem_base_N_16} show the results for $N=4$ and $N=16$, respectively. We additionally present the results for theoretically unproven cases, specifically for $N=8$ and $d\in\{3, 5\}$ (see Fig.~\ref{fig:sim_theorem_unproven}). The results provide empirical evidence that all combinations of mini-batches leads to the optimal solution of full-batch minimization for the theoretically unproven cases.

\begin{figure*}[!t]
\centering

\raisebox{2.4mm}[0mm][0mm]{\rotatebox[origin=l]{90}{\scriptsize{\textsc{$d=2N$}}}}\hspace{1mm}
    \includegraphics[height=.12\columnwidth]{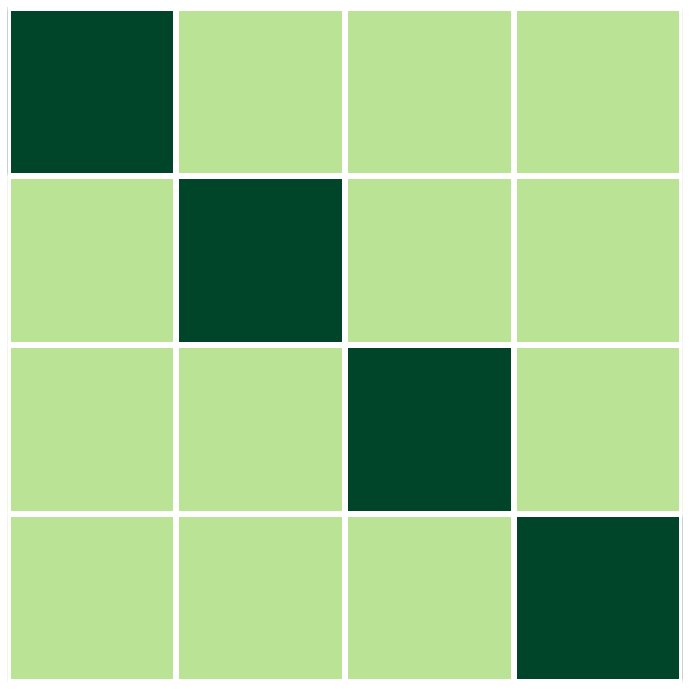}
    \label{fig:N4_d8_B2_ETF_wo_cbar}
    \hspace{1mm}
    \includegraphics[height=.12\columnwidth]{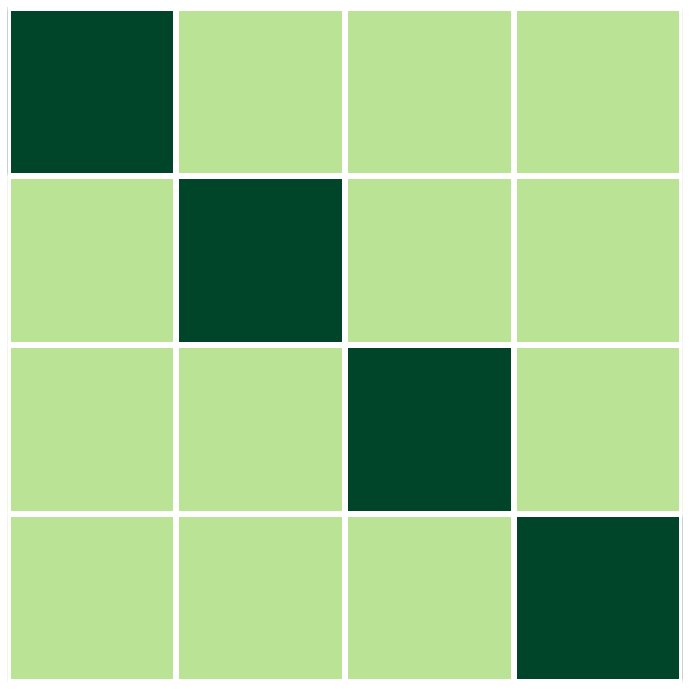}
    \label{fig:d=2N_N4_d8_lr0.5_s50000_z_fixed_mini_batch_B2_full_wo_cbar}
    \hspace{1mm}
    \includegraphics[height=.12\columnwidth]{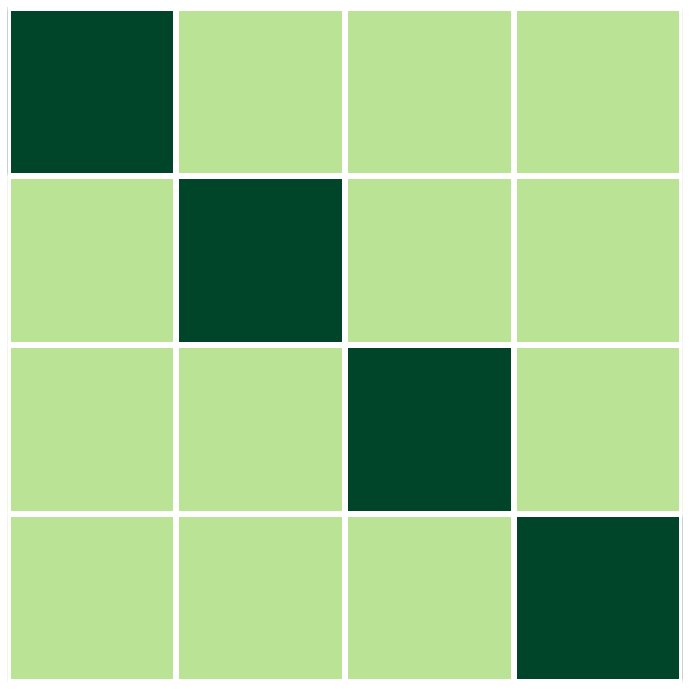}
    \label{fig:d=2N_N4_d8_lr0.5_s20000_z_fixed_mini_batch_B2_NcB_wo_cbar}
    \hspace{1mm}
    \includegraphics[height=.12\columnwidth]{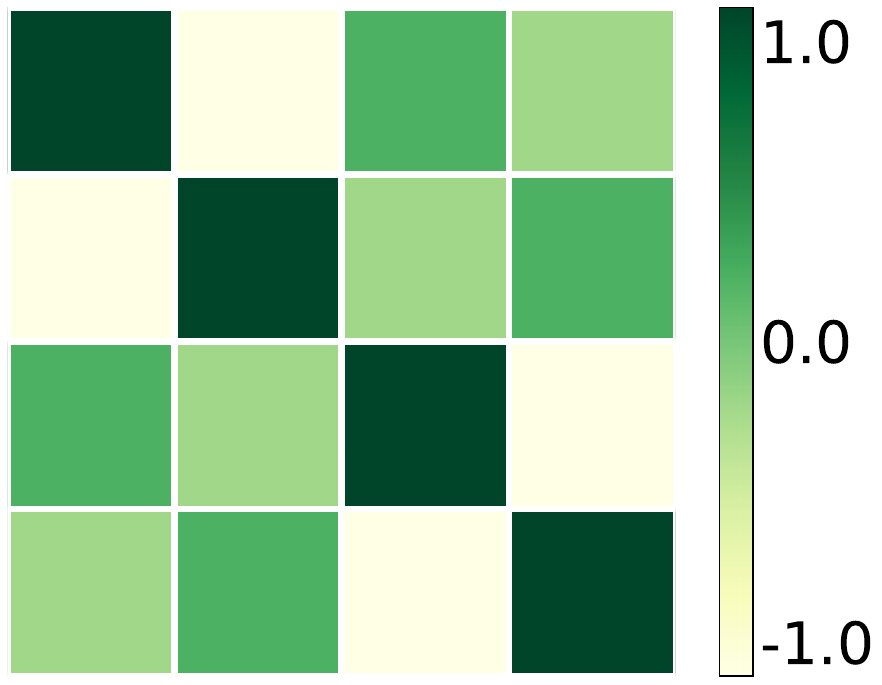}
    \label{fig:d=2N_N4_d8_lr0.5_s20000_z_fixed_mini_batch_B2_f_w_cbar}
    \hspace{1mm}
\includegraphics[height=.13\columnwidth]{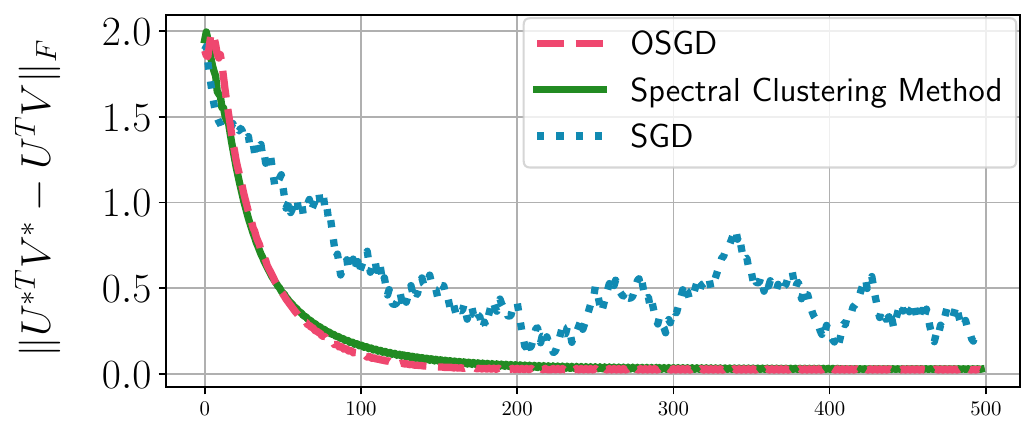}
\label{fig:d=2N_norm_diff_N4_d8}
\hspace{1mm}

\raisebox{2.2mm}[0mm][0mm]{\rotatebox[origin=l]{90}{\scriptsize{\textsc{$d=N/2$}}}
\hspace{-2.2mm}}\hspace{1mm}
\subfigure[solutions]{
    \includegraphics[height=.12\columnwidth]{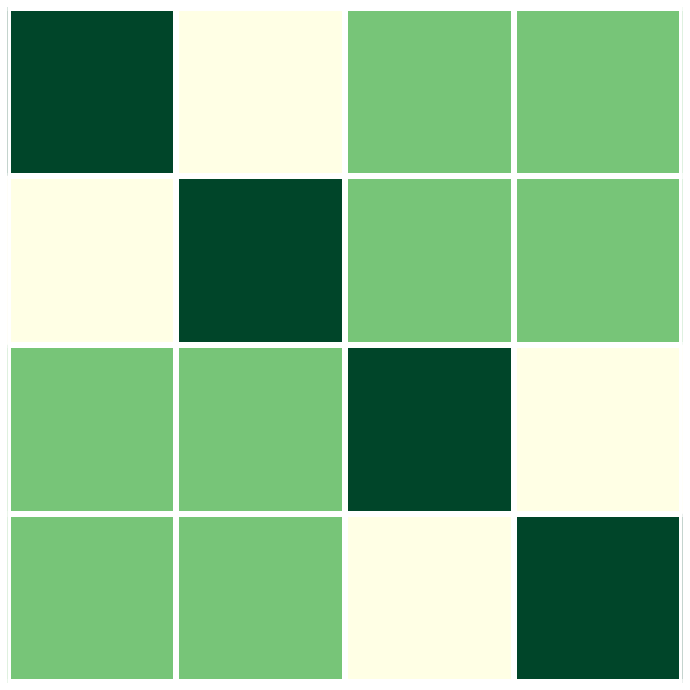}
    \label{fig:N4_d2_B2_CP_wo_cbar}
}
\subfigure[full-batch]{
    \includegraphics[height=.12\columnwidth]{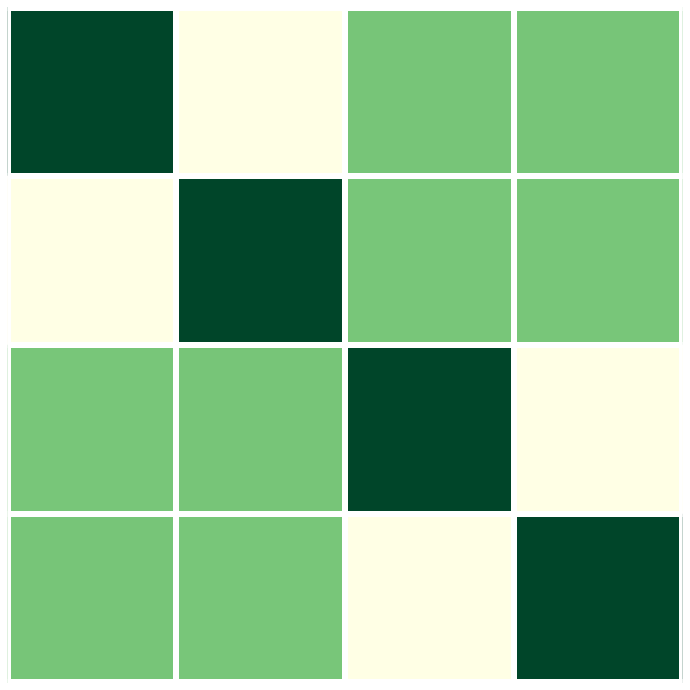}
    \label{fig:d=N_d_2_N4_d2_lr0.5_s50000_z_fixed_mini_batch_B2_full_wo_cbar}
}
\subfigure[${N\choose B}$-all]{
    \includegraphics[height=.12\columnwidth]{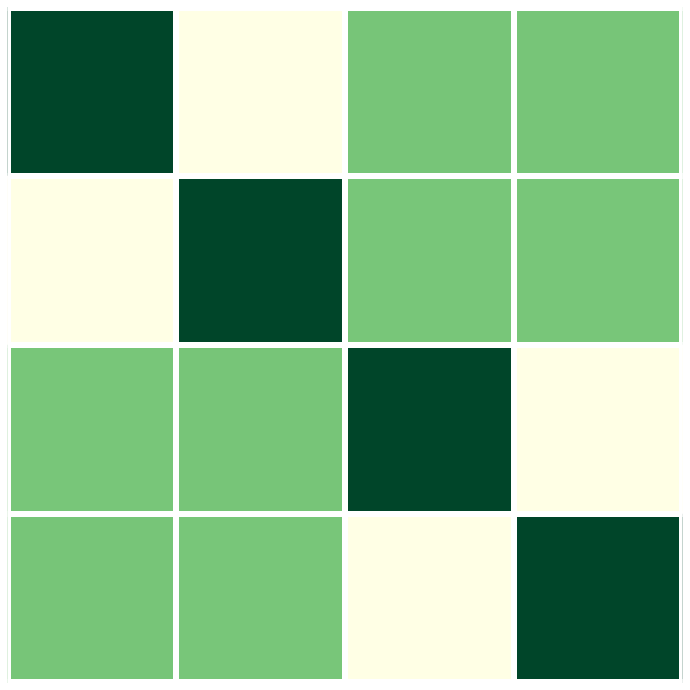}
    \label{fig:d=N_d_2_N4_d2_lr0.5_s50000_z_fixed_mini_batch_B2_NcB_wo_cbar}
}
\subfigure[${N\choose B}$-sub]{
    \includegraphics[height=.12\columnwidth]{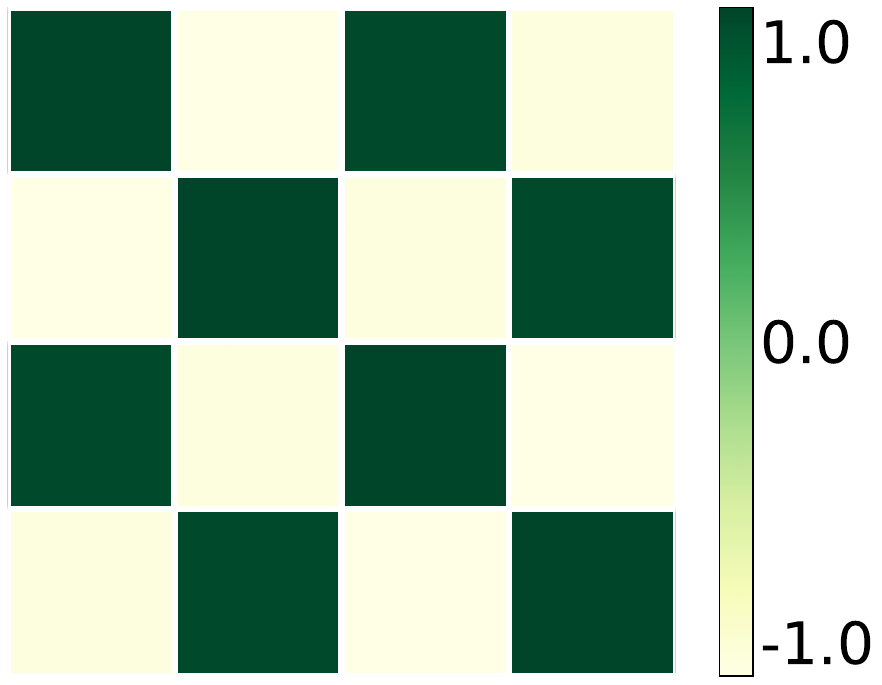}
    \label{fig:d=N_d_2_N4_d2_lr0.5_s50000_z_fixed_mini_batch_B2_f_w_cbar}
}
\subfigure[norm differences]{
    \includegraphics[height=.13\columnwidth]{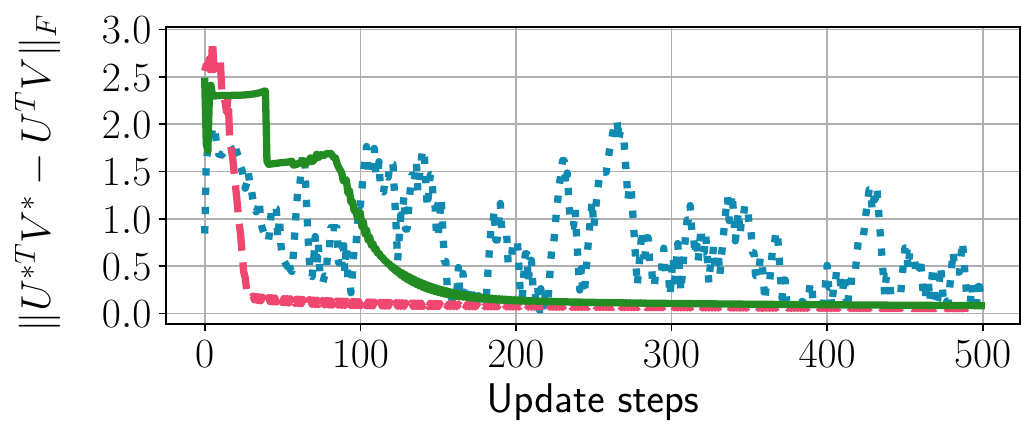}
    \label{fig:d=2N_norm_diff_N4_d2}
    \hspace{1mm}
}

\vspace{-3mm}
\caption{
Heatmap of $N \times N$ matrix visualizing the resulting values from the same settings with Fig~\ref{fig:sim_theorem_base_N_8} except $N=4$.
}
\label{fig:sim_theorem_base_N_4}
\end{figure*}
\vspace{-2mm}

\begin{figure*}[!t]
\centering

\raisebox{2.4mm}[0mm][0mm]{\rotatebox[origin=l]{90}{\scriptsize{\textsc{$d=2N$}}}}\hspace{1mm}
    \includegraphics[height=.12\columnwidth]{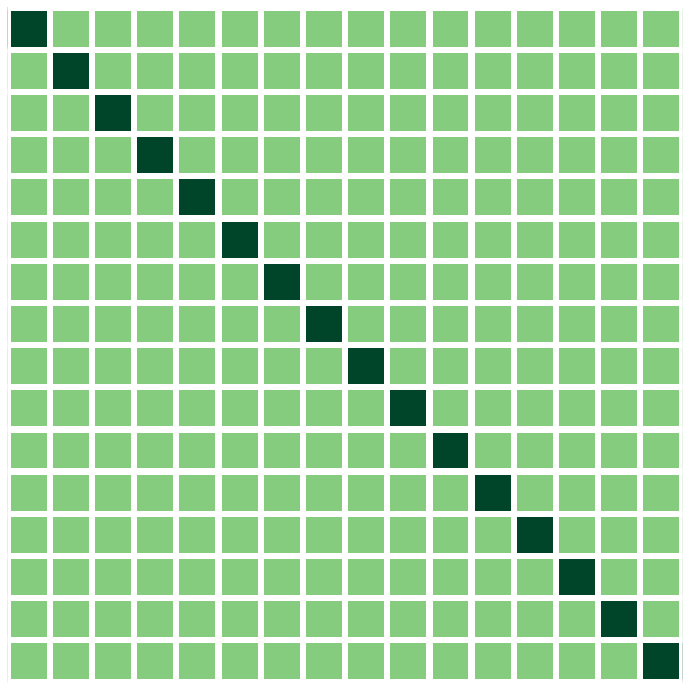}
    \label{fig:N16_d32_B2_ETF_wo_cbar}
    \hspace{1mm}
    \includegraphics[height=.12\columnwidth]{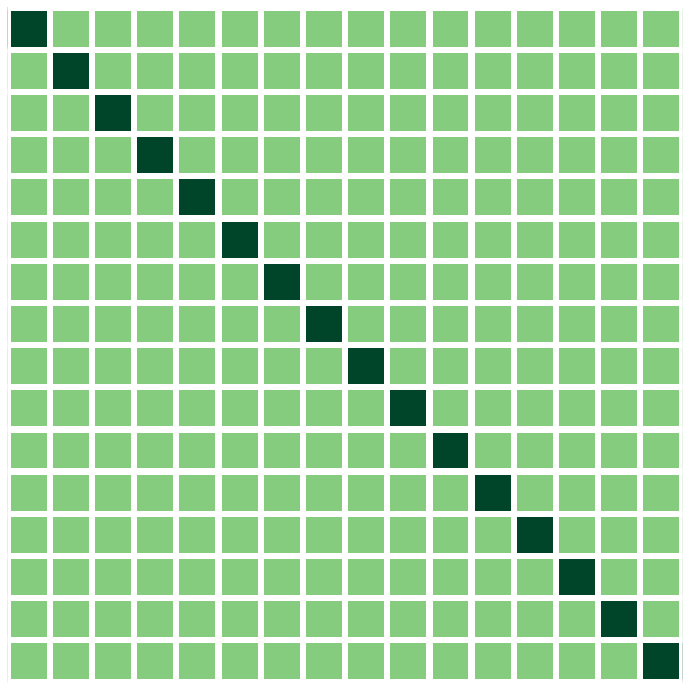}
    \label{fig:d=2N_N16_d32_lr0.5_s50000_z_fixed_mini_batch_B2_full_wo_cbar}
    \hspace{1mm}
    \includegraphics[height=.12\columnwidth]{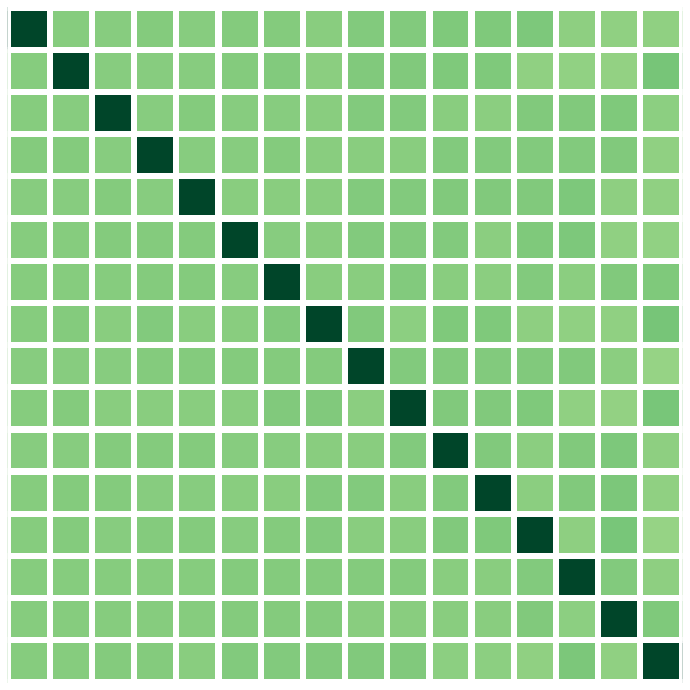}
    \label{fig:d=2N_N16_d32_lr0.5_s50000_z_fixed_mini_batch_B2_NcB_wo_cbar}
    \hspace{1mm}
    \includegraphics[height=.12\columnwidth]{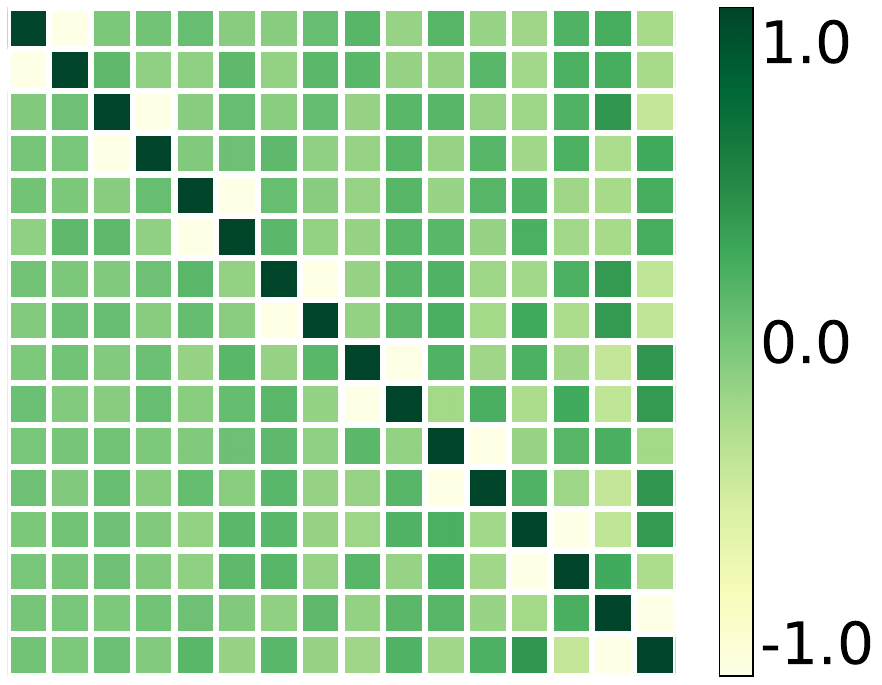}
    \label{fig:d=2N_N16_d32_lr0.5_s50000_z_fixed_mini_batch_B2_f_w_cbar}
    \hspace{1mm}
\includegraphics[height=.13\columnwidth]{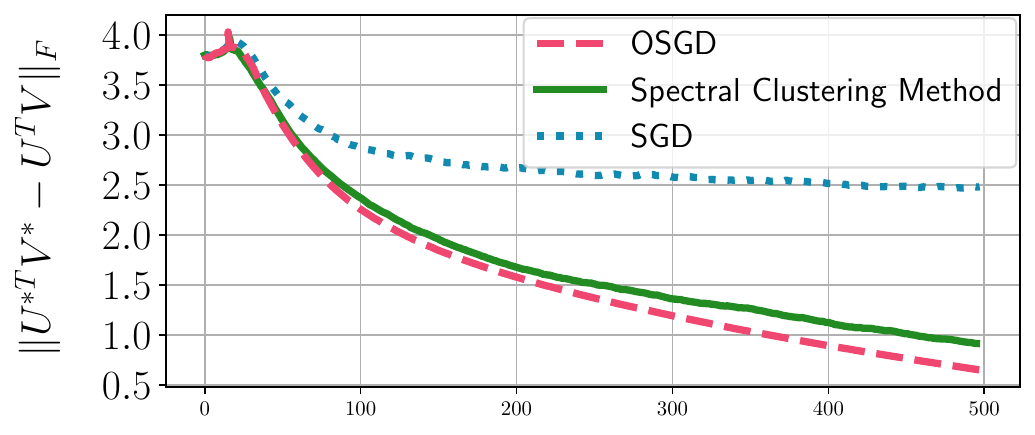}
\hspace{1mm}

\raisebox{2.2mm}[0mm][0mm]{\rotatebox[origin=l]{90}{\scriptsize{\textsc{$d=N/2$}}}
\hspace{-2.2mm}}\hspace{1mm}
\subfigure[solutions]{
    \includegraphics[height=.12\columnwidth]{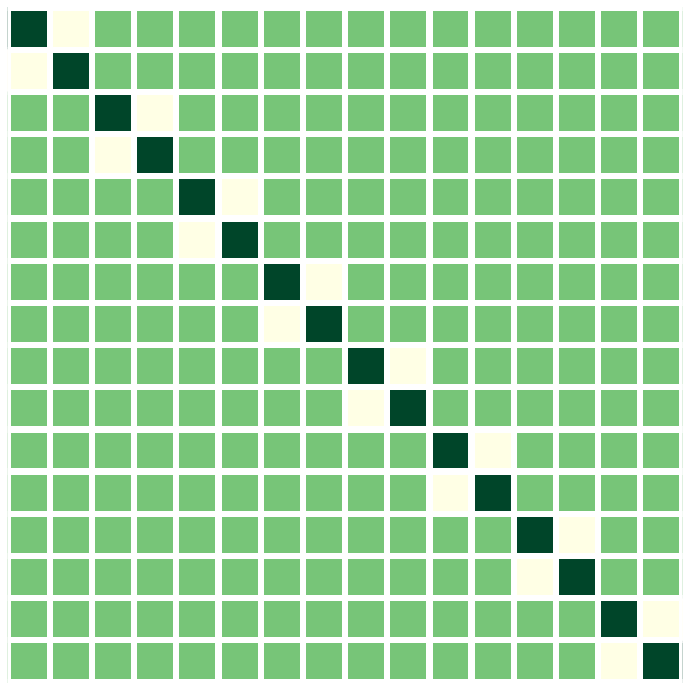}
    \label{fig:N16_d8_B2_CP_wo_cbar}
}
\subfigure[full-batch]{
    \includegraphics[height=.12\columnwidth]{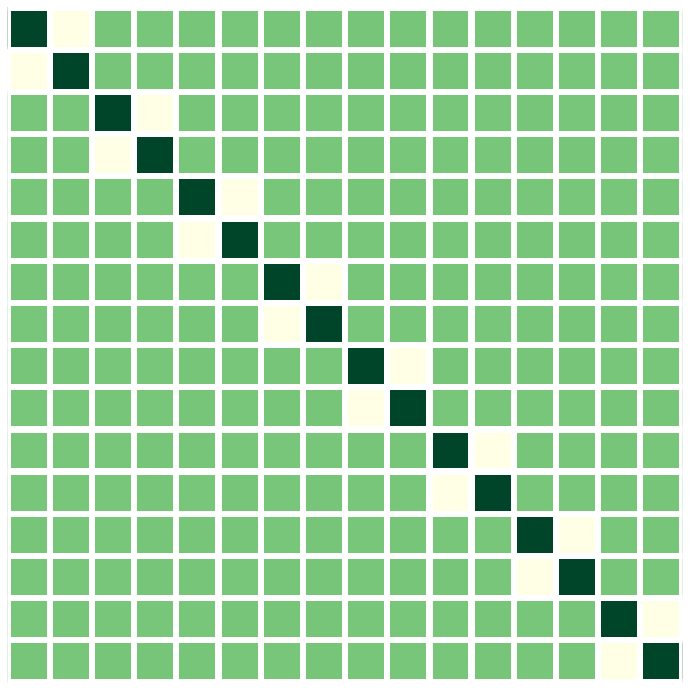}
    \label{fig:d=N_d_2_N16_d8_lr0.5_s500000_z_fixed_mini_batch_B2_full_wo_cbar}
}
\subfigure[${N\choose B}$-all]{
    \includegraphics[height=.12\columnwidth]{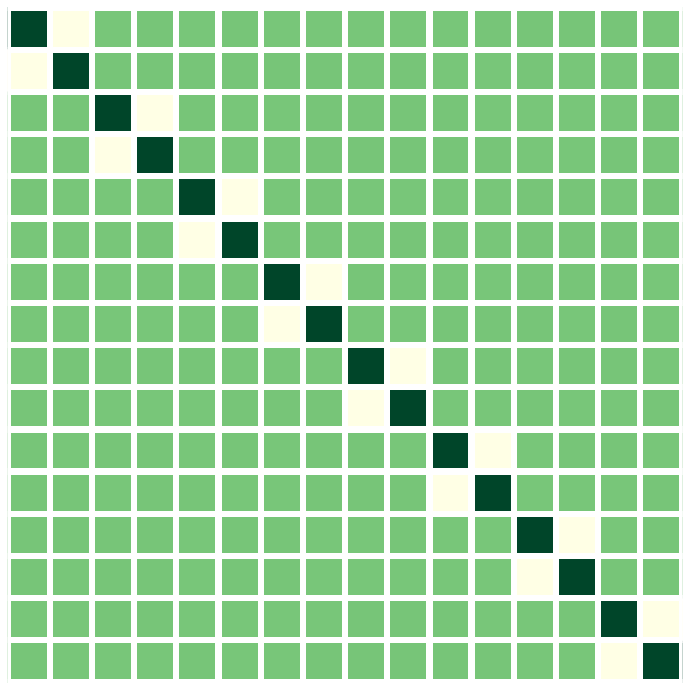}
    \label{fig:d=N_d_2_N16_d8_lr0.5_s500000_z_fixed_mini_batch_B2_NcB_wo_cbar}
}
\subfigure[${N\choose B}$-sub]{
    \includegraphics[height=.12\columnwidth]{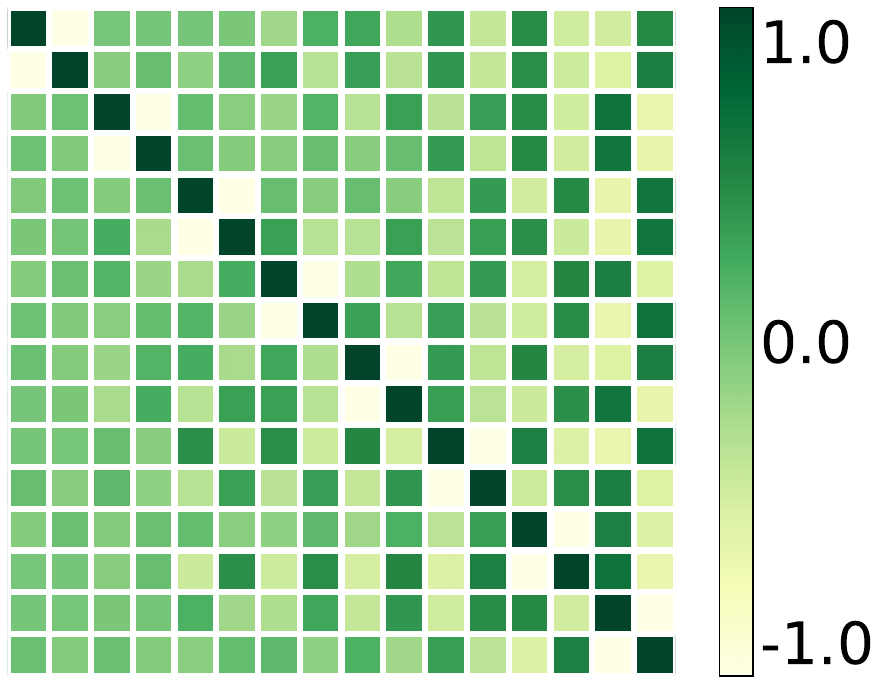}
    \label{fig:d=N_d_2_N16_d8_lr0.5_s500000_z_fixed_mini_batch_B2_f_w_cbar}
}
\subfigure[norm differences]{
    \includegraphics[height=.13\columnwidth]{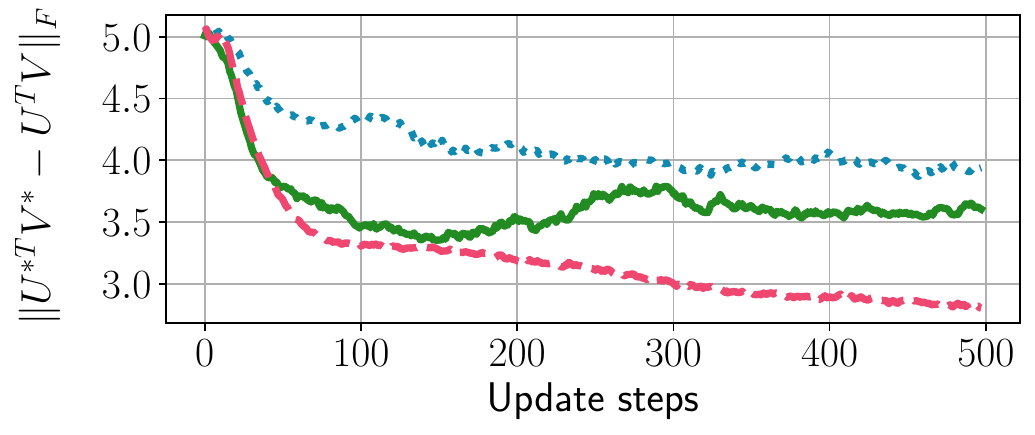}
    \label{fig:norm_diff_N16_d8}
    \hspace{1mm}
}

\vspace{-3mm}
\caption{
Heatmap of $N \times N$ matrix visualizing the resulting values from the same settings with Fig~\ref{fig:sim_theorem_base_N_8} except $N=16$.
}
\label{fig:sim_theorem_base_N_16}
\end{figure*}
\vspace{-2mm}

\begin{figure*}[!t]
\centering

\raisebox{4.0mm}[0mm][0mm]{\rotatebox[origin=l]{90}{\scriptsize{\textsc{$d=3$}}}}\hspace{1mm}\vspace{3mm}
    \includegraphics[height=.12\columnwidth]{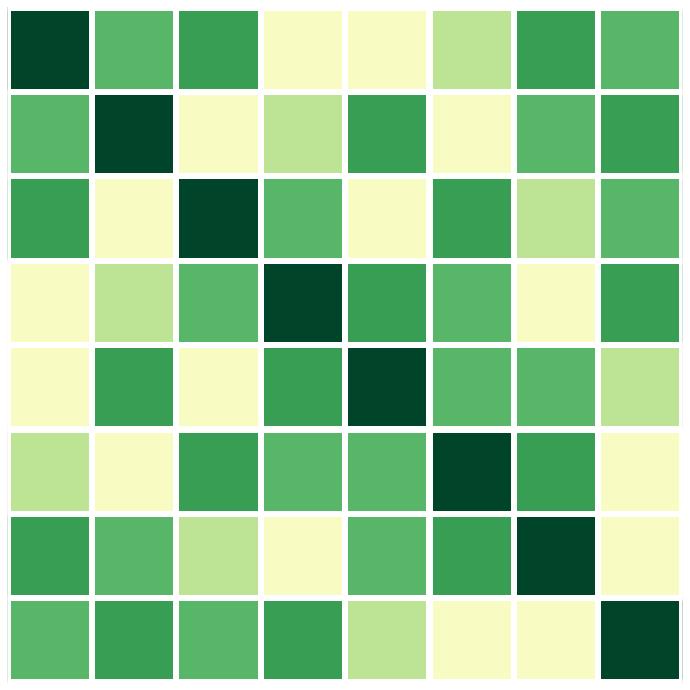}
    \label{fig:N=8_N8_d3_lr0.5_s20000_z_fixed_mini_batch_B2_full_step8000_wo_cbar}
    \hspace{1mm}
    \includegraphics[height=.12\columnwidth]{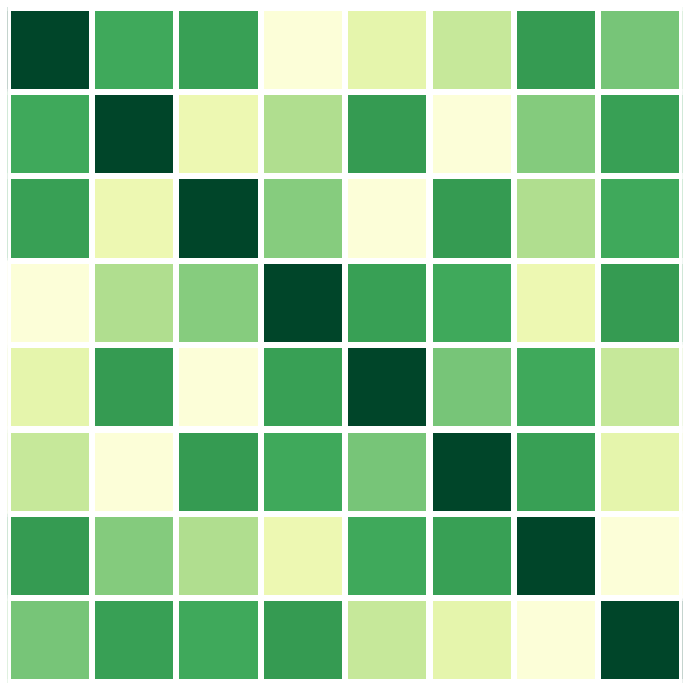}
    \label{fig:N=8_N8_d3_lr0.5_s20000_z_fixed_mini_batch_B2_NcB_step8000_wo_cbar}
    \hspace{1mm}
    \includegraphics[height=.12\columnwidth]{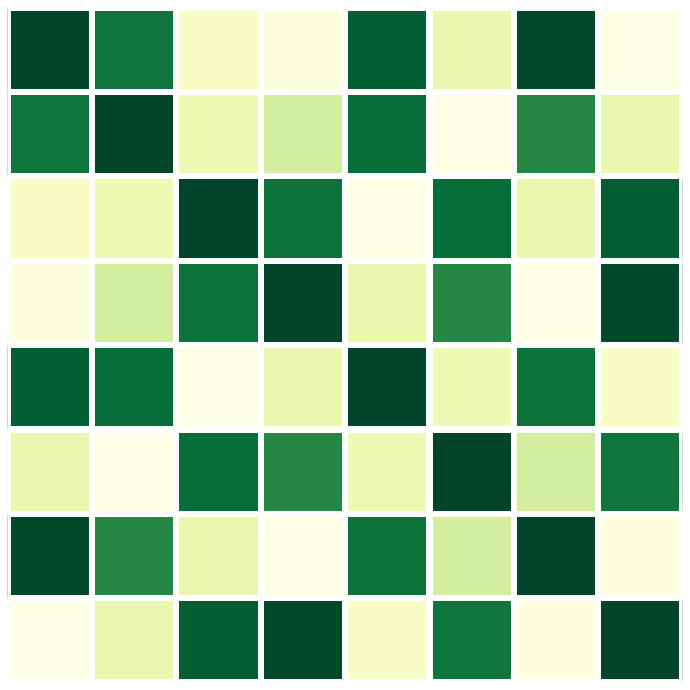}
    \label{fig:N=8_N8_d3_lr0.5_s20000_z_fixed_mini_batch_B2_f_step8000_wo_cbar}
    \hspace{1mm}

\vspace{-2.5mm}
\raisebox{4.0mm}[0mm][0mm]{\rotatebox[origin=l]{90}{\scriptsize{\textsc{$d=5$}}}\hspace{1mm}
\hspace{-2.4mm}}
\subfigure[full-batch]{
    \includegraphics[height=.12\columnwidth]{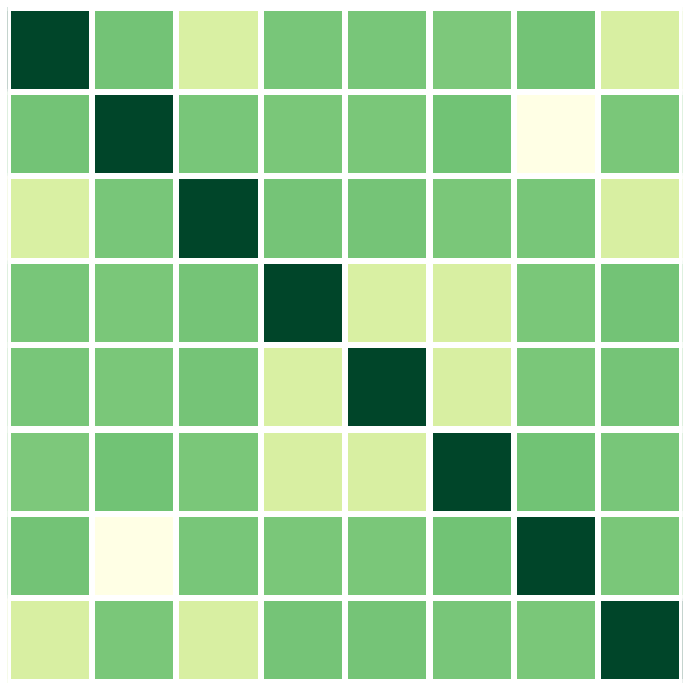}
    \label{fig:N=8_N8_d5_lr0.5_s20000_z_fixed_mini_batch_B2_full_step8000_wo_cbar}
}
\subfigure[${N\choose B}$-all]{
    \includegraphics[height=.12\columnwidth]{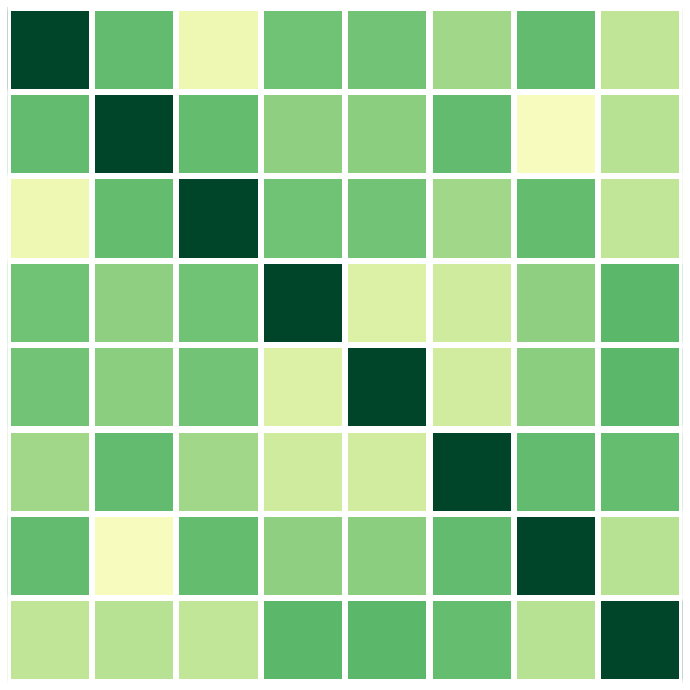}
    \label{fig:N=8_N8_d5_lr0.5_s200000_z_fixed_mini_batch_B2_NcB_step90000_wo_cbar}
}
\subfigure[${N\choose B}$-sub]{
    \includegraphics[height=.12\columnwidth]{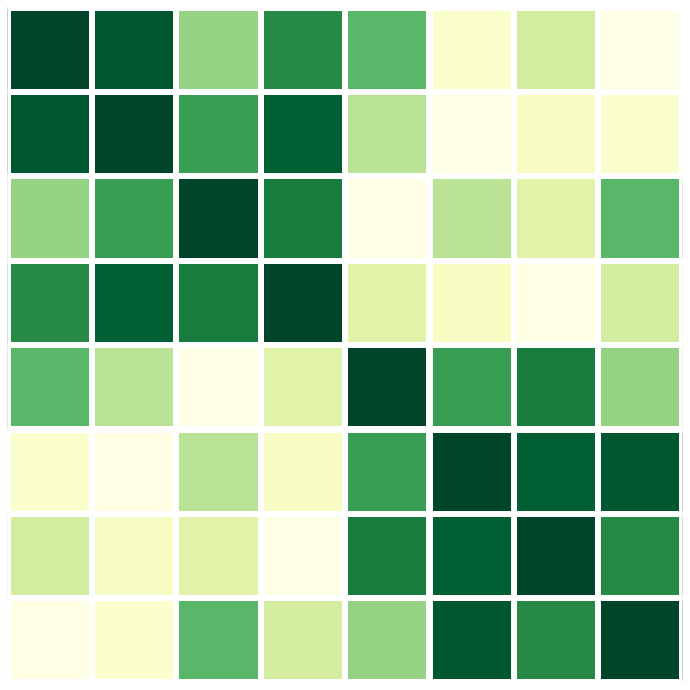}

    \label{fig:N=8_N8_d5_lr0.5_s20000_z_fixed_mini_batch_B2_f_step8000_wo_cbar}
}

\vspace{-3mm}
\caption{Theoretically unproven setting. Heatmap of $N \times N$ matrix when $N=8$ and $d<N-1$.
}
\label{fig:sim_theorem_unproven}
\end{figure*}

\subsection{Real Datasets}\label{sec:real_exp_appendix}
To demonstrate the practical effectiveness of the proposed SC method, we consider a setting where embeddings are learned by a parameterized encoder. We employ two widely recognized uni-modal mini-batch contrastive learning algorithms: SimCLR~\cite{chen2020simple} and SogCLR~\cite{yuan2022provable}, and integrate different batch selection methods from: (i) SGD (algo.~\ref{alg:sgd_wo_replacement}), (ii) OSGD (algo.~\ref{alg:osgd_wo_replacement}), (iii) SC (algo.~\ref{alg:sc_appendix}) into these frameworks. We compare the pre-trained models' performances in the retrieval downstream tasks on the corrupted and the original datasets.

We conduct the mini-batch contrastive learning with the mini-batch size $B=32$ using ResNet18-based encoders on CIFAR-100 and Tiny ImageNet datasets. All learning is executed on a single NVIDIA A100 GPU. The training code and hyperparameters are based on the official codebase of SogCLR\footnote{https://github.com/Optimization-AI/SogCLR}~\cite{yuan2022provable}. We use LARS optimizer\cite{you2017large} with the momentum of $0.9$ and the weight decay of $10^{-6}$. We utilize the learning rate scheduler which starts with a warm-up phase in the initial 10 epochs, during which the learning rate increases linearly to the maximum value  $\eta_{\max}=0.075 \sqrt{B}$. After this warm-up stage, we employ a cosine annealing (half-cycle) schedule for the remaining epochs. For OSGD, we employ $k=1500$, $q=150$. To expedite batch selection in the proposed SC, we begin by randomly partitioning $N$ positive pairs into $kB$-sized clusters, using $k=40$. We then apply the SC method to each $kB$ cluster to generate $k$ mini-batches, resulting in a total of $k\times (N/kB) = N/B$ mini-batches. We train models for a total of 100 epochs.

Table~\ref{tab:performance_comparison_all} presents the top-1 retrieval accuracy on CIFAR-100 and Tiny ImageNet. We measure validation retrieval performance on the true as well as corrupted datasets. The retrieval task is defined to be finding the positive pair image of a given image among all pairs (the number of images of the validation dataset).  

\begin{table}[ht]
\centering
\caption{Top-1 retrieval accuracy on CIFAR-100 (or Tiny ImageNet), when each algorithm uses CIFAR-100 (or Tiny ImageNet) to pretrain ResNet-18 with SimCLR and SogCLR objective. SC algorithm proposed in Sec.~\ref{subsec:spectral_clustering} outperforms existing baselines.
}

\label{tab:performance_comparison_all}
\begin{tabular}{@{}l*{4}{c}@{}}
\toprule
 & \multicolumn{4}{c}{Image Retrieval}\\
\cmidrule(lr){2-5} 
 & \multicolumn{2}{c}{CIFAR-100} & \multicolumn{2}{c}{Tiny ImageNet} \\
\cmidrule(lr){2-3} \cmidrule(lr){4-5} 
 & SimCLR & SogCLR & SimCLR & SogCLR \\
\midrule
SGD & 46.91\% & 12.34\% & 57.88\% & 16.70\% \\
OSGD & 47.55\% & 13.88\% & 59.34\% & 20.43\% \\
SC & $\bm{56.67}\%$ & $\bm{47.42}\%$ & $\bm{68.07}\%$ & $\bm{54.20}\%$ \\
\bottomrule
\end{tabular}
\end{table}

We also consider the retrieval task under a harder setting, where the various corruptions are applied per image so that we can consider a set of corrupted images as a hard negative samples. Table~\ref{tab:performance_comparison_corrupted} presents the top-1 retrieval accuracy results on CIFAR-100-C and Tiny ImageNet-C, the corrupted datasets~\cite{hendrycks2019benchmarking} designed for robustness evaluation. CIFAR-100-C (Tiny ImageNet-C) has the same images as CIFAR-100 (Tiny ImageNet), but these images have been altered by 19 (15) different types of corruption (e.g., image noise, blur, etc.). Each type of corruption has five severity levels. We utilize images corrupted at severity level 1. These images tend to be more similar to each other than those corrupted at higher severity levels, which consequently makes it more challenging to retrieve positive pairs among other images. To perform the retrieval task, we follow the following procedures: (i) We apply two distinct augmentations to each image to generate \emph{positive} pairs; (ii) We extract embedding features from the augmented images by employing the pre-trained models; (iii) we identify the pair image of the given augmented image among augmentations of 19 (15) corrupted images with the cosine similarity of embedding vectors. This process is iterated across $10$K CIFAR-100 images ($10$K Tiny-ImageNet images). The top-1 accuracy measures a percentage of retrieved images that match its positive pair image, where each pair contains two different modality stemming from a single image.

\end{document}